\documentclass[12pt]{article}

\usepackage[letterpaper,margin=1in]{geometry}
\usepackage{natbib}
\usepackage[utf8]{inputenc}
\usepackage[T1]{fontenc}
\usepackage{url}
\usepackage{seqsplit}  %
\usepackage{booktabs}
\usepackage{amsfonts}
\usepackage{nicefrac}
\usepackage{microtype}
\usepackage{xcolor}
\usepackage{graphicx}
\graphicspath{{figures/}}
\usepackage{caption}
\usepackage{titlesec}
\titlespacing*{\section}{0pt}
  {2.8ex plus 0.8ex minus 0.2ex}{1.5ex plus 0.2ex}
\titlespacing*{\subsection}{0pt}
  {2.4ex plus 0.6ex minus 0.2ex}{1.1ex plus 0.2ex}
\titlespacing*{\subsubsection}{0pt}
  {2.2ex plus 0.5ex minus 0.2ex}{0.9ex plus 0.2ex}
\titlespacing*{\paragraph}{0pt}
  {2.0ex plus 0.5ex minus 0.2ex}{1em}
\usepackage{tikz}
\usetikzlibrary{arrows.meta,calc,positioning}
\usepackage{amsmath,amssymb,amsthm}
\usepackage{mathtools}
\usepackage{algorithm}
\usepackage{algpseudocode}
\usepackage[hidelinks]{hyperref}

\newcommand{\stepnote}[1]{\bigl[\,\parbox[t]{0.32\linewidth}{\raggedright #1}\,\bigr]}

\newtheoremstyle{jasa-plain}
  {8pt plus 2pt minus 2pt}{8pt plus 2pt minus 2pt}
  {\itshape}{}{\bfseries}{.}{0.5em}{}
\theoremstyle{jasa-plain}
\newtheorem{theorem}{Theorem}

\newtheorem{corollary}[theorem]{Corollary}
\newtheorem{lemma}[theorem]{Lemma}
\newtheorem{proposition}[theorem]{Proposition}
\newtheorem{definition}{Definition}[section]
\newtheoremstyle{jasa-definition}
  {8pt plus 2pt minus 2pt}{8pt plus 2pt minus 2pt}
  {\normalfont}{}{\bfseries}{.}{0.5em}{}
\theoremstyle{jasa-definition}
\newtheorem{assumption}{Assumption}[section]
\newtheorem*{remark}{Remark}

\newcommand{\revhighlight}{0}
\newcommand{\rev}[1]{\if1\revhighlight{\color{blue}#1}\else#1\fi}
\newenvironment{revblock}{\if1\revhighlight\color{blue}\fi}{}
\newcommand{\revfloat}{\if1\revhighlight\color{blue}\captionsetup{font+={color=blue}}\fi}
\makeatletter
\@for\rev@key:={battey2018,cai2021,cai2022adaptive,cai2022integrative,%
cai2022nonparametric,carroll2006,chen2014,dempster1983,dobriban2021,%
donoho1994,duan2022,duan2023,duchi2018,fan2019,fan2023,genest1986,%
gneiting2007,heitjan1989,heitjan1991,hjort2003,huo2019,jordan2019,%
kpotufe2021,li2022,li2023,meinshausen2015,ranjan2010,shimodaira2000,%
stone1961,szabo2020,szabo2023,tian2023,tsybakov2009,volgushev2019,yu1997,%
zamir1992}\do{\expandafter\let\csname rev@bib@\rev@key\endcsname\@empty}
\AtBeginDocument{%
  \let\rev@orig@lbibitem\@lbibitem
  \def\@lbibitem[#1]#2{\rev@orig@lbibitem[#1]{#2}%
    \if1\revhighlight
      \leavevmode\ifcsname rev@bib@#2\endcsname\color{blue}\else\normalcolor\fi
    \fi}}
\makeatother

\newcommand{\anon}{1}

\renewenvironment{abstract}{%
  \normalsize
  \begin{center}%
    {\bfseries \abstractname\vspace{-.5em}\vspace{0pt}}%
  \end{center}%
  \noindent\ignorespaces
}{\par}

\begin{document}

\def\spacingset#1{\renewcommand{\baselinestretch}%
{#1}\small\normalsize} \spacingset{1}

\if1\anon
{
  \title{\bf Matching Rates and Optimal Allocation for Federated
    Probe-Logit Distillation under Heterogeneous Bandwidth Budgets}
  \author{Prasanjit Dubey \\
    H.~Milton Stewart School of Industrial and Systems Engineering, \\
    Georgia Institute of Technology \\
    and \\
    Xiaoming Huo \\
    H.~Milton Stewart School of Industrial and Systems Engineering, \\
    Georgia Institute of Technology}
  \date{}
  \maketitle
} \fi

\if0\anon
{
  \bigskip
  \bigskip
  \bigskip
  \begin{center}
    {\LARGE\bf Matching Rates and Optimal Allocation for Federated
      Probe-Logit Distillation under Heterogeneous Bandwidth Budgets}
  \end{center}
  \medskip
} \fi

\bigskip
\begin{abstract}
Large language models need vast training data, yet in medicine, finance,
and government the data cannot leave the site that holds it, leaving each
site a small local model. Federated distillation, an emerging remedy,
shares neither data nor model weights: each site reports its model's
log-probabilities (logits) on a public set of probe inputs, and a server
averages them and trains one global model on the average. Such systems
are already built, but without statistical guarantees; this paper
supplies them for the method's unavoidable cost: the uplink carries few
bits, so logits must be quantized. Accuracy is Kullback--Leibler risk,
expressed in the quantities an operator controls: $K$ sites, $V$ response
categories, and $B$ bits per site per input. For quantize-and-average
protocols the communication term has exact order $K^{-1}2^{-2B/V}$, with
matching upper and lower bounds, and refining residuals across rounds
improves it geometrically. Unequal links call for a closed-form
water-filling allocation of a fixed budget, using scale information the
messages already carry. Over arbitrary encoders the rate is governed by
the total budget $KB$ instead, which quantifies the price of averaging.
In a seven-model GPT federation the allocation rule cuts perplexity by
$27.8\%$ at the tightest budget.
\end{abstract}

\noindent%
\vfill

\newpage
\spacingset{1.8} %
\setlength{\abovedisplayskip}{10pt plus 2pt minus 4pt}
\setlength{\belowdisplayskip}{10pt plus 2pt minus 4pt}
\setlength{\abovedisplayshortskip}{4pt plus 2pt minus 2pt}
\setlength{\belowdisplayshortskip}{7pt plus 2pt minus 3pt}

\section{Introduction}\label{sec:intro}

Large language models (LLMs) owe most of their ability to the size of
the corpus behind them, and in a growing class of high-stakes
applications that corpus cannot be assembled: regulation, consent terms,
or institutional policy forbids moving the records. As a running example,
each hospital in a clinical consortium fine-tunes a small local LLM on its
own patient records, and that model is weak alone. What the consortium
wants is the accuracy of a model trained on every record, without any
record leaving its site, over an existing low-rate network that cannot
carry full-precision model weights or gradients.

\emph{Federated distillation} is the emerging answer, and systems of this
shape are already built. Instead of data or parameters, each site
transmits its model's outputs on a shared public set of probe inputs, and
a server averages the outputs and trains one global model on the average
--- as class scores, ensemble distillation on proxy data, compressed
student--teacher exchanges, or intermediate
features~\citep{fedmd2019,feddf2020,fedkdsony2022,fedfd2025}. The
message is charged per probe input, not per parameter, so it does not
grow with model size~\citep{fedllmsurvey2024}. To a statistician this is
distributed estimation under a communication constraint: $K$ sites each
hold $n$ observations, the target is a conditional distribution over a
vocabulary of $V$ categories, and each site may send at most $B$ bits per
probe input. \rev{The problem therefore belongs with divide-and-conquer
and one-shot averaging
estimators~\citep{chen2014,jordan2019,dobriban2021,fan2023}
and with minimax estimation and testing under communication
budgets~\citep{szabo2020,cai2022nonparametric,szabo2023},
except that each site communicates a predictive distribution rather than
a parameter estimate.} What the engineering literature has not supplied is the
statistical theory of that channel. Communication-oriented variants
compress soft labels through quantization and delta
coding~\citep{sattlermarbanrischkesamek2022}, and complementary theory
studies how distillation data and sampling affect
training~\citep{liuzhangsongletaief2024}, but neither characterizes the
finite-codebook KL risk as a function of $(K,B,V,m,n)$, establishes a
matching lower bound, optimizes a fixed aggregate budget across
heterogeneous nodes, or compares scalar averaging with arbitrary $B$-bit
protocols. Those are the questions answered here.

The vehicle is \emph{Federated Probe-Logit Distillation}
(FPLD)~\citep{dubey2026federatedlanguagemodelsbandwidth}, in which each
node transmits scalar-quantized logits on a public probe set at a stated
bit budget and the aggregator distills a global parametric student from
their average. FPLD is not a competitor to the systems above but the
smallest protocol that retains their essential primitive --- averaged
model outputs on shared inputs --- while making the bit budget explicit;
as Section~\ref{sec:fpld} makes precise, each of its ingredients is
either forced by the uplink constraint or is a free choice whose price
this paper computes. In one sentence, the contribution is a
statistical guarantee for this methodology: matching upper and lower
bounds on the communication cost, a closed-form rule for spending a bit
budget across unequal links, and the price paid for the anonymity and
dropout tolerance of averaging.

\paragraph{Statistical reading of the rate.}
The FPLD risk bound separates into two classical statistical terms and
one new communication term. The upper bound recapped in
Section~\ref{sec:fpld} (Theorem~\ref{thm:fpld-upper}) takes the form
$O\bigl(d/(Kn) + \rho\sqrt{V \log V/m} + K^{-1} \cdot 2^{-2B/V}\bigr)$
plus optimization and distillation-fit slack, with parameter dimension
$d$, $K$ nodes, $n$ samples per node, probe set size $m$, vocabulary
size $V$, and per-node uplink budget $B$. The first two terms are
classical: $d/(Kn)$ is the pooled parametric-MLE rate on $Kn$ effective
samples, and $\rho\sqrt{V \log V/m}$ is a Rademacher-type
uniform-convergence rate for the aggregator's softmax-linear fit on $m$
probe contexts, with $\rho$ the Radon--Nikodym ratio that accounts for
probing under a public marginal $Q$ instead of $P^\star_X$.

The bandwidth term $K^{-1} \cdot 2^{-2B/V}$ is the new statistical
object, and prior work leaves two questions about it open. Its $1/K$
gain comes from $K$-fold averaging of independent dithered errors at the
aggregator, as in classical distributed estimation under communication
constraints~\citep{zhangduchiwainwright2013, shamirsrebro2014,
huanghuo2019, acharyacanonnetyagi2020, gargmanguyen2014}. The open
questions are: (i)~is the rate tight as a matching lower bound, or could
a smarter encoder strictly improve? (ii)~when uplinks are heterogeneous
and node $i$ has $B_i$ payload bits, what is the optimal allocation of a
fixed aggregate payload budget $B_{\mathrm{tot}} = \sum_i B_i$?

The answers separate by protocol class: within scalar
quantize-and-average protocols, the trace bound
$\mathrm{tr}(H_\mathrm{soft}) \le 1$ applied to the exact softmax-KL
expansion gives the matching rates and the allocation rule, while for
arbitrary $B$-bit encoders a packing/Fano
argument\rev{~\citep{yu1997}} and a global-index
construction give bounds with matching exponential dependence on
$KB/(V-1)$ (Section~\ref{sec:uncond}).

\paragraph{Contributions.}
\begin{enumerate}
\item \textbf{Matching lower bound and multi-round achievability
(Theorems~\ref{thm:matching-lb} and~\ref{thm:lb-mr}).} Under
non-degeneracy, a matching $\Omega(K^{-1} \cdot 2^{-2B/V})$ lower bound
pins the single-round bandwidth-axis rate of the dithered FPLD
construction. Sequential refinement with nested/scaled residual
quantizers sharpens the achievable rate geometrically in the number of
rounds $T$, whereas vanilla FPLD's quadratic bandwidth term is
$T$-independent and hence strictly suboptimal for every $T>1$ when
$B\ge2V$.
\item \textbf{Heterogeneous-bandwidth bounds and optimal allocation
(Theorems~\ref{thm:het-upper} and~\ref{thm:alloc}).} The FPLD upper
bound generalizes to per-node budgets $B_i$, with the exact dither
variances entering its quadratic bandwidth term. Subject to an
aggregate payload constraint, the minimizer of the resulting
exponential envelope is a closed-form log-tilted water-filling
allocation, with an active-set form when per-node bounds bind.
\item \textbf{Self-describing scale metadata and robust allocation
(Corollary~\ref{cor:adaptive}, Algorithm~\ref{alg:adaptive-fpld}).}
Rounding each node's required cap upward to a public grid and reporting
its fixed-length index makes the decoder's scale requirement explicit,
and the same index reveals the allocation weight --- so exact metadata
delivers the envelope oracle with no warm-up and no separate estimation
channel. Logarithmic weight error of at most $\varepsilon$ costs at
most a factor $2^{2\varepsilon}$.
\item \textbf{One-round-lag allocation under scale drift
(Theorem~\ref{thm:streaming}, Section~\ref{ssec:streaming}).}
Scheduling round $t$ from the scale metadata decoded in round $t-1$ is
deterministically within $2^{2\delta}$ of the current clairvoyant
envelope oracle under per-round logarithmic drift $\delta$, with no
interiority condition; Proposition~\ref{prop:staticfail} is an exact
diagnostic for when a fixed schedule should be refreshed.
\item \textbf{Unconditional rates, the price of scalar averaging, and
grouped aggregation (Section~\ref{sec:uncond}).} Over \emph{arbitrary}
$B$-bit-per-node protocols, the noiseless-logit upper and lower bounds
have matching exponential dependence on $KB/(V-1)$, and a global-index
construction attains that total-budget exponent. Consequently no
construction-free matching bound at $K^{-1}2^{-2B/V}$ exists, and
bounded-resolution unbiased scalar quantize-and-average protocols pay
an exponential-in-$K$ price. Since the global-index protocol needs
exact cross-node agreement, grouped FPLD interpolates toward the more
dropout-robust coordinate-splitting protocol.
\item \textbf{Exact-codebook numerical validation.} Frozen simulations
use the finite $2^b$-point channel analyzed here rather than a nominal
infinite lattice, and recover the $1/K$ law, the finite-rate factor,
the heterogeneous allocation rule, and the scale-drift guarantee. A
seven-model GPT-2/GPT-Neo federation and exact replays on GPT-2,
Pythia-2.8B, and Falcon3-3B carry the bandwidth and multi-round
predictions to output dimensions from $50{,}257$ to $131{,}072$.
\end{enumerate}

\paragraph{From rates to design rules.} \rev{These rates say how risk moves
with $(K,n,B,V,m)$ and which design variable binds. For example,
Theorem~\ref{thm:alloc} replaces a search over per-node budgets with a
closed-form rule that, in the replay of Section~\ref{sec:reallm}, cuts
token perplexity by $27.8\%$ at the tightest budget.}

\paragraph{Scope.} This is a focused statistical study of FPLD's
training-time KL rate.
\rev{\citet{dubey2026federatedlanguagemodelsbandwidth} develop the full FPLD
protocol (Algorithm~\ref{alg:fpld}), inference-time conformal coverage,
and the upper bound that Theorem~\ref{thm:fpld-upper} restates here with
explicit finite-sample, quantizer, and expected-risk assumptions. Privacy
and Byzantine robustness are out of scope; unrestricted multi-round lower
bounds and non-parametric students are left open, while the unconditional
single-round question and dynamic heterogeneous allocation are treated in
Sections~\ref{sec:uncond} and~\ref{ssec:streaming}.}

\section{Setup}\label{sec:setup}

\paragraph{Vocabulary, contexts, predictors.} Let $\mathcal{V}$ be a finite
vocabulary of size $V$ and let $\mathcal{X} = \mathcal{V}^{\le L}$ be the set
of contexts of length at most $L$. A predictor is a map
$\mathcal{X} \to \Delta(\mathcal{V})$, where $\Delta(\mathcal{V})$ is the
simplex over tokens.

\paragraph{Parametric family and canonical logits.} Let
$\mathcal F_\Theta=\{P_\theta:\theta\in\Theta\subset\mathbb R^d\}$ be a
family of strictly positive conditional distributions. To remove the
additive-shift ambiguity of neural-network logits, throughout the theory
\emph{logit} means the canonical centered logit
\[
\ell_\theta(x)_v
:=\log p_\theta(v\mid x)-\frac1V\sum_{w\in\mathcal V}
\log p_\theta(w\mid x),
\qquad \mathbf 1^\top\ell_\theta(x)=0.
\]
Thus $P_\theta(\cdot\mid x)=\mathrm{softmax}(\ell_\theta(x))$. Under
Assumption~\ref{ass:A1}, compactness and smoothness on the finite sets
$\mathcal X$ and $\mathcal V$ give the finite uniform bound
$L_\ell:=\sup_{\theta\in\Theta,x\in\mathcal X,v\in\mathcal V}
|\ell_\theta(x)_v|$. The theoretical quantizers use a clip cap at least
$L_\ell$, so clipping is an identity on every transmitted model logit;
it remains in the protocol as an overload safeguard.

\paragraph{Nodes and data.} There are $K$ nodes; node $i \in \{1, \dots, K\}$ holds
an i.i.d.\ dataset
$\mathcal{D}_i = \{(x_j^{(i)}, y_j^{(i)})\}_{j=1}^n$ drawn from
$P^\star$ on $\mathcal{X} \times \mathcal{V}$. Raw data never leaves a
node.

\paragraph{Public probe set.} The probe contexts
$X_{\mathrm{probe}} = \{x^{(l)}\}_{l=1}^m$ are drawn i.i.d.\ from a probe
marginal $Q$ that covers $P^\star_X$ in the sense
that its Radon--Nikodym derivative is bounded:
$\rho = \operatorname{ess\,sup}_x dP^\star_X / dQ < \infty$.

\paragraph{Communication.} Bandwidth is charged on uplink only. In the
homogeneous case (Section~\ref{sec:fpld}), every node transmits at most
$B$ bits per probe context per round; in the heterogeneous case
(Section~\ref{sec:het}), node $i$ transmits a payload of at most $B_i$ bits. Downlink
broadcast of the aggregated student is uncharged, matching the standard
federated convention.

\paragraph{Estimand.} The aggregator produces a global student
$\hat P : \mathcal{X} \to \Delta(\mathcal{V})$ approximating $P^\star$;
the risk is
$\mathbb{E}_{X \sim P^\star_X}\bigl[\mathrm{KL}(P^\star(\cdot \mid X)
\,\|\, \hat P(\cdot \mid X))\bigr]$.

\paragraph{Randomness convention.} Let $U$ collect all quantizer
dithers and any other encoder or decoder randomization. For a randomized
protocol, every displayed expected KL risk means
$\mathbb E_{X\sim P_X^\star}\mathbb E_U[\mathrm{KL}(P^\star(\cdot\mid
X)\|\hat P(\cdot\mid X))]$, conditional on the realized local datasets
and public probe set unless a statement says otherwise. The upper-bound
theorems below display $\mathbb E_{\mathcal D,X_{\mathrm{probe}}}$
explicitly and are expected-risk statements under
Assumption~\ref{ass:A7}; the metadata-robustness and streaming results
are deterministic conditional comparisons of allocation objectives.
Deterministic protocols simply omit
the inner expectation over $U$.

\section{FPLD Protocol and Upper-Bound Recap}\label{sec:fpld}

Federated Probe-Logit Distillation (FPLD) replaces gradients or weights
with \emph{quantized logits on a public probe set}. All nodes initialize from a
shared pretrained base $\hat P^{(0)}$.
In each of $T$ rounds, a node fine-tunes locally on its private dataset,
evaluates its updated model on the probe set, and transmits a
per-coordinate scalar-quantized logit vector using
$b:=B/V\in\mathbb N$ bits per coordinate, where $b\ge1$.
The aggregator averages the $K$ quantized logit vectors coordinatewise
and distills a global student on the averaged logits, then broadcasts the
student back to all nodes. Total uplink is $O(K T m B)$; downlink is
uncharged.

Figure~\ref{fig:fpld-overview} and Algorithm~\ref{alg:fpld} of the
supplement state the protocol in full, with its operators defined in
Appendix~\ref{app:fpld-operators}.

We adopt seven assumptions on the parametric family
(\ref{ass:A1}), the probe distribution (\ref{ass:A2}), homogeneous data
(\ref{ass:A3}), local-optimization slack (\ref{ass:A4}), the uniform
quantization channel (\ref{ass:A5}), and the aggregator's distillation
fit and probe complexity (\ref{ass:A6}), together with explicit finite-sample local-MLE
regularity (\ref{ass:A7}); full statements are in
Appendix~\ref{app:assumptions}.

\paragraph{Why these ingredients.} The protocol is deliberately
minimal, and each component is either forced by the setting or is a
free choice whose price appears in the analysis. Sending logits rather
than gradients or weights makes the probe payload $O(mB)$ per node per
round, independent of parameter dimension, and is operationally
compatible with different local architectures. The expected-risk theory
below assumes the common parametric family $\mathcal F_\Theta$ of
Assumption~\ref{ass:A1}. A shared public probe set makes coordinatewise
aggregation meaningful; its statistical cost enters through the
coverage ratio $\rho$ and the $\sqrt{V\log V/m}$ probe term. \rev{Averaging
logits forms the logarithmic opinion pool of the site predictive
distributions~\citep{genest1986}, rather than the linear
pool~\citep{ranjan2010}, and KL risk is the expected excess
logarithmic score, a strictly proper scoring
rule~\citep{gneiting2007}.} Centering
the logits is a gauge choice that removes softmax's additive-shift
ambiguity and makes the quantizer scale well defined.

Subtractive dithering is a genuine modeling choice: it makes the
reconstruction error exactly unbiased and independent of the input, so
the aggregator's $1/K$ variance reduction can be analyzed sharply.
\rev{Quantized logits are coarsened data in the sense
of~\citet{heitjan1991}; with subtractive dither the rounding error is
exactly uniform, so the Sheppard-type variance
$\Delta^2/12$~\citep{dempster1983} holds without the
smoothness conditions it otherwise requires.}
Theorem~\ref{thm:matching-lb} and
Proposition~\ref{prop:rebuttal-lb} show that this choice loses no rate
within the bounded-resolution unbiased scalar quantize-and-average
class. Coordinatewise averaging is anonymous, permutation invariant,
and tolerant of node dropout; Section~\ref{sec:uncond} quantifies the
communication price of those structural advantages and develops a
grouped interpolation. Finally, well-specification and homogeneous
data isolate the bandwidth axis: the heterogeneity treated here is
across uplinks, where the allocation question arises directly.
\rev{Section~\ref{sec:discuss} and Appendix~\ref{app:practical-noniid}
state which results survive heterogeneous sites.}

\begin{theorem}[FPLD upper bound]\label{thm:fpld-upper}
Under Assumptions~\ref{ass:A1}--\ref{ass:A7}, the small-error
condition \textbf{(SE$_{c_0}$)} of Appendix~\ref{ssec:trace-upper}, and the finite-sample
condition $dK/n\le c_{\mathrm{stat}}$ for a sufficiently small
regularity-dependent constant $c_{\mathrm{stat}}>0$, there exist
constants $c_1,c_2,c_{\mathrm{opt}},c_{\mathrm{fit}},C_R>0$ and
$C_{\mathrm{dist}}\ge1$, depending only on the quantities identified
below, such that
\[
\begin{aligned}
&\mathbb E_{\mathcal D,X_{\mathrm{probe}}}
\mathbb{E}_{X \sim P^\star_X}\mathbb E_U
\!\left[\mathrm{KL}\bigl(P^\star(\cdot \mid X)
\,\|\, \hat P(\cdot \mid X)\bigr)\right] \\
&\quad\le
\frac{c_1 d}{Kn}
+c_2 \rho \sqrt{\frac{V \log V}{m}}
+\frac{C_{\mathrm{dist}}C_q}{2K}\, 2^{-2B/V} \\
&\qquad
+C_R\frac{C_q^{3/2}(\log V)^{3/2}}{K^{3/2}}\,2^{-3B/V}
+c_{\mathrm{opt}}\varepsilon_{\mathrm{opt}}
+c_{\mathrm{fit}}\rho\varepsilon_{\mathrm{fit}}.
\end{aligned}
\]
\end{theorem}

Appendix~\ref{app:constants} inventories what each constant depends on,
gives the trace-sharpened coefficient
$C_q/2\le2L_\ell^2/3$, and records that the finite-sample condition is
the standard one-shot-averaging regime $n\gtrsim Kd$.

\rev{The expectation over local datasets is deliberate:
Assumption~\ref{ass:A7} states only moment bounds, so no
confidence-free high-probability $d/(Kn)$ claim is made.}

\paragraph{Interpretation.} The leading terms reflect
pooled parametric-MLE convergence on the $Kn$ effective samples;
probe generalization at rate $1/\sqrt m$ with a coverage penalty
$\rho$; exponentially decaying quantization distortion in
bits-per-coordinate; and controllable optimization slack. The $1/K$
prefactor of the quantization term comes from averaging $K$
independent dithered errors at the aggregator; the factor of $V$ one
might naively expect from summing per-coordinate variances is
absorbed by the softmax-Hessian trace bound (Appendix~\ref{ssec:soft-trace}).
The student-transfer factor is dimension-free, so it changes neither
the $K,V,B$ rates nor any allocation optimizer, and the exact aggregate
predictor of the bandwidth-isolation results needs no such transfer.
The bandwidth term depends on $B$ but not on $T$: vanilla FPLD does
not exploit multi-round refinement to sharpen its quantization rate.
Section~\ref{ssec:lb-mr} shows this is strictly suboptimal for
$T > 1$ at the contracting finite-codebook rates $B\ge2V$.

Theorem~\ref{thm:fpld-upper} is a refined expected-risk restatement of
the upper-bound analysis developed in prior
work~\citep{dubey2026federatedlanguagemodelsbandwidth}; everything from
Section~\ref{sec:lb} onward is new here.

\section{Matching Lower Bound and Multi-Round Achievability}\label{sec:lb}

Within the dithered-FPLD construction, the single-round bandwidth-axis
rate is pinned at $\Theta(K^{-1} \cdot 2^{-2B/V})$ by a matching lower
bound, and spending the budget across rounds sharpens it geometrically.
These are the two natural questions left by the aggregate teacher's
upper bound at quadratic order $C_q(2K)^{-1} \cdot 2^{-2B/V}$.
Theorem~\ref{thm:matching-lb} establishes the matching lower bound, and
Theorem~\ref{thm:lb-mr} gives a multi-round achievability rate
$O\bigl(K^{-1} (2^{B/V}-1)^{-2T}\bigr)$ --- effectively $2^{-2TB/V}$
at moderate per-round rates --- via sequential refinement with
nested/scaled residual quantizers; vanilla FPLD's quadratic bandwidth
term is $T$-independent and therefore strictly suboptimal for every
$T > 1$ when $B\ge2V$.

\paragraph{Protocol class.} The single-round FPLD-class protocols
$\mathcal{P}_B^{\mathrm{FPLD}}$ consist of those in which each node
$i$ observes its local dataset $\mathcal{D}_i$ and the public probe
set $X_{\mathrm{probe}}$, applies a per-coordinate clipped scalar
quantizer at $B/V$ bits with subtractive dithering
(Assumption~\ref{ass:A5}), and the aggregator forms an unbiased
$K$-fold average. The target class $\mathcal{P}^\star$ is the
parametric family of conditional distributions $P^\star$ for which
some $\theta^\star \in \mathrm{int}(\Theta)$ realizes
$P^\star = P_{\theta^\star}$; local estimates remain in $\Theta$ and
transmit the bounded canonical logits of Section~\ref{sec:setup}, so
the theoretical clip is inactive. Both match the operating regime of
Theorem~\ref{thm:fpld-upper}. For the larger class $\mathcal{P}_B$
of arbitrary $B$-bit-per-node protocols no matching bound at this rate
can exist (Section~\ref{sec:uncond}), so the construction-specific
scope here is a real price of scalar averaging, not a proof artifact.

\paragraph{The risk functional being bounded.} Theorems~\ref{thm:matching-lb}
and~\ref{thm:lb-mr}, and the bandwidth term of
Theorem~\ref{thm:fpld-upper}, all concern one isolated piece of the KL
risk, not the total risk. We make that piece explicit, since a lower
bound on it is \emph{not} a lower bound on the total risk of the FPLD
estimator: KL risk does not decompose into nonnegative statistical and
bandwidth parts, and statistical error could in principle partially
offset the quantization perturbation.

\begin{definition}[Bandwidth contribution]\label{def:bandwidth}
Let $\varepsilon_i := \tilde\ell_i - \ell_i$ be node $i$'s
quantization error under Assumption~\ref{ass:A5}, and let
$\bar e := \frac1K\sum_i \varepsilon_i$ be the average error. Evaluate
the aggregated logit at the population value (\emph{bandwidth isolation}:
$\ell_i \equiv \ell^\star$ for every $i$, so that
$\bar\ell = \ell^\star + \bar e$). The \emph{bandwidth contribution}
of a protocol is
\[
\mathbb{E}\bigl[\mathrm{KL}(P^\star \,\|\, \hat P)\bigr]_{\mathrm{bandwidth}}
\;:=\;
\mathbb{E}_{X \sim P^\star_X}\,\mathbb{E}_{\bar e}\,\mathrm{KL}\bigl(
\mathrm{softmax}(\ell^\star(X)) \,\big\|\,
\mathrm{softmax}(\ell^\star(X) + \bar e(X))\bigr),
\]
equivalently, by a second-order expansion of the log-partition
function,
\[
\mathbb{E}\bigl[\mathrm{KL}(P^\star \,\|\, \hat P)\bigr]_{\mathrm{bandwidth}}
\;=\;
\tfrac12\,\mathbb{E}_X \sum_v
H_{\mathrm{soft},vv}(p^\star)\,\mathrm{Var}(\bar e_v \mid X)
\;+\; R_3,
\]
where $\mathrm{Cov}(\bar e \mid X)$ is diagonal by cross-coordinate
independence and the third-order remainder $R_3$ is bounded in
Appendices~\ref{ssec:trace-upper} and~\ref{ssec:matching-lb-proof}.
\end{definition}

Because $\ell_i$ is held at $\ell^\star$, this functional is free of
$n$ and $m$; quantifying over them in
Theorem~\ref{thm:matching-lb} is therefore harmless rather than a
claim about the statistical terms. Every statement below carrying the
$[\cdot]_{\mathrm{bandwidth}}$ subscript --- including
Theorem~\ref{thm:lb-mr} and its proof in Appendix~\ref{app:lb-mr} ---
is to be read in the sense of Definition~\ref{def:bandwidth}.

\begin{theorem}[Matching lower bound for the dithered FPLD construction]\label{thm:matching-lb}
Under Assumptions~\ref{ass:A1} and~\ref{ass:A5}, the non-degeneracy
condition
\[
\textbf{(ND)}\quad
\mathbb{E}_{X \sim P^\star_X}\!\bigl[1 - \|p^\star(\cdot \mid X)\|_2^2\bigr]
\;\ge\; c_p \;>\; 0,
\]
and the $c_p$-calibrated small-error condition
\textbf{(SE$_{c_p}$)} of Appendix~\ref{ssec:matching-lb-proof}, the bandwidth contribution of the
dithered FPLD construction satisfies, for any $K, n, m \ge 1$ and
$B \ge V$,
\[
\mathbb{E}\bigl[\mathrm{KL}(P^\star \,\|\, \hat P_{\mathrm{FPLD}})\bigr]_{\mathrm{bandwidth}}
\;\ge\;
\frac{c_p L_\ell^2}{12 K}\cdot 2^{-2B/V}.
\]
\end{theorem}

Combined with the direct aggregate-teacher calculation underlying
Theorem~\ref{thm:fpld-upper}, the bandwidth-axis rate of the dithered
FPLD construction is $\Theta(K^{-1} \cdot 2^{-2B/V})$ up to the
non-degeneracy and finite-codebook constants:
$c_{\mathrm{lb}} = c_p L_\ell^2 / 12$ on the lower-bound side and
$C_q/2$ on the upper-bound side. The proof, in Appendix~\ref{ssec:matching-lb-proof},
is a direct computation of the dither variance combined with the
softmax-KL cumulant expansion and the trace bound
$\mathrm{tr}(H_\mathrm{soft}) \le 1$. Both bounds, like
Definition~\ref{def:bandwidth}, average contexts under the risk's own
law $P^\star_X$. The probe marginal $Q$ plays no role in this
bandwidth-isolated direct-teacher quantity; in the fitted-student bound
of Theorem~\ref{thm:fpld-upper}, it enters through probe generalization,
fit transfer, and $C_{\mathrm{dist}}$.

Theorem~\ref{thm:matching-lb} is for the dithered construction, but under
the finite-resolution spread conditions below the
$K^{-1}2^{-2B/V}$ floor holds for the whole \emph{bounded-resolution
unbiased scalar quantize-and-average} class $\mathcal{P}_B^{\mathrm{scalar}}$: one fixed-rate,
bounded-resolution scalar quantizer per coordinate, conditionally independent
errors across coordinates and nodes, and an aggregator that averages
and applies the softmax. It contains the dithered uniform quantizer of
Assumption~\ref{ass:A5} and also includes non-uniform, companding, and learned
fixed-rate scalar quantizers whenever they satisfy these same unbiasedness,
bounded-resolution, and conditional-independence requirements.

\begin{proposition}[Scalar-class lower bounds]\label{prop:rebuttal-lb}
Let $b:=B/V\in\mathbb N$, $M:=2^b$,
$d_{\mathcal C}(z):=\min_{c\in\mathcal C}(z-c)^2$, and
$\bar e=\bar\ell-\ell^\star$.
\textup{(i)}~Suppose that, for every coordinate $v$,
\[
\inf_{\mathcal C:\,|\mathcal C|\le M}
\mathbb E_Xd_{\mathcal C}(\ell_v^\star(X))
\ge c_{\mathrm{sq}}M^{-2}
\qquad\textbf{(QS-$\ell$)}
\]
for some $c_{\mathrm{sq}}>0$. Then every protocol in
$\mathcal{P}_B^{\mathrm{scalar}}$ satisfies, for all $v$ and $K\ge1$, $B\ge V$,
\[
\mathbb{E}_X\bigl[\mathrm{Var}(\bar e_v\mid X)\bigr]\;\ge\;c_{\mathrm{sq}}\,K^{-1}2^{-2b}.
\]
\textup{(ii)}~Let $\mu_v$ be the coordinatewise
softmax-Hessian-weighted logit measures defined in
Appendix~\ref{app:rebuttal-lb}. If
\[
\inf_{\substack{\mathcal C_1,\ldots,\mathcal C_V:\\
                  |\mathcal C_v|\le M}}
\sum_{v=1}^V\int d_{\mathcal C_v}(z)\,d\mu_v(z)
\ge c_HM^{-2}
\qquad\textup{(QS-H)}
\]
for some $c_H>0$, then under the deterministic resolution condition
\textbf{(RES)} of Appendix~\ref{app:rebuttal-lb} (a $K$-free bound on the
per-coordinate quantizer resolution),
\[
\mathbb{E}\bigl[\mathrm{KL}(P^\star\,\|\,\hat P)\bigr]_{\mathrm{bandwidth}}
\;\ge\;\tfrac14\,c_H\,K^{-1}2^{-2b}.
\]
\end{proposition}

Hence no encoder in this bounded-resolution unbiased scalar class improves the
$K^{-1}2^{-2B/V}$ rate, in
mean-squared logit error or in KL; a sharper scalar codebook changes only the constant.
For the finite context space of this paper, \textbf{(QS-$\ell$)}
and \textbf{(QS-H)} are finite-resolution assumptions and
can hold only over a stated range of rates; conditions involving a
density bounded below on an interval give simple sufficient constants in continuous-context
extensions. Proofs and exact conditions are in
Appendix~\ref{app:rebuttal-lb}; quantizers are
compared empirically in Appendix~\ref{app:synthetic-extra}.

\subsection{Multi-round sequential-refinement achievability}\label{ssec:lb-mr}

Spending the bit budget across rounds with nested/scaled residual
quantization sharpens the exponent linearly in $T$, whereas
Theorem~\ref{thm:matching-lb}'s rate is for a single round.

\paragraph{Multi-round protocol class.} Let
$\mathcal{P}_{B,T}^{\mathrm{FPLD\text{-}ref}}$ denote the class of
$T$-round FPLD-class protocols in which each round transmits $B$ bits
per node and each node uses a \emph{nested/scaled residual
quantizer}: in round $t$, the quantizer step size is scaled to the
residual range from round $t-1$, equivalently producing an effective
scalar quantizer of step $2L_\ell(2^{B/V}-1)^{-T}$ per coordinate
after $T$ rounds --- $T\log_2(2^{B/V}-1)$ effective bits, approaching
$TB/V$ as $B/V$ grows. A
fixed-step residual quantizer that does not rescale across rounds
stalls at $O(K^{-1} \cdot 2^{-2B/V})$ regardless of $T$; the
rescaling is what unlocks the multi-round gain.

\begin{theorem}[Sequential-refinement achievability]\label{thm:lb-mr}
Under Assumption~\ref{ass:A1} and the independently dithered scalar
channel of Assumption~\ref{ass:A5}, rescaled in each round to the
current deterministic residual range, with the small-error condition
imposed at the total effective rate ---
$(2^{B/V}-1)^{-T}\,(\log V)^{3/2} \le c_0 \sqrt{K}$ --- sequential
refinement achieves
\[
\mathbb{E}\bigl[\mathrm{KL}(P^\star \,\|\, \hat P_{\mathrm{seq\text{-}ref}})\bigr]_{\mathrm{bandwidth}}
\;\le\; \frac{C_3'}{K}\cdot \bigl(2^{B/V}-1\bigr)^{-2T}
+ O\!\Bigl(\tfrac{(\log V)^{3/2}}{K^{3/2}}\cdot \bigl(2^{B/V}-1\bigr)^{-3T}\Bigr),
\]
where $C_3'=L_\ell^2/6$, whereas vanilla FPLD's quadratic bandwidth term stalls at
$C_q(2K)^{-1}2^{-2B/V}$ for every $T$, with the same controlled cubic
remainder as in Theorem~\ref{thm:fpld-upper}. Since
$(2^{B/V}-1)^{-2T} = 2^{-2TB/V}(1 - 2^{-B/V})^{-2T}$, this
is the idealized rate $2^{-2TB/V}$ up to an absolute constant whenever
$T \le 2^{B/V}$, and in general loses at most one bit per coordinate
per round: $(2^{B/V}-1)^{-2T} \le 2^{-2T(B-V)/V}$. At $B = V$ the
finite codebook cannot contract the worst-case residual, so the
multi-round gain requires $B \ge 2V$. The exact coefficient $C_3'$
coincides with the large-$B/V$ limit of $C_q/2$ in
Assumption~\ref{ass:A5}.
\end{theorem}

The takeaway is a protocol-design one, within the bandwidth
functional of Definition~\ref{def:bandwidth}: the $T$-fold sharpening
requires only a change in the per-round residual encoder, not a change
in the bit budget or channel structure. Vanilla FPLD as specified in
Algorithm~\ref{alg:fpld} does not exploit this; sequential refinement
does. Like Theorem~\ref{thm:matching-lb}, the result is stated for the
bandwidth-isolation quantity, in which the pre-quantization logit is
$\ell^\star$; an implementable variant, in which each node refines its
\emph{local} residual $\ell_i - \hat\ell^{(t-1)}$ and accumulates
local statistical error across rounds, is left to future work. The
proof iterates the single-round bound on the residual at each round
(Appendix~\ref{app:lb-mr}).

At equal total budget $TB$, this matches the single-round rate at $TB$ bits
(Theorem~\ref{thm:fpld-upper}) up to the finite-codebook factor above; even \emph{averaging} the vanilla rounds rather than
keeping only the latest improves the bandwidth term only polynomially,
$\Theta((KT)^{-1} 2^{-2B/V})$ (Appendix~\ref{app:rebuttal-vanilla}), so the
exponential gain in $T$ comes from the residual rescaling, not from extra bits.

\section{Heterogeneous Bandwidth and Optimal Allocation}\label{sec:het}

This section drops the uniform per-node budget $B$ of
Section~\ref{sec:fpld}, gives node $i$ a payload of at most $B_i$ bits
per probe context, and asks how the upper bound generalizes and how a
network operator should allocate a fixed aggregate payload budget
$\sum_i B_i \le B_{\mathrm{tot}}$. Real federated deployments need this:
hospitals on a metropolitan WAN have one rate; a regional clinic on
satellite has another.

Assumption~\ref{ass:A5} (uniform scalar quantization) is replaced by a
self-describing per-node version. This is not optional bookkeeping:
the step size and reconstruction codebook depend on the cap, so a
heterogeneous finite-codebook decoder must know the cap used by each
encoder.

\begin{assumption}[Self-describing per-node scalar quantization]
\label{ass:A5p}
Fix a public ordered scale grid
$\mathcal L_s=\{0<\lambda_1<\cdots<\lambda_{M_s}\}$ with
$M_s\le2^s$ for a fixed $s\in\mathbb N_0$. For a codec epoch, node $i$ has a required cap
$\lambda_1\le L_{\ell,i}^{\mathrm{req}}\le\lambda_{M_s}$ that bounds every canonical
logit it will transmit. It reports
\[
r_i:=\min\{r:\lambda_r\ge L_{\ell,i}^{\mathrm{req}}\}
\quad\text{and sets}\quad L_{\ell,i}:=\lambda_{r_i}.
\]
The encoder and decoder both use the reported grid value
$L_{\ell,i}$, so upward rounding preserves the no-overload property.
The required cap and its reported index are selected before, and
independently of, the quantizer dithers; they may depend on the node's
data and logits.
Node $i$ then uses the integer per-coordinate payload rate
$b_i:=B_i/V\ge1$. Applying Assumption~\ref{ass:A5}'s finite-codebook
construction with the shared cap gives, conditional on the transmitted
logits, reported caps, and scheduled rates, mean-zero errors that are
independent across nodes and coordinates, with the exact per-coordinate distortion
\[
\begin{aligned}
q_i(B_i)
&:=\mathbb E[(\tilde\ell_{i,v}-\ell_{i,v})^2
  \mid\ell_i,L_{\ell,i},B_i]
=\frac{L_{\ell,i}^2}{3(2^{B_i/V}-1)^2}
=C_{q,i}\,2^{-2B_i/V},\\
C_{q,i}
&:=\frac{L_{\ell,i}^2}{3}(1-2^{-B_i/V})^{-2}
\le\frac{4L_{\ell,i}^2}{3}.
\end{aligned}
\]
The index is encoded in a fixed-length $s$-bit header per node per
codec epoch; throughout this section $B_i$ and $B_{\mathrm{tot}}$
denote payload bits. Appendix~\ref{app:bitaccounting} converts a strict
aggregate per-context budget into this payload budget and records the
resulting feasibility condition and multi-round amortization.
\end{assumption}

\begin{theorem}[Heterogeneous-bandwidth FPLD upper bound]\label{thm:het-upper}
Replace Assumption~\ref{ass:A5} with Assumption~\ref{ass:A5p} and keep
Assumptions~\ref{ass:A1}--\ref{ass:A4} and~\ref{ass:A6}--\ref{ass:A7}.
The reported caps and scheduled payload rates may be functions of the
local data and probe set, but must be fixed before the payload dithers.
Assume also that the finite-sample condition $dK/n\le c_{\mathrm{stat}}$
of Theorem~\ref{thm:fpld-upper} holds and that the small-error
condition \textbf{(SE$_{c_0}$)} of Appendix~\ref{ssec:trace-upper}, whose
$\sigma_{\boldsymbol B}^2=K^{-2}\sum_iq_i(B_i)$ is here heterogeneous,
holds almost surely with respect to $(\mathcal D,X_{\mathrm{probe}})$.
Then there exist constants $c_1,c_2,c_{\mathrm{opt}},c_{\mathrm{fit}},C_R>0$ and the
dimension-free $C_{\mathrm{dist}}\ge1$ of
Theorem~\ref{thm:fpld-upper} such that
\[
\begin{aligned}
&\mathbb E_{\mathcal D,X_{\mathrm{probe}}}
\mathbb{E}_{X,U}\!\left[\mathrm{KL}(P^\star \,\|\, \hat P)\right]\\
&\quad\le
\frac{c_1 d}{K n}
+c_2 \rho \sqrt{\frac{V \log V}{m}}
+\frac{C_{\mathrm{dist}}}{2K^2}
 \mathbb E_{\mathcal D,X_{\mathrm{probe}}}
 \!\left[\sum_{i=1}^K C_{q,i}2^{-2B_i/V}\right]\\
&\qquad
+C_R(\log V)^{3/2}
 \mathbb E_{\mathcal D,X_{\mathrm{probe}}}
 \!\left[\sigma_{\boldsymbol B}^3\right]
+c_{\mathrm{opt}}\varepsilon_{\mathrm{opt}}
+c_{\mathrm{fit}}\rho\varepsilon_{\mathrm{fit}}.
\end{aligned}
\]
The constants have the same dependencies as in the homogeneous case,
and the proof is in Appendix~\ref{app:het-upper}.
\end{theorem}

\begin{remark}
\rev{The quadratic bandwidth term uses the exact per-node dither variances and
is at most $\frac{2C_{\mathrm{dist}}}{3K^2}
\mathbb E_{\mathcal D,X_{\mathrm{probe}}}
[\sum_iL_{\ell,i}^2 2^{-2B_i/V}]$. It is \emph{trace-sharpened},
carrying no $V$ and no $\lambda_{\min}$ factor, because
Lemma~\ref{lem:soft-trace} is applied to the aggregated dither
covariance. Setting $B_i\equiv B$ and $L_{\ell,i}\equiv L_\ell$ recovers
Theorem~\ref{thm:fpld-upper}'s homogeneous term; the Euclidean-logit
route would be a factor $V/2$ larger (Appendix~\ref{ssec:trace-upper}).}
\end{remark}

Under a fixed aggregate payload budget $B_{\mathrm{tot}} = \sum_i B_i$,
the optimal allocation across nodes has a closed form once the exact
finite-codebook term is replaced by its exponential envelope. The exact
term contains $(2^{B_i/V}-1)^{-2}$ and has no closed-form optimizer; its
uniform upper envelope is proportional to
$\sum_iL_{\ell,i}^2 2^{-2B_i/V}$, and Theorem~\ref{thm:alloc} minimizes
this surrogate exactly. The statistical and probe-generalization terms do
not depend on the $B_i$.

\begin{theorem}[Optimal bandwidth allocation]\label{thm:alloc}
Fix per-node weights $w_1, \dots, w_K > 0$ and a total budget
$B_{\mathrm{tot}}$, and let $0\le B_{\min}\le B_{\max}\le\infty$ satisfy
$KB_{\min}\le B_{\mathrm{tot}}\le KB_{\max}$. The minimizer of the
weighted exponential surrogate
\[
F(B_1, \dots, B_K) \;=\; \frac{1}{K^2}\sum_{i=1}^{K} w_i \cdot 2^{-2B_i/V}
\quad\text{subject to}\quad \sum_{i=1}^{K} B_i = B_{\mathrm{tot}},\quad
B_{\min}\le B_i\le B_{\max},
\]
is the water-filling allocation
\[
B_i^\star \;=\; \frac{B_{\mathrm{tot}}}{K} \;+\; \frac{V}{2}\log_2\!\left(\frac{w_i}{w_{\mathrm{gm}}}\right),
\qquad w_{\mathrm{gm}} \;:=\; \left(\prod_{j=1}^{K} w_j\right)^{1/K},
\]
\emph{provided neither bound binds}. In general, the unique solution is
  the active-set water-filling
\[
B_i^\star \;=\; \min\!\left\{B_{\max},\,
\max\!\left\{B_{\min},\tfrac{V}{2}\log_2(w_i/\nu)\right\}\right\},
\]
with the level $\nu > 0$ chosen so that
$\sum_i B_i^\star = B_{\mathrm{tot}}$ (a rescaling of the Lagrange
multiplier in Appendix~\ref{app:alloc}); the displayed log-tilted
formula is the special case in which no constraint is active, and
$\nu = w_{\mathrm{gm}} 2^{-2B_{\mathrm{tot}}/(KV)}$ there. Nodes attaining
either boundary are fixed there and the residual budget is redistributed
by re-solving on the remaining active set. For the finite-codebook FPLD
channel take $B_{\min}=V$. When all $w_i$ are equal,
$B_i^\star = B_{\mathrm{tot}}/K$.
\end{theorem}

Theorem~\ref{thm:alloc} is a continuous relaxation. For the integer
per-coordinate rates of Assumption~\ref{ass:A5p}, the exact discrete
optimum is obtained by starting at the lower bounds and repeatedly
assigning the next $V$-bit block to the node with the largest marginal
decrease --- in $F$, or in the exact finite-codebook objective $G_w$ of
Corollary~\ref{cor:adaptive}, which is the ``finite-exact'' policy used
below. Appendix~\ref{app:integer} verifies the required separability
and monotonicity, and records a simpler factor-$4$ rounding guarantee.

\paragraph{Interpretation.} The allocation is $\log_2$-tilted by the weight
ratio: a node with double the weight gets $V/2$ extra bits. In the
theoretical envelope $w_i$ is proportional to the squared clip scale
$L_{\ell,i}^2$, which is read directly from the self-describing codec
header below. Thus the allocation gives more bits to larger-scale
codebooks. This is the same direction as classical
reverse-water-filling for distortion-rate optimization
\citep{gershogray1992}, applied per node rather than per source.

\paragraph{What the weights $w_i$ should be.} The natural weights are
the per-node contributions to the quantization-term constant of
Theorem~\ref{thm:het-upper}, which Theorem~\ref{thm:alloc} otherwise
treats as given. In the common-cap special case $L_{\ell,i}=L_\ell$ of
Assumption~\ref{ass:A5p}, the common weight may be normalized to
$w_i=1$ and the optimal allocation is uniform. If clipping levels differ
across nodes (e.g., a node with a calibrated logit range
$L_{\ell,i}$), then $w_i \propto L_{\ell,i}^2$ and the bit
allocation tilts proportionally.
We use this clip-variance formulation in the experiments
(Section~\ref{sec:exp}). The proof of Theorem~\ref{thm:alloc} follows
from a single Lagrangian calculation (Appendix~\ref{app:alloc}).

\subsection{Self-describing scales and robust allocation}
\label{ssec:adaptive}

The weights in Theorem~\ref{thm:alloc} are not latent distributional
parameters: for the rigorous finite-codebook envelope they are the
squared codec caps, $w_i=L_{\ell,i}^2$. The decoder must know
$L_{\ell,i}$ to reconstruct a logit from a cell index and public
dither, so inferring that scale from the decoded logits would be
circular; Assumption~\ref{ass:A5p} resolves this by making the codec
self-describing.

\paragraph{Exact allocation from the codec header.}
The same $s$-bit header the decoder needs also gives the aggregator the
exact oracle for the exponential envelope, with no warm-up probes,
common clip-to-second-moment ratio, or separate estimation channel. At
the start of a codec epoch, each node sends its grid index $r_i$ before
the payload allocation is selected. The aggregator reads
$L_{\ell,i}=\lambda_{r_i}$, sets $w_i=L_{\ell,i}^2$, and computes the
active-set solution of Theorem~\ref{thm:alloc} or the exact integer
solution described above. If all nodes use one globally known cap, the header is unnecessary and
the oracle is uniform.

The grid introduces a transparent codec-resolution cost, controlled by
the largest log-gap $h_s$ between adjacent grid values: upward rounding
inflates each exact dither variance by at most $2^{2h_s}$ relative to
the unrounded required cap (Appendix~\ref{app:gridres}). This is a
codec-resolution comparison; for the grid codec actually used, the
corresponding weight is observed exactly.

\begin{corollary}[Robust metadata-aware allocation]
\label{cor:adaptive}
Let $\mathcal B$ be any nonempty feasible allocation set, including
the continuous active-set region of Theorem~\ref{thm:alloc} or its
integer-rate subset. For $w,\hat w\in(0,\infty)^K$, let
\[
B^\star\in\arg\min_{B\in\mathcal B}F_w(B),\qquad
\hat B^\star\in\arg\min_{B\in\mathcal B}F_{\hat w}(B),
\]
and suppose
\[
\max_i\left|\log_2\frac{\hat w_i}{w_i}\right|\le\varepsilon.
\]
Then
\[
F_w(\hat B^\star)\le2^{2\varepsilon}F_w(B^\star).
\]
With the exact scale-grid header of Assumption~\ref{ass:A5p}, take
$\hat w=w$ and $\varepsilon=0$, so the allocation equals an oracle allocation.
If instead $|\hat w_i-w_i|\le\eta w_i$ for all $i$ and
$0\le\eta<1$, the alternative bound
\[
F_w(\hat B^\star)\le\frac{1+\eta}{1-\eta}F_w(B^\star)
\]
holds. On an integer finite-codebook feasible set with $B_i/V\in\mathbb N$, both
conclusions remain valid, with the same factors, for the exact
quadratic objective
\[
G_w(B):=\frac1{6K^2}\sum_{i=1}^K
w_i(2^{B_i/V}-1)^{-2}.
\]
No interiority or uniqueness assumption is required.
\end{corollary}

The argument uses only positivity and linearity in the weights, so it
applies verbatim to $G_w$ and to every common feasible set;
Appendix~\ref{app:adaptive} gives the short proof.

Algorithm~\ref{alg:adaptive-fpld} of the supplement states the resulting
metadata-aware allocation procedure end to end.

\subsection{One-round-lag allocation under scale drift}
\label{ssec:streaming}

Scheduling round $t$ from the scale indices decoded in round $t-1$ costs
at most a factor $2^{2\delta_t}$ relative to the current oracle, where
$\delta_t$ bounds the one-step drift of the reported scales. The lag is
forced by the protocol: when codec caps change across rounds their grid
indices must be refreshed to decode the current payload, and in a
low-latency single-message round the allocation is committed before the
current header arrives. Appendix~\ref{app:lagsetup} fixes the observed weights
$w_i(t):=L_{\ell,i}(t)^2$, the feasible set $\mathcal B$, and the
per-round objective $F_t$ and oracle $F_t^\star$ used below.

\begin{assumption}[Bounded reported-scale drift]\label{ass:drift}
For every $i$ and $t\ge2$,
\[
\bigl|\log_2w_i(t)-\log_2w_i(t-1)\bigr|\le\delta_t,
\qquad \delta_t\ge0.
\]
\end{assumption}

Only the increment of the \emph{reported} grid scale is controlled;
the cumulative path length may be unbounded, rankings may cross
arbitrarily often, and no stochastic model for the logits or caps is
needed. Because $w_i=L_{\ell,i}^2$, the condition is equivalent to
$\max\{L_{\ell,i}(t)/L_{\ell,i}(t-1),
L_{\ell,i}(t-1)/L_{\ell,i}(t)\}\le2^{\delta_t/2}$. A coarser scale grid
may make $\delta_t$ larger at a cell crossing, but the theorem below
remains valid for every finite $\delta_t$, including jumps.

The lagged protocol needs one message per node per round. From any
feasible initialization $B(1)$, node $i$ headers its round-$t$ message
with $r_i(t)$ and encodes its payload at the already scheduled
$B_i(t)$; the aggregator decodes with the same cap, stores
$w_i(t)=L_{\ell,i}(t)^2$, and for $t\ge2$ schedules from the most
recently decoded weights,
$B(t)\in\arg\min_{B\in\mathcal B}F_{\hat w(t-1)}(B)$, sent on the
uncharged downlink. Appendix~\ref{app:lagprotocol} states the protocol
in full, including the header-only control phase that instead recovers
the current envelope oracle exactly.

\begin{theorem}[One-round-lag allocation under scale drift]
\label{thm:streaming}
Suppose Assumption~\ref{ass:drift} holds and, for some
$\varepsilon_{\mathrm w}\ge0$, the weights available to the scheduler obey
\[
\max_i\left|\log_2\frac{\hat w_i(t-1)}{w_i(t-1)}\right|
\le\varepsilon_{\mathrm w}.
\]
Then, for every $t\ge2$, every one-round-lag minimizer described above satisfies
\[
F_t\bigl(B(t)\bigr)
\le 2^{2(\delta_t+\varepsilon_{\mathrm w})}F_t^\star.
\]
In particular, exact scale-grid metadata gives
\[
F_t\bigl(B(t)\bigr)\le2^{2\delta_t}F_t^\star
=\{1+2(\ln2)\delta_t+O(\delta_t^2)\}F_t^\star.
\]
Thus if $\delta_t\le\delta$ uniformly, the factor is at most
$2^{2\delta}=1+O(\delta)$.
The statement holds for continuous active-set and integer allocations,
without interiority or uniqueness. It also holds verbatim for the exact
finite-codebook objective of Corollary~\ref{cor:adaptive} when the lagged
scheduler minimizes $G_{\hat w(t-1)}$ and the current oracle minimizes
$G_{w(t)}$. With a current-round control header, the scheduler minimizes
using $w(t)$ itself, so the comparison has zero log-weight error and
returns an exact current-oracle allocation.
\end{theorem}

The bound is per round and pathwise. It has no warm-up, gain, burn-in,
horizon, or probe-sample term, and it covers abrupt changes through the
realized $\delta_t$. If the scale is stationary, exact metadata gives
the oracle from the first post-registration round. Its only price is
the one-message timing lag; the control-phase variant removes even
that price at the cost of an additional scheduling exchange.

\begin{proposition}[Failure of the frozen allocation]\label{prop:staticfail}
Suppose the water-filling allocations based on $w(t_0)$ and $w(t)$
are interior. Let $\bar B = \mathrm{WF}(w(t_0))$ be held fixed for
$t > t_0$, and
let $\Delta_i(t) := \log_2 w_i(t) - \log_2 w_i(t_0)$,
$\bar\Delta(t) := \tfrac1K\sum_j \Delta_j(t)$, and define the current
arithmetic and geometric means
$w_{\mathrm{am}}(t):=K^{-1}\sum_i w_i(t)$ and
$w_{\mathrm{gm}}(t):=(\prod_i w_i(t))^{1/K}$. Then the exact identity is
\[
\frac{F_t(\bar B)}{F_t^\star}
= \frac{1}{K}\sum_{i=1}^K 2^{\,\Delta_i(t) - \bar\Delta(t)} \;\ge\; 1
\qquad\text{(with equality exactly when all $\Delta_i(t)$ are equal),}
\]
while $F_t(\mathrm{uniform})/F_t^\star = w_{\mathrm{am}}(t)/w_{\mathrm{gm}}(t)$, the
arithmetic-to-geometric-mean ratio of the current weights. Hence the
frozen rule is worse than uniform if and only if
\[
\frac1K\sum_{i=1}^K2^{\Delta_i(t)-\bar\Delta(t)}
>\frac{w_{\mathrm{am}}(t)}{w_{\mathrm{gm}}(t)}.
\]
A crossover in the weight ranking is, in general, neither necessary nor sufficient
for this inequality, even when $K=3$; the criterion, rather
than crossover alone, is the exact statement.
\end{proposition}

Section~\ref{sec:exp} reports exact finite-codebook paths for which the
one-round-lag ratio is checked directly against the realized drift bound.

\section{Unconditional Rates and the Price of Scalar Averaging}\label{sec:uncond}

Over arbitrary $B$-bit-per-node protocols, \emph{no matching
$K^{-1}2^{-2B/V}$ lower bound can exist}: the node-specific
global-index protocol is exponentially better in $K$ on the bandwidth
axis, the correct unconditional rate is exponential in the \emph{total}
budget $KB$, and the gap to the scalar-averaging FPLD rate is a
quantifiable price of that protocol structure.
Theorem~\ref{thm:matching-lb} and Proposition~\ref{prop:rebuttal-lb}
pin the rate $\Theta(K^{-1}2^{-2B/V})$ only \emph{within} the
bounded-resolution unbiased scalar quantize-and-average class under
their stated spread and small-error conditions, leaving the
unconditional problem open; this section 
\rev{treats} its single-round
case. Proofs are in Appendix~\ref{app:uncond}.

\paragraph{The noiseless-logit subproblem.} To isolate the bandwidth
axis, fix one probe context and let every node know the population
logit exactly ($n = \infty$):
$\ell^\star \in \mathcal L := \{\ell \in [-L_\ell, L_\ell]^V \colon
\mathbf 1^\top \ell = 0\}$, $P^\star = \mathrm{softmax}(\ell^\star)$.
For $K,B\in\mathbb N$, a protocol $\Pi \in \mathcal P_B$ consists of
arbitrary per-node encoders
$\phi_i : \mathcal L \times U \times U_i \to \{0,1\}^B$
(shared public randomness $U$, node-private randomness $U_i$; the
encoders may \emph{differ} across nodes) and an arbitrary, possibly
randomized, decoder producing
$\hat P \in \Delta(\mathcal V)$. Write
$\mathcal L_{c_p}:=\{\ell\in\mathcal L:
1-\|\mathrm{softmax}(\ell)\|_2^2\ge c_p\}$ and
\[
\mathfrak M(K,B,V):=\inf_{\Pi\in\mathcal P_B}
\sup_{\ell^\star\in\mathcal L_{c_p}}
\mathbb E[\mathrm{KL}(P^\star\|\hat P)].
\]

\begin{theorem}[Unconditional lower bound]\label{thm:uncond}
Suppose $V \ge 2$ and $e^{L_\ell}/V \le 1 - c_p$. Then for all
$K,B\in\mathbb N$,
\[
\mathfrak M(K, B, V)
\;\ge\;
\frac{L_\ell^2\, e^{-2L_\ell}}{64\, V^2\, (KB+2)}\;
2^{-\frac{2(KB+3)}{V-1}}.
\]
The exponent is governed by the total budget $KB$: extra nodes can
substitute for extra bits in the exponent.
\end{theorem}

\begin{proposition}[Global-index splitting achievability]\label{prop:vqsplit}
For every $V\ge2$ and $K,B\in\mathbb N$, there is a deterministic protocol in
$\mathcal P_B$ that covers the centered logit ball with one global
codebook, assigns the resulting $KB$-bit codeword index across the
$K$ node messages, and satisfies
\[
\sup_{\ell^\star\in\mathcal L}
\mathrm{KL}(P^\star\|\hat P)
\;\le\; \frac92\,V L_\ell^2\,
2^{-2KB/(V-1)}.
\]
The statement includes the small-budget regime; there the singleton
codebook already obeys the displayed bound.
\end{proposition}

\begin{corollary}[Two-sided unconditional rate]\label{cor:twosided}
Under the hypotheses above, for every $K,B\in\mathbb N$ the noiseless-logit
minimax rate satisfies
\[
\frac{L_\ell^2 e^{-2L_\ell}}{64 V^2 (KB+2)}\,
2^{-\frac{2(KB+3)}{V-1}}
\;\le\; \mathfrak M(K, B, V) \;\le\;
\frac92 V L_\ell^2\, 2^{-\frac{2KB}{V-1}}.
\]
Thus the bounds have matching exponential dependence on $KB/(V-1)$;
the remaining multiplicative gap is polynomial in $(V,K,B)$ and
exponential in $L_\ell$.
\end{corollary}

\paragraph{Remark (finite-sample categorical counterpart).}
The lower bound transfers to the finite-sample categorical problem, but
the achievability does not. \rev{A node that knows $\ell^\star$ can simulate
$n$ i.i.d.\ draws and feed them to any finite-$n$ encoder, so
Theorem~\ref{thm:uncond}'s lower bound holds verbatim for that counterpart; the
global-index protocol needs all nodes to agree on one codeword index,
which independent estimates need not do
(Appendix~\ref{app:catcounterpart}).}

\begin{proposition}[Price of bounded-resolution scalar averaging]\label{prop:symmetry}
Fix a conditional predictor $P^\star$ satisfying \textbf{(QS-H)} of
Proposition~\ref{prop:rebuttal-lb}, let $b=B/V\in\mathbb N$, and
suppose the class resolution and rate satisfy \textbf{(RES)} while
the dithered upper bound satisfies \textbf{(SE$_{c_0}$)}. If
$R_{\mathrm{sc}}^\star(P^\star)$ is the infimum of
Definition~\ref{def:bandwidth}'s risk over the bounded-resolution
unbiased scalar quantize-and-average class, then
\[
\frac{c_H}{4K}2^{-2b}
\;\le\;R_{\mathrm{sc}}^\star(P^\star)
\;\le\;\frac{C_q}{K}2^{-2b}.
\]
By contrast, applying Proposition~\ref{prop:vqsplit}'s universal
global-index codebook separately at every context gives an unrestricted
protocol whose risk on the same $P^\star$ is at most
$\frac92VL_\ell^2 2^{-2KB/(V-1)}$. Thus, at fixed $b$ and $V$, the
scalar-class lower bound divided by this unrestricted upper bound grows
exponentially in $K$. Under Theorem~\ref{thm:uncond}'s additional
non-degeneracy hypothesis, Corollary~\ref{cor:twosided} shows that the
global-index exponent is also minimax-optimal for the noiseless
one-context problem. Whether an analogous separation holds against the
full symmetric class --- arbitrary identical randomized encoders with a
permutation-invariant decoder --- remains open.
\end{proposition}

The global-index construction is fragile to node dropout and to
finite-sample disagreement across nodes. Coordinate splitting is a robust
finite-$n$ alternative, but forfeits the $K$-fold statistical averaging
($d/n$ in place of $d/(Kn)$). Appendix~\ref{app:groupedstmt} gives a
grouped protocol interpolating between these endpoints
(Theorem~\ref{thm:grouped}, spanning FPLD at $G=1$ and full splitting at
$G=K$), with its envelope-optimal group size
(Corollary~\ref{cor:groupopt}).

\section{Numerical Illustrations}\label{sec:exp}

The controlled experiments below use the exact $2^b$-point codecs of
Assumptions~\ref{ass:A5} and~\ref{ass:A5p}. The dithered step is
$\Delta=2L/(2^b-1)$, allocation policies spend the same integer payload,
and reported caps are rounded upward on the public scale grid when scale
metadata is used. Synthetic bands are mean $\mathbin{\pm}1.96$ Monte Carlo
standard errors over independent codec seeds. Appendix~\ref{app:rebuttal-reallm} gives the
configurations and provenance for the reported studies.

\subsection{Exact finite-codebook checks}

\paragraph{Single-round rate.} Empirical KL tracks the quadratic
prediction with no fitted constants, recovering both the $K^{-1}$ law
and the finite-rate factor (Figure~\ref{fig:fig1}). The study fixes a
centered target logit in $[-2,2]^{256}$, dithers its transmission
independently at each node, and compares empirical KL to the computed
quadratic prediction $\Delta^2\{1-\|p^\star\|_2^2\}/(24K)$. At $b=4$, increasing $K$ from
$1$ to $16$ decreases mean KL from $2.97\times10^{-3}$ to
$1.82\times10^{-4}$. At $K=4$, the bit sweep follows the finite factor
$(2^b-1)^{-2}$, including its low-rate curvature; from $b=3$ onward the
log-base-two slope is essentially $-2$.

\begin{figure}[!tbp]
\centering
\includegraphics[width=0.95\linewidth]{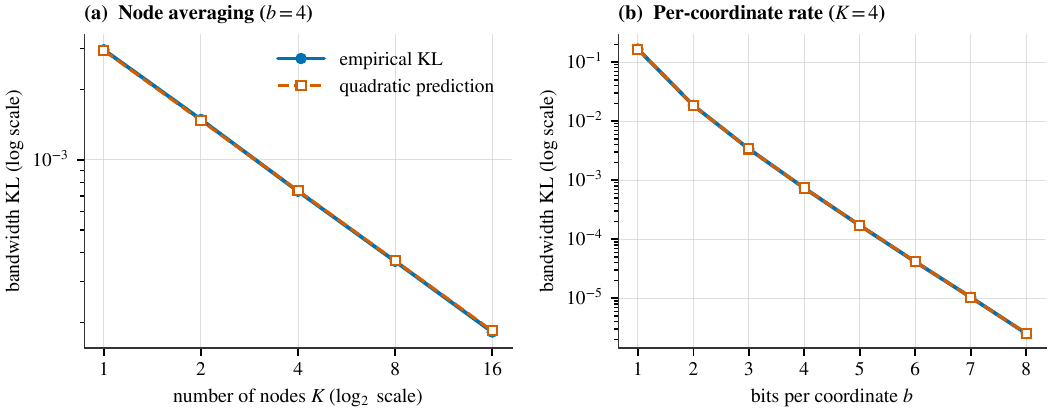}
\caption{Empirical KL tracks the trace-sharpened quadratic prediction
without fitted constants, recovering both the $K^{-1}$ node-averaging law
and the finite-rate $(2^b-1)^{-2}$ law. Exact finite-codebook bandwidth
isolation with $V=256$ and $100$ codec seeds: panel (a) varies $K$ at
$b=4$ bits per coordinate, and panel (b) varies $b$ at $K=4$. The dashed,
square-marked curve is computed from the target softmax and exact dither
variance; bands are mean $\mathbin{\pm}1.96$ Monte Carlo standard errors.}
\label{fig:fig1}
\end{figure}

\paragraph{Allocation, vocabulary, and multi-round checks.} Three further
exact-channel studies, reported in Appendix~\ref{app:synthetic-extra},
confirm the allocation rule, the dimension-free approximation, and the
multi-round contraction. With $K=4$ and caps $L=(1,1,4,4)$, finite-exact allocation cuts KL by $30.4\%$
at $B_{\mathrm{tot}}/V=8$ and $52.9\%$ at $20$. A controlled $2{,}374$-fold
vocabulary sweep holds empirical KL within $0.83\%$ of its quadratic
prediction; this, not the real-model comparison below, is the
dimension-isolating check. Residual refinement at two bits per round drives
KL from $1.85\times10^{-2}$ at $T=1$ to $3.10\times10^{-7}$ at $T=6$.

\paragraph{Scale-metadata check.} The one-round-lag scheduler of
Theorem~\ref{thm:streaming} stays well inside its pathwise guarantee on
a $200$-round exact-codebook path. The largest observed one-step
log-weight drift is $0.25$ and the largest lag/current objective ratio
is $1.070$, below the guarantee $2^{2(0.25)}=1.414$. Configuration, header accounting, and the full path are
in Appendix~\ref{app:rebuttal-reallm}.

\section{A Real Language-Model Federation}\label{sec:reallm}

\paragraph{Setup.} The allocation study uses a vocabulary-aligned
federation of seven public checkpoints: five GPT-2-family checkpoints from
DistilGPT-2 through GPT-2-XL, together with GPT-Neo-125M/1.3B. They share the GPT-2
vocabulary ($V=50{,}257$) but differ in scale and training corpus. On
$2{,}000$ fixed WikiText-2 contexts~\citep{merity2016pointer} of length $128$, their centered
last-position logit caps span $33.62$--$76.66$, so the allocation weights
$w_i=L_{\ell,i}^2$ span a factor of $5.2$. The unquantized aggregate
teacher has token perplexity $19.874$. The study is inference-only: each
node transmits quantized logits and the aggregator emits the softmax of
their average; no auxiliary student is trained.

\paragraph{Practical continuous-rate allocation.} At equal nominal
payload, water-filling lowers both probe KL and token perplexity on both
text streams, with the largest gains at the tightest budget. The study is
a prespecified cache-only confirmation of an earlier three-seed pilot
and uses the practical subtractively dithered lattice with step
$2L_{\ell,i}/2^{b_i}$ and continuous per-node rates. It is an uncapped
nominal-rate lattice, not the exact $2^{b_i}$-point channel of
Assumption~\ref{ass:A5p}, and is reported as a practical allocation
experiment rather than a calibration of finite-codebook constants. The
two policies spend the same total nominal payload and use paired dither
arrays. On WikiText-2, the designated primary comparison at
$B_{\mathrm{tot}}/V=14$ reduces probe KL by $10.35\%$
$[9.90\%,10.80\%]$ and token perplexity by $27.78\%$
$[25.50\%,29.96\%]$. On the independently locked LAMBADA text
stream~\citep{paperno2016lambada}, the
corresponding reductions are $12.26\%$ $[11.82\%,12.70\%]$ and $31.05\%$
$[28.82\%,33.21\%]$. The brackets report paired crossed-bootstrap percentile
intervals that independently resample the codec-seed and observed-context
axes, conditional on the locked checkpoints and source streams. All $50$
seed-level means favor water-filling for both metrics at this rate.
A post hoc paired circular-block bootstrap with blocks of up to $256$
consecutive contexts leaves the perplexity conclusion unchanged
(Appendix~\ref{app:rebuttal-reallm}).

\begin{table}[!htbp]
\centering\small
\setlength{\tabcolsep}{7pt}
\renewcommand{\arraystretch}{1.08}
\caption{Water-filling lowers probe KL and token perplexity at every displayed
budget on both streams; gains are largest at $B_{\mathrm{tot}}/V=14$, while
rate-$35$ values approach the full-precision teachers. Entries average $50$
paired dither seeds and $2{,}000$ contexts for the seven-model GPT federation
under equal total nominal payload. Teacher perplexities are $19.874$ on
WikiText-2 and $27.781$ on the LAMBADA stream; LAMBADA uses consecutive text
windows rather than passage-final evaluation. Paired crossed-bootstrap
intervals appear in the text and
Appendix~\ref{app:rebuttal-reallm}.}
\label{tab:reallm-main}
\begin{tabular}{llrrrr}
\toprule
& & \multicolumn{2}{c}{probe KL} & \multicolumn{2}{c}{token perplexity}\\
\cmidrule(lr){3-4}\cmidrule(lr){5-6}
probe & $B_{\mathrm{tot}}/V$ & uniform & water-filling & uniform & water-filling\\
\midrule
WikiText-2 & 14 & 3.043 & 2.728 & 416.57 & 300.85\\
& 21 & 0.8094 & 0.7143 & 44.53 & 40.55\\
& 28 & 0.2061 & 0.1815 & 24.44 & 23.79\\
& 35 & 0.05174 & 0.04541 & 20.91 & 20.78\\
\addlinespace
LAMBADA stream & 14 & 3.029 & 2.658 & 584.88 & 403.26\\
& 21 & 0.8169 & 0.7048 & 62.37 & 55.94\\
& 28 & 0.2091 & 0.1792 & 34.32 & 33.37\\
& 35 & 0.05271 & 0.04525 & 29.34 & 29.05\\
\bottomrule
\end{tabular}
\end{table}

\rev{At $B_{\mathrm{tot}}/V=35$, five nominal bits per coordinate per node,
the water-filling perplexities are within $4.55\%$ and $4.58\%$ of the
unquantized teachers (uniform: $5.23\%$ and $5.60\%$);
Appendix~\ref{app:rate35} checks the exponential surrogate at this
interior rate.}

\paragraph{Exact-codebook confirmation.} \rev{Under the strict finite
codebook, with integer rates, caps rounded upward on the public grid, and
$20$ dither seeds, optimized allocation lowers probe KL at every displayed
rate where it differs from uniform: by $3.70\%$ $[2.96\%,4.41\%]$ and
$5.33\%$ $[4.52\%,6.09\%]$ at $B_{\mathrm{tot}}/V=28$ and $35$ on
WikiText-2, and by $3.02\%$ $[2.46\%,3.59\%]$ and $6.51\%$
$[5.80\%,7.24\%]$ at $21$ and $28$ on LAMBADA
(Table~\ref{tab:reallm-alloc}; descriptive paired $95\%$ intervals).}
\rev{These gains are smaller than in the continuous-rate study because
integer rates move only in whole bits per coordinate: at these budgets
the optimum departs from uniform at two of the seven nodes.} \rev{The one-time
four-bit scale header costs $0.014$ bits per probe context in total.}

\rev{\paragraph{Larger models.} The bit-rate slope, the trace-sharpened
prediction, and the multi-round gain hold for Pythia-2.8B
\citep{biderman2023pythia} and Falcon3-3B Base \citep{tii2024falcon3}
as they do for GPT-2. Each model is replicated across nodes with
independent dithers on the exact finite codebook, so no cross-family
logits are averaged. Across the six model--probe pairs, fitted slopes of
$\log_2(\mathrm{KL})$ against $b\ge3$ range from $-2.068$ to $-1.931$,
and at $K=8,b=8$ empirical KL is $0.999$--$1.006$ times the
trace-sharpened prediction (Table~\ref{tab:reallm-rate} in the
supplement). With two bits per coordinate per round, residual
refinement reduces KL from $7.62$ to $1.72\times10^{-4}$ by $T=6$ for
GPT-2 and by a factor of about $6{,}000$ by $T=5$ for both larger
models, whose native output dimension reaches $131{,}072$.}

\par\addvspace{2ex}
\noindent\begin{minipage}{\linewidth}
\paragraph{Takeaway.} \rev{At equal payload, scale-aware water-filling
improves KL and perplexity on both streams, most when bandwidth is
scarce, and the bit-rate slope, the trace-sharpened approximation, and
the multi-round gain persist up to $3.2$ billion parameters: register
scale, allocate scarce bits heterogeneously, and refine residuals across
rounds.}
\end{minipage}

\section{Discussion}\label{sec:discuss}

\rev{For the bandwidth-isolated scalar-quantized logit problem over a shared
probe set, the single-round rate is matched
(Theorem~\ref{thm:matching-lb}) and the multi-round rate is achievable
(Theorem~\ref{thm:lb-mr}), paralleling the program for
communication-constrained distributed estimation
\citep{zhangduchiwainwright2013, huanghuo2019,
acharyacanonnetyagi2020, gargmanguyen2014}. Over unrestricted protocols
the single-round bounds have matching exponential dependence on
$KB/(V-1)$ (Section~\ref{sec:uncond}), but the global codeword attaining
them needs exact cross-node agreement and is dropout-fragile; this is why the
construction-specific bound of Theorem~\ref{thm:matching-lb} is the
operative one, and why the grouped protocol of Theorem~\ref{thm:grouped}
interpolates toward coordinate splitting. Open questions are unrestricted
multi-round lower bounds, a joint $(n,B)$ unconditional lower bound, the
price of symmetry against the full symmetric class, composing grouped aggregation
with within-group water-filling, and non-parametric students.}

\begin{revblock}
\paragraph{From teacher to student.} Section~\ref{sec:reallm}
measures the aggregate teacher, whereas Theorem~\ref{thm:fpld-upper}
bounds a student, and the allocation changes only its bandwidth term.
Students distilled on the exact channel
(Appendix~\ref{app:practical-student}) show that, when the fit
converges, water-filling removes the same fraction of the student's
quantization-attributable KL as of the teacher's (about $30\%$ at
$B_{\mathrm{tot}}/V=8$ and $51\%$ at $16$), even for a misspecified
student; an early-stopped fit retains about half that fraction. A
student with $p+1$ parameters per coordinate fitted on $m$ probes
inherits about $1/m+p/(m-p-2)$ of the teacher-level quantization KL, so
the overall gain is the teacher-level gain times the quantization share
of the student's risk: large under tight budgets or for students
large relative to $m$, and small when the student's own error dominates.

\paragraph{Below one bit per coordinate.} When $V>10^5$ even $b=1$ costs
$V$ bits per node per probe. Because the second-order KL weighs
coordinate $v$ by $p_v(1-p_v)$, Theorem~\ref{thm:alloc} applied across
coordinates gives zero bits to low-probability coordinates: top-$k$
sparsification with $k$ and the depths set by the water level, using
weights from the broadcast student. On synthetic $V=131{,}072$ data it attains lower KL
with $1/128$ bit per coordinate than the dense one-bit codebook
(Appendix~\ref{app:practical-subbit}). Entropy coding of the dithered
indices is lossless: it changes no distortion bound and lowers only the
expected bit count.

\paragraph{Heterogeneous sites.} The quantization increment and the
allocation results (Theorems~\ref{thm:alloc} and~\ref{thm:streaming},
Corollary~\ref{cor:adaptive}) are conditional on the transmitted logits,
so they hold for sites that differ, as the members of the federation of
Section~\ref{sec:reallm} do. What changes is the target and the
statistical term: averaged logits estimate the logarithmic pool of the
site distributions, the probe law must cover every site's contexts
through $\rho$~\citep{shimodaira2000}, and pooling weights $a_i$ make the
allocation weights $a_i^2L_{\ell,i}^2$. Partial-pooling and transfer
estimators~\citep{duan2022,li2022,li2023} can replace the equal-weight pool
(Appendix~\ref{app:practical-noniid}).
\end{revblock}

\rev{Together the results give a finite-codebook rate theory for the uplink
of federated distillation, which deployed
systems~\citep{fedmd2019,feddf2020,fedkdsony2022,fedfd2025} and
their variants~\citep{sattlermarbanrischkesamek2022,liuzhangsongletaief2024}
lack: where the next bit should go, how it should be spread over rounds,
what scale metadata a message must carry, and when further scalar-encoder
engineering cannot improve the rate.}

\if1\anon
{
\section*{Funding}
Dubey acknowledges partial support from the Stewart Topper Fellowship at
Georgia Tech. Huo was partially supported by the National Science Foundation
under Grant 2229876, the A.~Russell Chandler III Professorship, the
NIH-sponsored Georgia Clinical and Translational Science Alliance, and the
Georgia Department of Transportation.
} \fi

\section*{Declaration of Generative AI Use}
We used generative AI tools (ChatGPT, Google Gemini, and Claude) for language
editing and code formatting support only. All data, results, and mathematical
derivations are the authors' own work.

\bibliography{references}

\clearpage
\setcounter{figure}{0}
\setcounter{table}{0}
\setcounter{algorithm}{0}
\setcounter{equation}{0}
\setcounter{footnote}{0}
\renewcommand{\thefigure}{S\arabic{figure}}
\renewcommand{\thetable}{S\arabic{table}}
\renewcommand{\thealgorithm}{S\arabic{algorithm}}
\renewcommand{\theequation}{S\arabic{equation}}
\def\theHfigure{S\arabic{figure}}
\def\theHtable{S\arabic{table}}
\def\theHalgorithm{S\arabic{algorithm}}
\def\theHequation{S\arabic{equation}}

\spacingset{1}

  \bigskip
  \begin{center}
    {\LARGE\bf \leavevmode\phantomsection\addcontentsline{toc}{section}{Supplementary Material}%
      Supplementary Material for\\
      ``Matching Rates and Optimal Allocation for Federated
      Probe-Logit Distillation under Heterogeneous Bandwidth Budgets''}
  \end{center}
  \medskip

\spacingset{1.8}

\noindent
This supplement contains the assumption block, the full proofs, and the
supporting experimental detail for the manuscript. Numbering continues
that of the manuscript: references of the form ``Theorem~3.1'' point to
the main text, and all equation, figure, table, and algorithm numbers
prefixed by ``S'' are internal to this supplement. Assumptions and
definitions carry their own counters, so a number such as
``Assumption~A.1'' names an assumption rather than the section with the
same number.

\appendix

\section{Assumption Block}\label{app:assumptions}

\begin{figure}[!htbp]
\centering
\begingroup
\definecolor{fpldblue}{HTML}{0072B2}
\definecolor{fpldgreen}{HTML}{009E73}
\definecolor{fpldorange}{HTML}{D55E00}
\definecolor{fpldcharcoal}{HTML}{333333}
\begin{tikzpicture}[
  font=\scriptsize,
  box/.style={
    draw=fpldcharcoal,
    line width=0.65pt,
    rounded corners=1.2pt,
    align=center,
    inner sep=4pt,
    minimum height=1.25cm
  },
  bluebox/.style={box,draw=fpldblue,fill=fpldblue!5},
  greenbox/.style={box,draw=fpldgreen,fill=fpldgreen!5},
  orangebox/.style={box,draw=fpldorange,fill=fpldorange!5},
  nodebadge/.style={
    circle,
    draw=fpldblue,
    fill=white,
    line width=0.65pt,
    minimum size=0.84cm,
    inner sep=0pt
  },
  flow/.style={-{Stealth[length=2.1mm,width=1.35mm]},draw=fpldcharcoal,line width=0.8pt},
  control/.style={-{Stealth[length=2.1mm,width=1.35mm]},draw=fpldgreen,dashed,line width=0.85pt},
  feedback/.style={-{Stealth[length=2.1mm,width=1.35mm]},draw=fpldorange,densely dotted,line width=1.0pt},
  arrowlabel/.style={font=\scriptsize,fill=white,inner sep=1.5pt,text=fpldcharcoal},
  phaselabel/.style={font=\bfseries\footnotesize,anchor=west}
]
\path[use as bounding box] (0,-2.55) rectangle (15.80,7.25);

\node[phaselabel,text=fpldblue] at (0.10,7.04)
  {(1) Problem setup};

\node[bluebox,minimum width=4.65cm,minimum height=2.05cm] (nodes) at (2.55,5.65) {};
\node[font=\bfseries\scriptsize,text=fpldblue] at ($(nodes.north)+(0,-0.26)$)
  {Nodes and data};
\node[nodebadge] (nodeone) at (1.16,5.63) {$\widehat P_1$};
\node[nodebadge] (nodei) at (2.55,5.63) {$\widehat P_i$};
\node[nodebadge] (nodek) at (3.94,5.63) {$\widehat P_K$};
\node at (1.86,5.63) {$\cdots$};
\node at (3.24,5.63) {$\cdots$};
\node at (1.16,4.98) {$\mathcal D_1$};
\node at (2.55,4.98) {$\mathcal D_i$};
\node at (3.94,4.98) {$\mathcal D_K$};

\node[greenbox,text width=4.05cm,minimum height=2.05cm] (interface) at (7.65,5.65) {
  \textbf{Public probes and channel}\\[1.5pt]
  $X_{\mathrm{probe}}=\{x^{(l)}\}_{l=1}^m\sim Q^m$\\
  payload uplink cap: $B_i$ bits\\
  per context and round\\
  downlink: uncharged
};

\node[orangebox,text width=4.45cm,minimum height=2.05cm] (target) at (13.00,5.65) {
  \textbf{Estimand}\\[1.5pt]
  $\widehat P:\mathcal X\to\Delta(\mathcal V)$\\
  $R(\widehat P)=\mathbb E_{X\sim P_X^\star}$\\[-1pt]
  $\mathrm{KL}\!\left(P^\star(\cdot\mid X)\,\|\,\widehat P(\cdot\mid X)\right)$
};

\node[font=\normalsize,text=fpldcharcoal] at (5.33,5.65) {$+$};
\node[font=\normalsize,text=fpldcharcoal] at (10.23,5.65) {$\Longrightarrow$};

\node[phaselabel,text=fpldgreen] at (0.10,3.87)
  {(2) Heterogeneous codec setup (once per codec epoch)};

\node[greenbox,text width=4.00cm] (register) at (2.55,2.66) {
  \textbf{Node $i$: register scale}\\[1.5pt]
  choose $L_{\ell,i}^{\mathrm{req}}$; send the $s$-bit index $r_i$
};

\node[greenbox,text width=3.55cm] (scale) at (7.15,2.66) {
  \textbf{Server: decode scale}\\[1.5pt]
  $L_{\ell,i}=\lambda_{r_i}$; $w_i=L_{\ell,i}^{2}$
};

\node[greenbox,text width=4.65cm] (schedule) at (12.90,2.66) {
  \textbf{Server: schedule payload}\\[1.5pt]
  active-set water-filling or exact integer allocation\\
  $\sum_i B_i=B_{\mathrm{tot}}$
};

\draw[control] (register.east) -- node[arrowlabel,above] {$r_i$} (scale.west);
\draw[control] (scale.east) -- node[arrowlabel,above] {$w_i$} (schedule.west);

\node[phaselabel,text=fpldblue] at (0.10,1.24)
  {(3) Repeated FPLD round ($t=1,\ldots,T$)};

\node[bluebox,text width=4.65cm] (local) at (2.60,-0.32) {
  \textbf{Node $i$: update and probe}\\[1.5pt]
  fine-tune $\widehat P^{(t-1)}$ on private $\mathcal D_i$\\
  $x^{(l)}\mapsto$ centered logit $\ell_i^{(t,l)}$
};

\node[greenbox,text width=3.50cm] (codec) at (7.75,-0.32) {
  \textbf{Finite payload codec}\\[1.5pt]
  $b_i=B_i/V$; public dither\\
  \textsc{ScalarQuantize} at $L_{\ell,i}$\\
  $\ell_i^{(t,l)}\mapsto\widetilde\ell_i^{(t,l)}$
};

\node[orangebox,text width=4.45cm] (server) at (13.05,-0.32) {
  \textbf{Server: average and distill}\\[1.5pt]
  $\overline\ell=K^{-1}\sum_i\widetilde\ell_i$\\
  teacher $\operatorname{softmax}(\overline\ell)$; distill $\widehat P^{(t)}$
};

\draw[flow] (local.east) -- node[arrowlabel,above] {$\ell_i^{(t,l)}$} (codec.west);
\draw[flow] (codec.east) -- node[arrowlabel,above] {$\widetilde\ell_i^{(t,l)}$} (server.west);
\draw[control] (schedule.south) -- node[arrowlabel,right] {$B_i$} ++(0,-0.35) -|
  (codec.north);

\coordinate (returnright) at (13.05,-2.05);
\coordinate (returnleft) at (2.60,-2.05);
\draw[feedback] (server.south) -- (returnright) --
  node[arrowlabel,below=2pt] {uncharged downlink: broadcast $\widehat P^{(t)}$ for round $t+1$}
  (returnleft) -- (local.south);
\end{tikzpicture}
\endgroup
\caption{FPLD keeps raw data and model parameters local while distilling a
global student from quantized probe logits. The optional middle layer
registers node-specific scales and allocates heterogeneous payloads; each
round then updates local models, quantizes probe logits, averages them, and
distills the next student. Dashed green arrows carry control messages, and the
dotted orange path is the uncharged broadcast. With a globally known common
cap, omit registration and set $B_i=B$. Sequential refinement and unrestricted
protocols are analyzed separately.}
\label{fig:fpld-overview}
\end{figure}

\begin{algorithm}[!htbp]
\caption{Federated Probe-Logit Distillation (FPLD).}
\label{alg:fpld}
\begin{algorithmic}[1]
\Require rounds $T$; local epochs $E$; per-vector uplink budget $B$ bits
  with $b:=B/V\in\mathbb N$, $b\ge1$;
  probe set $X_{\mathrm{probe}} = \{x^{(l)}\}_{l=1}^m$; shared base
  $\hat P^{(0)}$.
\For{$t = 1, \dots, T$}
  \For{each node $i = 1, \dots, K$ \textbf{in parallel}}
    \State $\hat P_i^{(t)} \gets \textsc{FineTune}(\hat P^{(t-1)}, \mathcal{D}_i, E)$
      \Comment{local optimizer; App.~\ref{app:fpld-operators}}
    \For{$l = 1, \dots, m$}
      \State $\ell_i^{(t,l)} \gets
        \textsc{CanonicalLogits}(\hat P_i^{(t)},x^{(l)})$
        \Comment{centered log-probabilities; App.~\ref{app:fpld-operators}}
      \State $\tilde\ell_i^{(t,l)} \gets
        \textsc{ScalarQuantize}(\ell_i^{(t,l)};b,[-L_\ell,L_\ell])$
        \Comment{clip/dither; App.~\ref{app:fpld-operators}}
    \EndFor
    \State uplink $\{\tilde\ell_i^{(t,l)}\}_{l=1}^m$ to aggregator
  \EndFor
  \State $\bar\ell^{(t,l)} \gets \tfrac{1}{K}\sum_i \tilde\ell_i^{(t,l)}$ for all $l$
  \State $\hat P^{(t)} \gets$ any $\varepsilon_{\mathrm{fit}}$-approximate minimizer over $P \in \mathcal{F}_\Theta$ of $\sum_l \mathrm{KL}(\mathrm{softmax}(\bar\ell^{(t,l)}) \,\|\, P(\cdot \mid x^{(l)}))$, as in Assumption~\ref{ass:A6}
  \State broadcast $\hat P^{(t)}$ to all nodes
\EndFor
\State \Return $\hat P \gets \hat P^{(T)}$
\end{algorithmic}
\end{algorithm}

\subsection{Operators used in Algorithm~\ref{alg:fpld}}
\label{app:fpld-operators}

For $\mathcal D_i=\{(x_j^{(i)},y_j^{(i)})\}_{j=1}^n$, define the local
negative log-likelihood
\[
\mathcal L_i(\theta):=-\frac1n\sum_{j=1}^n
  \log p_\theta(y_j^{(i)}\mid x_j^{(i)}),
\qquad
\hat\theta_i^{\mathrm{MLE}}\in\arg\min_{\theta\in\Theta}\mathcal L_i(\theta).
\]
The operator
$\textsc{FineTune}(P_{\theta_0},\mathcal D_i,E):=P_{\hat\theta_i}$
denotes the prespecified local optimizer initialized at $\theta_0$, run for
$E$ epochs on $\mathcal D_i$, and constrained to return
$\hat\theta_i\in\Theta$. The theory is optimizer-agnostic beyond
Assumption~\ref{ass:A4}, which bounds the mean-squared distance from
$\hat\theta_i^{\mathrm{MLE}}$.

For $P_\theta\in\mathcal F_\Theta$ and $x\in\mathcal X$,
\[
\textsc{CanonicalLogits}(P_\theta,x)_v
:=\log p_\theta(v\mid x)-\frac1V\sum_{w\in\mathcal V}
\log p_\theta(w\mid x),
\qquad v\in\mathcal V.
\]
This fixes the additive softmax gauge and agrees with the definition in
Section~\ref{sec:setup}. To define \textsc{ScalarQuantize}, clip each input
coordinate to $z_v\in[-L,L]$. For $L>0$, let $M=2^b$,
$\Delta=2L/(M-1)$, $c_j=-L+j\Delta$, and draw the public dithers
$D_v\sim\mathrm{Unif}[-\Delta/2,\Delta/2]$ independently. The encoder sends
the $b$-bit index
\[
j_v\in\arg\min_{0\le j<M}|z_v+D_v-c_j|,
\]
with ties resolved by the half-open cell convention in
Assumption~\ref{ass:A5}; the aggregator reconstructs
$\textsc{ScalarQuantize}(\ell;b,[-L,L])_v:=c_{j_v}-D_v$.
Thus the payload contains the indices $(j_v)_{v\in\mathcal V}$, whereas
$\tilde\ell$ in Algorithm~\ref{alg:fpld} denotes their decoded
reconstruction. At $L=0$, the reconstruction is zero and a fixed $b$-bit
index is sent. Assumption~\ref{ass:A5} gives the complete boundary
convention and the exact error law used in the proofs.

\begin{assumption}[Parametric well-specification]\label{ass:A1}
$P^\star \in \mathcal{F}_\Theta = \{P_\theta : \theta \in \Theta\}$,
where $\Theta \subset \mathbb{R}^d$ is compact and convex, and every
$P_\theta(\cdot\mid x)$ is strictly positive. There is a
$\theta^\star \in \mathrm{int}(\Theta)$ with
$P_{\theta^\star} = P^\star$, and the Fisher information
$I(\theta^\star) \succ 0$. For every $(x,v)$, $\log
p_\theta(v\mid x)$ is twice continuously differentiable in $\theta$ on
an open neighborhood of $\Theta$. Consequently the canonical centered
logit map defined in Section~\ref{sec:setup} is twice continuously
differentiable and, because $\mathcal X$, $\mathcal V$, and $\Theta$ are
finite or compact, its value, gradient, and Hessian are uniformly bounded
on $\Theta$.
\end{assumption}

\begin{assumption}[Probe coverage]\label{ass:A2}
$Q$ covers $P^\star_X$ with a bounded Radon--Nikodym derivative
$\rho = \operatorname{ess\,sup}_x dP^\star_X / dQ < \infty$.
\end{assumption}

\begin{assumption}[Homogeneous data]\label{ass:A3}
The sampling distribution at every node is $P^\star$.
\end{assumption}

\begin{assumption}[Local optimization slack]\label{ass:A4}
Every local MLE and local optimizer takes values in $\Theta$, and every
local optimizer terminates within mean-squared distance
$\varepsilon_{\mathrm{opt}}$ of the local MLE:
$\mathbb{E}\|\hat\theta_i^{(t)} - \hat\theta_i^{\mathrm{MLE}}\|^2 \le
\varepsilon_{\mathrm{opt}}$.
\end{assumption}

\begin{assumption}[Uniform scalar quantization]\label{ass:A5}
By Assumption~\ref{ass:A1}, every transmitted canonical logit lies in
$[-L_\ell,L_\ell]^V$, so clipping is an identity but remains an overload
safeguard. Each coordinate is quantized at the integer rate
$b:=B/V\ge1$ by the following subtractively dithered uniform quantizer.
At $L_\ell=0$, set $D=\Delta_B=C_q=0$ and use a fixed $b$-bit message
with zero reconstruction. For $L_\ell>0$, the $M:=2^b$
codepoints $c_j := -L_\ell + j\,\Delta_B$, $j = 0, \dots, M-1$, with
step $\Delta_B := 2L_\ell/(M-1)$, have half-open cells
$[c_j - \Delta_B/2,\, c_j + \Delta_B/2)$ for $j<M-1$ and the
right-closed final cell
$[c_{M-1}-\Delta_B/2,\,c_{M-1}+\Delta_B/2]$. These cells tile
$[-L_\ell - \Delta_B/2,\, L_\ell + \Delta_B/2]$ without ambiguous
boundary assignments. A public dither
$D \sim \mathrm{Unif}[-\Delta_B/2, \Delta_B/2]$ --- independent across
coordinates, nodes, and probes, and shared with the aggregator --- is
added. Using exactly $b$ bits, the encoder transmits the index of the cell
containing $\ell_v + D$. No overload occurs because
$|\ell_v + D| \le L_\ell + \Delta_B/2$. The decoder reconstructs
$\tilde\ell_v := c_j - D$. Conditional
on the input, the error $\tilde\ell_v - \ell_v \sim
\mathrm{Unif}[-\Delta_B/2, \Delta_B/2]$ exactly, independent of the
input, with errors independent across nodes and coordinates;
per-coordinate distortion is
$\mathbb{E}[(\tilde\ell_v - \ell_v)^2\mid\ell_v] = \Delta_B^2/12
= C_q \, 2^{-2B/V}$, where $C_q := \kappa_B\, L_\ell^2/3$ with the
finite-codebook factor $\kappa_B := (1 - 2^{-B/V})^{-2} \in (1, 4]$;
equivalently $2L_\ell 2^{-B/V} \le \Delta_B \le 4L_\ell 2^{-B/V}$, and
$C_q \to L_\ell^2/3$ as $B/V$ grows.
\end{assumption}

\begin{assumption}[Distillation fit and probe complexity]\label{ass:A6}
The aggregator's distillation step produces $\hat P \in \mathcal{F}_\Theta$
within $\varepsilon_{\mathrm{fit}}$ of the infimum of empirical probe-KL\@.
Let $\mathcal T$ be the class of teacher conditionals obtainable
from the protocol's aggregate canonical logits, and assume their
canonical logits are uniformly bounded by $L_{\mathcal T}<\infty$.
Let
\[
\mathcal G:=\{x\mapsto\mathrm{KL}(R(\cdot\mid x)\|P_\theta(\cdot\mid x)):
R\in\mathcal T,\ \theta\in\Theta\}.
\]
For some $C_{\mathrm{rad}}<\infty$, this class obeys the expected
uniform-deviation bound of Rademacher-complexity type
\citep{bartlettmendelson2002}
\[
\mathbb E_{X_{\mathrm{probe}}}
\sup_{g\in\mathcal G}\left|
\frac1m\sum_{l=1}^m g(x^{(l)})-\mathbb E_{X\sim Q}g(X)
\right|
\le C_{\mathrm{rad}}\sqrt{\frac{V\log V}{m}}.
\]
The supremum makes the condition valid for the data- and
dither-dependent teacher selected by the protocol. Since empirical
fit is measured under $Q$, transferring its tolerance to the target
context law using Assumption~\ref{ass:A2}, followed by the bounded-logit
KL comparison of Lemma~\ref{lem:kl-transfer}, contributes
$c_{\mathrm{fit}}\rho\varepsilon_{\mathrm{fit}}$ in the upper bounds
for a dimension-free comparison constant $c_{\mathrm{fit}}$.
\end{assumption}

\begin{assumption}[Finite-sample local-MLE regularity]\label{ass:A7}
The local datasets are independent across nodes. Writing
$u_i:=\hat\theta_i^{\mathrm{MLE}}-\theta^\star$ and
$z_i:=u_i-\mathbb E u_i$, there are constants
$c_\Sigma,c_b,c_4<\infty$, independent of $(n,K,V)$, such that
\[
\mathrm{Cov}(u_i)\preceq\frac{c_\Sigma}{n}I(\theta^\star)^{-1},
\qquad
\|\mathbb E u_i\|_2\le\frac{c_b}{n},
\qquad
\mathbb E\|u_i\|_2^4\le c_4\left(\frac dn\right)^2.
\]
These are finite-sample conditions; they are not asserted to follow
from $C^2$ smoothness alone, and they strengthen the classical
asymptotic local-MLE regularity of \citet{vandervaart1998}. No sub-Gaussian
coordinate-maximum condition is imposed; the grouped proof instead
uses the global softmax-Hessian comparison enabled by the bounded
canonical logits in Assumption~\ref{ass:A1}.
\end{assumption}

\subsection{Constants in Theorem~\ref{thm:fpld-upper}}\label{app:constants}

The constants of Theorem~\ref{thm:fpld-upper} depend only on the
following. $c_1$ depends only on the finite-sample constants in
Assumption~\ref{ass:A7}, the coverage constant, the bounded
logit-derivative envelopes, and $I(\theta^\star)$; $c_2$ and
$c_{\mathrm{fit}}$ depend only on the bounded-logit envelopes and,
for $c_2$, the Rademacher constant $C_{\mathrm{rad}}$ in
Assumption~\ref{ass:A6}; $c_{\mathrm{opt}}$ depends on the same
regularity envelopes. The dimension-free transfer constant may be
taken as
$C_{\mathrm{dist}}=C_\triangle(1+\rho C_\leftrightarrow)$, where
$C_\triangle$ and $C_\leftrightarrow$ are the bounded-logit KL
comparison constants in Lemma~\ref{lem:kl-transfer}; $C_R$ absorbs the
same transfer applied to the cubic remainder. Before distillation, the
aggregate teacher's trace-sharpened quadratic upper-bound coefficient is
$C_q/2\le2L_\ell^2/3$, tending to $L_\ell^2/6$ as $B/V$ grows. The
total distilled-student coefficient is at most
$C_{\mathrm{dist}}C_q/2$. If the aggregate teacher is emitted directly,
the probe and fit terms vanish and the direct coefficient $C_q/2$
applies. \textbf{(SE$_{c_0}$)} makes the displayed cubic term no larger than the
quadratic term, so the total student's full bandwidth contribution is at
most $C_{\mathrm{dist}}C_qK^{-1}2^{-2B/V}$. The finite-sample
condition is the standard one-shot-averaging regime $n\gtrsim Kd$,
in which the squared-risk contribution
$O((d/n)^2)$ of the per-node Taylor remainders is dominated by the
$d/(Kn)$ variance~\citep{zhangduchiwainwright2013}.

\subsection{Bit accounting for the self-describing codec}
\label{app:bitaccounting}

In Assumption~\ref{ass:A5p} the scale index is encoded in a
fixed-length $s$-bit header per node per codec epoch. If it is
refreshed every round and the same cap is used for all $m$ probe
contexts, node $i$'s round cost is exactly $mB_i+s$ bits, or
$B_i+s/m$ bits per context. Thus a strict aggregate per-context budget
$B_{\mathrm{tot}}^{\mathrm{all}}$ leaves payload budget
$B_{\mathrm{tot}}=B_{\mathrm{tot}}^{\mathrm{all}}-Ks/m$. A cap
refreshed every round is feasible under the one-bit-per-coordinate floor
only if $B_{\mathrm{tot}}^{\mathrm{all}}\ge KV+Ks/m$. A cap
that remains valid may be cached, in which case the index is sent only
once and its aggregate amortized overhead over $T$ rounds is
$Ks/(mT)$ bits per context-round.

\section{Proofs}\label{app:lb}

This appendix collects the proofs deferred from the main text.
Appendix~\ref{ssec:soft-trace} states the softmax-Hessian trace bound used
throughout. Appendix~\ref{ssec:trace-upper} derives the bandwidth-term form
of Theorem~\ref{thm:fpld-upper} via the trace bound (paralleling the
analogous derivation in prior
work~\citep{dubey2026federatedlanguagemodelsbandwidth}). The remaining
subsections follow the order of the
corresponding statements in the body:
Appendix~\ref{ssec:matching-lb-proof} proves the matching lower bound of
Theorem~\ref{thm:matching-lb};
Appendix~\ref{app:lb-mr} proves the sequential-refinement achievability of
Theorem~\ref{thm:lb-mr};
Appendix~\ref{app:het-upper} proves the heterogeneous-bandwidth upper
bound of Theorem~\ref{thm:het-upper};
Appendix~\ref{app:alloc} proves the optimal-allocation closed form of
Theorem~\ref{thm:alloc};
Appendix~\ref{app:adaptive} proves the metadata-robustness result in
Corollary~\ref{cor:adaptive}.

\subsection{Softmax-Hessian trace bound}\label{ssec:soft-trace}

\begin{lemma}[Softmax-Hessian trace bound]\label{lem:soft-trace}
For any $p \in \Delta(\mathcal{V})$:
\begin{equation*}
\begin{aligned}
H_\mathrm{soft}(p) &= \mathrm{diag}(p) - pp^\top
  && [\text{categorical softmax Fisher/Hessian}] \\
\mathrm{tr}(H_\mathrm{soft}(p))
  &= \mathrm{tr}(\mathrm{diag}(p)) - \mathrm{tr}(pp^\top)
  && [\text{linearity of trace}] \\
  &= \textstyle\sum_v p_v - \sum_v p_v^2
  && [\text{trace of diagonal and rank-one matrices}] \\
  &= 1 - \|p\|_2^2
  && [\textstyle\sum_v p_v = 1] \\
  &\le 1
  && [\|p\|_2^2 \ge 0]
\end{aligned}
\end{equation*}
A matching lower bound:
\begin{equation*}
\begin{aligned}
\|p\|_2^2 &= \textstyle\sum_v p_v^2
  && [\text{definition of } \ell_2 \text{ norm squared}] \\
  &\le \textstyle\sum_v p_v \cdot \max_j p_j
  && [p_v \le \max_j p_j \text{ for all } v] \\
  &= \max_j p_j
  && [\textstyle\sum_v p_v = 1] \\
1 - \|p\|_2^2 &\ge 1 - \max_j p_j
  && [\text{subtract from } 1]
\end{aligned}
\end{equation*}
So if $\max_v p_v \le 1 - c_p$, then
$\mathrm{tr}(H_\mathrm{soft}(p)) \ge c_p$.
\end{lemma}

\begin{lemma}[Bounded-logit KL transfer]\label{lem:kl-transfer}
Let $p,q,r\in\Delta(\mathcal V)$ be strictly positive and suppose their
canonical logits have sup-norm at most $L_0<\infty$. There are finite
constants $C_\triangle,C_\leftrightarrow$, depending only on $L_0$ and
not on $V$, such that
\[
\mathrm{KL}(p\|r)
\le C_\triangle\{\mathrm{KL}(p\|q)+\mathrm{KL}(q\|r)\},
\qquad
\mathrm{KL}(q\|p)\le C_\leftrightarrow\mathrm{KL}(p\|q).
\]
Consequently both inequalities remain valid after averaging over any
context distribution.
\end{lemma}

\begin{proof}
The logit bound implies that every pairwise likelihood ratio among
$p,q,r$ lies in $[M^{-1},M]$ with $M=e^{4L_0}$. Let
$h^2(p,q):=\tfrac12\sum_v(\sqrt{p_v}-\sqrt{q_v})^2$ be the squared
Hellinger distance of Appendix~\ref{app:prelim}. On the compact interval
$[M^{-1},M]$, the ratio of
$f(t):=t\log t-t+1$ to $(\sqrt t-1)^2$, extended continuously at
$t=1$, is bounded above and below by positive constants. Applying this
comparison in both argument orders shows that forward KL, reverse KL,
and $h^2$ are pairwise equivalent with constants depending only on
$M$. The Hellinger triangle inequality and
$(a+b)^2\le2a^2+2b^2$ then give the first display; the same compact-ratio
comparison gives the second. Neither argument introduces a factor
depending on $V$.
\end{proof}

\subsection{Bandwidth-term form of Theorem~\ref{thm:fpld-upper}}\label{ssec:trace-upper}

For the FPLD scheme under Assumption~\ref{ass:A5}, we derive the
trace-sharpened upper bound directly from the exact softmax-KL
formula. All of the small-error hypotheses used in this paper are
instances of a single condition, indexed by the constant it requires:
for $c>0$,
\begin{equation}\label{eq:SE}
\textbf{(SE$_c$)}\qquad
\sigma_{\boldsymbol B}(\log V)^{3/2}\;\le\;c,
\qquad
\sigma_{\boldsymbol B}^2:=\frac1{K^2}\sum_{i=1}^Kq_i(B_i),
\end{equation}
where $\sigma_{\boldsymbol B}$ is the standard deviation of the
aggregated dither error. Under the homogeneous Assumption~\ref{ass:A5}
this reads $\sigma_{\boldsymbol B}=\sqrt{C_q/K}\,2^{-B/V}$, so
\eqref{eq:SE} is the familiar requirement that
$2^{-B/V}(\log V)^{3/2}/\sqrt K$ be at most a constant depending only
on $L_\ell$ and a universal third-derivative bound $C$ for log-sum-exp.
Two instances recur. \textbf{(SE$_{c_0}$)} takes
$c_0=c_0(L_\ell,C)$ chosen so the cubic remainder is dominated by the
quadratic term; the $c_p$-calibrated \textbf{(SE$_{c_p}$)} of
Appendix~\ref{ssec:matching-lb-proof} instead takes the constant small
enough that $C\sigma_{\boldsymbol B}(\log V)^{3/2}\le c_p/4$.
Throughout this derivation we assume \textbf{(SE$_{c_0}$)}. Fix a context and condition on the local estimates. Put
\[
a:=\frac1K\sum_{i=1}^K\ell_{\hat\theta_i}(x),\qquad
r:=\mathrm{softmax}(a),\qquad
\bar\varepsilon:=\frac1K\sum_{i=1}^K\varepsilon_i,
\]
and let $R_0=r$ and $R=\mathrm{softmax}(a+\bar\varepsilon)$ denote the
unquantized and quantized aggregate teachers. With
$A(z):=\log\sum_ve^{z_v}$, the exact conditional risk increment is
\[
\begin{aligned}
&\mathrm{KL}(p^\star\|R)-\mathrm{KL}(p^\star\|R_0)\\
&\qquad=A(a+\bar\varepsilon)-A(a)-p^{\star\top}\bar\varepsilon\\
&\qquad=(r-p^\star)^\top\bar\varepsilon
+\frac12\bar\varepsilon^\top H_\mathrm{soft}(r)\bar\varepsilon
+\mathcal R_3(\bar\varepsilon),
\end{aligned}
\]
where $|\mathcal R_3(\eta)|\le C\|\eta\|_\infty^3$. Assumption~\ref{ass:A5}
gives $\mathbb E[\bar\varepsilon\mid a]=0$, so the conditional linear
term vanishes and
\[
\mathbb E_U[\mathrm{KL}(p^\star\|R)-\mathrm{KL}(p^\star\|R_0)\mid a]
=\frac12\operatorname{tr}\!\left\{H_\mathrm{soft}(r)
\operatorname{Cov}(\bar\varepsilon\mid a)\right\}
+\mathbb E_U[\mathcal R_3(\bar\varepsilon)\mid a].
\]
By the exact conditional variance and node/coordinate independence in
Assumption~\ref{ass:A5},
\[
\operatorname{Cov}(\bar\varepsilon\mid a)=\sigma^2 I_V,
\qquad
\sigma^2:=\frac{C_q}{K}\,2^{-2B/V}.
\]

Apply the trace bound:
\[
\operatorname{tr}\!\left\{H_\mathrm{soft}(r)
\operatorname{Cov}(\bar\varepsilon\mid a)\right\}
=\sigma^2\operatorname{tr}\{H_\mathrm{soft}(r)\}\le\sigma^2,
\]
where the inequality is Lemma~\ref{lem:soft-trace}.
Bound the cubic remainder:
\begin{equation*}
\begin{aligned}
\mathbb{E}\|\bar\varepsilon\|_\infty^3
  &\le C' \sigma^3 (\log V)^{3/2}, \\
|\mathbb{E}[\mathcal R_3(\bar\varepsilon)\mid a]|
  &\le C'' \sigma^3 (\log V)^{3/2} \\
  &= C'' C_q^{3/2}\cdot \tfrac{(\log V)^{3/2}}{K^{3/2}}\cdot 2^{-3B/V}
  \qquad [\sigma^2=C_q2^{-2B/V}/K].
\end{aligned}
\end{equation*}
Denote the conditional quantization increment by
\[
\Delta_{\mathrm q}(a):=
\mathbb E_U[\mathrm{KL}(p^\star\|R)-\mathrm{KL}(p^\star\|R_0)\mid a].
\]
Combining the preceding bounds gives
\[
\Delta_{\mathrm q}(a)
\le \frac{C_q}{2K}\,2^{-2B/V}
+O\!\left(
\frac{C_q^{3/2}(\log V)^{3/2}}{K^{3/2}}\,2^{-3B/V}
\right).
\]
This is the aggregate teacher's conditional quantization increment.
Its trace-sharpened quadratic upper-bound coefficient is $C_q/2$, and
under \textbf{(SE$_{c_0}$)} the cubic envelope is no larger than the quadratic
term. The
$V$-factor cancellation arises directly from
$\mathrm{tr}(H_\mathrm{soft}) \le 1$ in the second step of the trace
bound; in contrast, the conservative Euclidean-logit bound uses
$\mathbb{E}\|\bar\varepsilon\|_2^2 = V\sigma^2$ and the operator norm
$\|H_\mathrm{soft}\|_\mathrm{op} \le 1/2$, which produces the $V/K$
form.

We next transfer this direct-teacher calculation to the fitted student.
For a context law $\mu$, write
$D_\mu(S,T):=\mathbb E_{X\sim\mu}\mathrm{KL}(S(\cdot\mid X)\|T(\cdot\mid X))$.
The bounded canonical logits in Assumptions~\ref{ass:A1}
and~\ref{ass:A6} permit Lemma~\ref{lem:kl-transfer} to be applied pointwise
to $P^\star$, $R$, and $\hat P$. Assumption~\ref{ass:A2} then gives
\begin{equation}\label{eq:student-transfer}
\begin{aligned}
D_{P_X^\star}(P^\star,\hat P)
&\le C_\triangle\{D_{P_X^\star}(P^\star,R)
                  +D_{P_X^\star}(R,\hat P)\}\\
&\le C_\triangle\{D_{P_X^\star}(P^\star,R)
                  +\rho D_Q(R,\hat P)\}.
\end{aligned}
\end{equation}
The probe and fit terms follow from a uniform oracle inequality that
remains valid for the random aggregate teacher. For
$R\in\mathcal T$, write
$\widehat L_m(R,\theta):=m^{-1}\sum_l g_{R,\theta}(x^{(l)})$,
$L_Q(R,\theta):=\mathbb E_Qg_{R,\theta}(X)$, and
$D_m:=\sup_{R,\theta}|\widehat L_m(R,\theta)-L_Q(R,\theta)|$.
Assumption~\ref{ass:A6} gives, for the selected teacher $R$ and fitted
parameter $\hat\theta$,
\[
L_Q(R,\hat\theta)
\le \inf_{\theta\in\Theta}L_Q(R,\theta)
   +\varepsilon_{\mathrm{fit}}+2D_m.
\]
Since $P^\star=P_{\theta^\star}$ belongs to the student family, the
population comparator may be chosen as $\theta^\star$. Combining the
oracle inequality with \eqref{eq:student-transfer} and the reverse-KL
part of Lemma~\ref{lem:kl-transfer} yields
\[
\begin{aligned}
D_{P_X^\star}(P^\star,\hat P)
\le C_\triangle\bigl\{&D_{P_X^\star}(P^\star,R)
+\rho C_\leftrightarrow D_Q(P^\star,R)\\
&+\rho\varepsilon_{\mathrm{fit}}+2\rho D_m\bigr\}.
\end{aligned}
\]
The direct-teacher expansion above applies under both $P_X^\star$ and
$Q$, because it is pointwise in the context. Its quadratic bandwidth
coefficient in the last display is therefore multiplied by the
dimension-free constant
\[
C_{\mathrm{dist}}:=C_\triangle(1+\rho C_\leftrightarrow).
\]
The same transfer multiplies the cubic envelope and is absorbed into
$C_R$. Taking expectations and using
$\mathbb ED_m\le C_{\mathrm{rad}}\sqrt{V\log V/m}$ gives the theorem's
$c_2\rho\sqrt{V\log V/m}+c_{\mathrm{fit}}\rho
\varepsilon_{\mathrm{fit}}$ terms, where
$c_2\ge2C_\triangle C_{\mathrm{rad}}$ and
$c_{\mathrm{fit}}\ge C_\triangle$.

For completeness, the expected local-data term follows from the
$G=1$ specialization of the global statistical argument in
Appendix~\ref{app:grouped}: equations~\eqref{eq:group-global}--\eqref{eq:group-kappa}
and Assumption~\ref{ass:A7} give the centered variance contribution
$O(d/(Kn))$, while the MLE bias and logit Taylor remainder contribute
$O(n^{-2}+(d/n)^2)$. The condition $dK/n\le c_{\mathrm{stat}}$
absorbs both into $O(d/(Kn))$. Assumptions~\ref{ass:A4}
and~\ref{ass:A6} add the displayed optimization and fit slacks, and the
expected probe Rademacher bound adds
$c_2\rho\sqrt{V\log V/m}$. The same bounded-derivative calculation
under $Q$ controls the comparator term $D_Q(P^\star,R_0)$ at
$O(d/(Kn))+O(n^{-2}+(d/n)^2)$, with constants absorbed into $c_1$ and
$c_{\mathrm{opt}}$. Together with the transfer and bandwidth
calculations above, this proves Theorem~\ref{thm:fpld-upper} under its
stated expected-risk quantifiers. If the aggregate teacher $R$ is
emitted directly, no transfer is needed: the probe and fit terms vanish
and the bandwidth coefficient is the direct value $C_q/2$.
\hfill$\square$

\subsection{Proof of Theorem~\ref{thm:matching-lb}}\label{ssec:matching-lb-proof}

Here \textbf{(SE$_{c_p}$)} denotes \eqref{eq:SE} with its constant
chosen so that $C\sigma_{\boldsymbol B}(\log V)^{3/2}\le c_p/4$.
We prove Theorem~\ref{thm:matching-lb}: under
Assumptions~\ref{ass:A1} and~\ref{ass:A5}, together with \textbf{(ND)} and
\textbf{(SE$_{c_p}$)}, the bandwidth contribution of the dithered FPLD
construction satisfies
\[
\mathbb{E}[\mathrm{KL}(P^\star \,\|\, \hat P_{\mathrm{FPLD}})]_{\mathrm{bandwidth}}
\ge \frac{c_p L_\ell^2}{12 K} \cdot 2^{-2B/V}.
\]

\begin{proof}
Compute the dither variance directly under Assumption~\ref{ass:A5}.
The explicit quantizer of Assumption~\ref{ass:A5} at $B/V$ bits has
step size
\begin{equation*}
\begin{aligned}
\Delta_B &= \frac{2 L_\ell}{2^{B/V} - 1} \;\ge\; 2 L_\ell \cdot 2^{-B/V}
  && [\text{$2^{B/V}$ codepoints spanning $[-L_\ell, L_\ell]$}] \\
\varepsilon_{i,v} &\sim \mathrm{Unif}[-\Delta_B/2, +\Delta_B/2]
  && [\text{subtractive dither; no overload}] \\
\mathbb{E}[\varepsilon_{i,v} \mid X] &= 0
  && [\text{subtractive dither is unbiased}] \\
\mathrm{Var}(\varepsilon_{i,v} \mid X)
  &= \frac{\Delta_B^2}{12}
  && [\text{uniform law of width $\Delta_B$}] \\
\frac{L_\ell^2}{3}\cdot 2^{-2B/V}
  &\le \frac{\Delta_B^2}{12}
     \le \frac{4 L_\ell^2}{3}\cdot 2^{-2B/V}
  && \stepnote{step-size brackets on $\Delta_B$ at $B\ge V$; the upper
     bound is used in \textbf{(SE$_{c_p}$)}}
\end{aligned}
\end{equation*}
The aggregator forms
$\bar\varepsilon_v = (1/K)\sum_{i=1}^K \varepsilon_{i,v}$.
Conditional on $X$:
\begin{equation*}
\begin{aligned}
\mathbb{E}[\bar\varepsilon_v \mid X]
  &= 0
  && \stepnote{linearity; each $\varepsilon_{i,v}$ is conditionally
     mean-zero} \\
\mathrm{Var}(\bar\varepsilon_v \mid X)
  &= \frac{1}{K^2}\sum_{i=1}^K \mathrm{Var}(\varepsilon_{i,v} \mid X)
  && \stepnote{independence across nodes; no cross terms} \\
  &= \frac{\Delta_B^2}{12 K} \;=:\; \sigma^2
  && \stepnote{substitute per-node variance} \\
\mathrm{Cov}(\bar\varepsilon \mid X)
  &= \sigma^2 I_V
  && \stepnote{independence across coordinates by
     Assumption~\ref{ass:A5}} \\
\frac{L_\ell^2}{3K}\, 2^{-2B/V}
  &\le \sigma^2 \le \frac{4L_\ell^2}{3K}\, 2^{-2B/V}
  && \stepnote{step-size brackets on $\Delta_B$}
\end{aligned}
\end{equation*}
Apply the exact softmax-KL Taylor expansion conditional on $X$:
\[
\resizebox{\linewidth}{!}{$\displaystyle
\begin{aligned}
\mathbb{E}_{\bar\varepsilon}\!\bigl[\mathrm{KL}(\mathrm{softmax}(\ell^\star(X)) \,\|\, \mathrm{softmax}(\ell^\star(X) + \bar\varepsilon)) \mid X\bigr]
  &= \tfrac{1}{2}\,\mathrm{tr}\!\bigl(H_\mathrm{soft}(p^\star(\cdot \mid X)) \cdot \mathrm{Cov}(\bar\varepsilon \mid X)\bigr) + R_3(X)
  && \stepnote{cumulant expansion; $R_3$ is the cubic remainder} \\
  &= \tfrac{\sigma^2}{2}\,\mathrm{tr}\!\bigl(H_\mathrm{soft}(p^\star(\cdot \mid X))\bigr) + R_3(X)
  && \stepnote{$\mathrm{Cov}(\bar\varepsilon \mid X) = \sigma^2 I_V$}
\end{aligned}
$}
\]
Take expectation over $X \sim P^\star_X$, the context law of the
risk and of Definition~\ref{def:bandwidth}; the probe marginal plays
no role in the bandwidth-isolated functional:
\[
\resizebox{\linewidth}{!}{$\displaystyle
\begin{aligned}
\mathbb{E}_{X, \bar\varepsilon}[\mathrm{KL}]
  &= \frac{\sigma^2}{2}\, \mathbb{E}_X\!\bigl[\mathrm{tr}(H_\mathrm{soft}(p^\star(\cdot \mid X)))\bigr] + \mathbb{E}_X[R_3(X)]
  && \stepnote{linearity} \\
  &= \frac{\sigma^2}{2}\, \mathbb{E}_X\!\bigl[1 - \|p^\star(\cdot \mid X)\|_2^2\bigr] + \mathbb{E}_X[R_3(X)]
  && \stepnote{$\mathrm{tr}(H_\mathrm{soft}(p)) = 1 - \|p\|_2^2$ by Lemma~\ref{lem:soft-trace}} \\
  &\ge \frac{\sigma^2 c_p}{2} - |\mathbb{E}_X[R_3(X)]|
  && \stepnote{apply \textbf{(ND)}; bound the remainder in absolute value}
\end{aligned}
$}
\]
Apply (SE$_{c_p}$) to dominate the cubic remainder. The remainder satisfies
$|\mathbb{E}_X[R_3(X)]| \le C \sigma^3 (\log V)^{3/2}$; under (SE$_{c_p}$),
the constant in~\eqref{eq:SE} is calibrated to $(c_p,L_\ell,C)$ so that
$C \cdot \sigma \cdot (\log V)^{3/2} \le c_p/4$, equivalently
$|\mathbb{E}_X[R_3(X)]| \le (c_p/4) \cdot \sigma^2$. Hence:
\begin{equation*}
\begin{aligned}
\mathbb{E}_{X, \bar\varepsilon}[\mathrm{KL}]
  &\ge \frac{\sigma^2 c_p}{2} - \frac{c_p \sigma^2}{4}
  && [\text{cubic remainder bounded by (SE$_{c_p}$)}] \\
  &= \frac{c_p \sigma^2}{4}
  && [\text{combine}] \\
  &\ge \frac{c_p}{4}\cdot \frac{L_\ell^2}{3 K}\cdot 2^{-2B/V}
  && [\sigma^2 \ge L_\ell^2/(3K) \cdot 2^{-2B/V}] \\
  &= \frac{c_p L_\ell^2}{12 K}\cdot 2^{-2B/V}
  && [\text{simplify to closed form}]
\end{aligned}
\end{equation*}
This is the displayed lower bound with
$c_{\mathrm{lb}} := c_p L_\ell^2 / 12$.
\end{proof}

\paragraph{Remark on scope.}
Theorem~\ref{thm:matching-lb} is a lower bound for the
subtractively dithered FPLD construction
(Assumption~\ref{ass:A5}); the unconditional minimax bandwidth-axis
lower bound over the larger class $\mathcal{P}_B$ of arbitrary
$B$-bit-per-node protocols (biased or adaptive scalar quantizers,
vector quantizers, and learned codebooks lie outside
Assumption~\ref{ass:A5}) is a different object and is 
\rev{treated} in
Section~\ref{sec:uncond}. In particular, its rate is exponentially
smaller in $K$, so the construction-specific scope here is essential.

\subsection{Proof of Theorem~\ref{thm:lb-mr}}\label{app:lb-mr}

Theorem~\ref{thm:lb-mr} is an achievability statement: sequential
refinement with nested/scaled residual quantizers and $B$ bits per
node per round achieves $O\bigl(K^{-1} (2^{B/V}-1)^{-2T}\bigr)$. The proof
iterates the single-round trace-sharpened upper bound of
Theorem~\ref{thm:fpld-upper} over the residual at each round.

\paragraph{Setup.}
In round $t \in \{1, \dots, T\}$, each node $i$ holds the residual
$r^{(t-1)}(X) := \ell^\star(X) - \hat\ell^{(t-1)}(X)$ (with
$\hat\ell^{(0)}(X) = 0$, or any common initialization satisfying
$\sup_X\|\ell^\star(X)-\hat\ell^{(0)}(X)\|_\infty\le L_\ell$), quantizes it
to $\tilde r_i^{(t)}(X) = r^{(t-1)}(X) + \varepsilon_i^{(t)}(X)$ with
the round-$t$ dithered quantizer, and the aggregator updates the
reconstruction \emph{cumulatively},
\[
\hat\ell^{(t)}(X) \;=\; \hat\ell^{(t-1)}(X)
\;+\; \frac1K \sum_{i=1}^K \tilde r_i^{(t)}(X),
\]
which it broadcasts. Substituting
$\tilde r_i^{(t)} = r^{(t-1)} + \varepsilon_i^{(t)}$, the update adds
back the full residual plus the averaged noise, so
$\ell^\star - \hat\ell^{(t)} = -\tfrac1K \sum_i \varepsilon_i^{(t)}$:
the reconstruction error after round $t$ is exactly the $K$-averaged
round-$t$ quantization noise on a residual whose range has contracted
$t-1$ times, which is what drives the geometric recursion below. The nested/scaled quantizer
of round $t$ uses step size
\begin{equation*}
\begin{aligned}
\Delta_t &= \frac{2 L_\ell^{(t)}}{2^{B/V}-1}
  && [\text{Assumption~\ref{ass:A5}, rescaled to round-$t$ range}] \\
L_\ell^{(t)} &= L_\ell\,(2^{B/V}-1)^{-(t-1)}
  && [\text{worst case: } |\tfrac1K\textstyle\sum_i\varepsilon_i^{(t-1)}| \le \Delta_{t-1}/2]
\end{aligned}
\end{equation*}
The range recursion is deterministic: the round-$(t-1)$ residual is
the $K$-average of errors each supported on
$[-\Delta_{t-1}/2, \Delta_{t-1}/2]$, so its worst-case half-range is
$\Delta_{t-1}/2 = L_\ell^{(t-1)}/(2^{B/V}-1) = L_\ell^{(t)}$;
averaging shrinks the variance by $K$ but not the support, and the
finite codebook contracts the range by $(2^{B/V}-1)^{-1}$ per round,
not $2^{-B/V}$.

The post-aggregation residual variance per coordinate,
$\sigma_t^2 := \mathrm{Var}(\hat\ell^{(t)}_v - \ell^\star_v \mid X)$,
satisfies a geometric recursion:
\begin{equation*}
\begin{aligned}
\sigma_t^2 &= \frac{\Delta_t^2}{12 K}
  && \stepnote{exact dither variance of round $t$, $K$-fold averaged} \\
  &= \frac{L_\ell^2}{3 K}\,\bigl(2^{B/V}-1\bigr)^{-2t}
  && [\text{substitute } \Delta_t = 2 L_\ell\,(2^{B/V}-1)^{-t}]
\end{aligned}
\end{equation*}
After $T$ rounds,
$\sigma_T^2 = L_\ell^2/(3K) \cdot (2^{B/V}-1)^{-2T}$.
Applying the trace bound from Lemma~\ref{lem:soft-trace} to the final
aggregated residual gives
\[
\resizebox{\linewidth}{!}{$\displaystyle
\begin{aligned}
\mathbb{E}\bigl[\mathrm{KL}(P^\star \,\|\, \hat P_{\mathrm{seq\text{-}ref}})\bigr]_{\mathrm{bandwidth}}
  &\le \tfrac{1}{2}\,\sigma_T^2 \cdot \mathrm{tr}(H_\mathrm{soft}) + \text{cubic remainder}
  && \stepnote{trace-sharpened cumulant chain (Appendix~\ref{ssec:trace-upper}) applied at round $T$} \\
  &\le \tfrac{1}{2}\,\sigma_T^2 + O(\sigma_T^3 (\log V)^{3/2})
  && \stepnote{$\mathrm{tr}(H_\mathrm{soft}) \le 1$} \\
  &= \frac{C_3'}{K}\cdot \bigl(2^{B/V}-1\bigr)^{-2T}
    + O\!\Bigl(\tfrac{(\log V)^{3/2}}{K^{3/2}}\cdot \bigl(2^{B/V}-1\bigr)^{-3T}\Bigr)
  && \stepnote{substitute $\sigma_T^2 = L_\ell^2/(3K)\,(2^{B/V}-1)^{-2T}$; set $C_3' := L_\ell^2/6$}
\end{aligned}
$}
\]
Here $C_3'$ is the large-$B/V$ limit of the direct-teacher coefficient
$C_q/2$: both equal $L_\ell^2/6$ when
$C_q\to L_\ell^2/3$. Since
$(2^{B/V}-1)^{-2T} = 2^{-2TB/V}(1-2^{-B/V})^{-2T}$, the finite
codebook costs a factor that is exponential in $T$ but absolute
whenever $T \le 2^{B/V}$, and at most one bit per coordinate per
round in general; at $B = V$ the worst-case residual does not
contract and refinement gives no gain. The
vanilla FPLD comparison is the special case where each round
re-quantizes the full $\ell^\star$ at full range $L_\ell$ (rather
than the shrinking residual), so $\sigma_t^2$ does not decrease
across rounds and the quadratic bandwidth term stalls at
$C_q/(2K) \cdot 2^{-2B/V}$ for every $T$, equivalently
$\frac{L_\ell^2}{6K}(2^{B/V}-1)^{-2}$. \qed

\paragraph{Note on protocol class.}
The above proof relies on each round being an FPLD-class protocol on
the residual, with the quantizer step rescaled to the residual range.
A matching unconditional multi-round minimax lower bound over
arbitrary $B$-bit interactive protocols is not claimed. Within the
same scalar-refinement FPLD multi-round class under \textbf{(ND)} and \textbf{(SE$_{c_p}$)},
Theorem~\ref{thm:matching-lb}'s dither-variance chain applied at
round $T$ gives a fixed-scheme lower bound
$\Omega(K^{-1} \cdot 2^{-2TB/V})$, which matches the upper bound only
up to the finite-codebook factor $(1-2^{-B/V})^{-2T}$ (an absolute
constant for $T \le 2^{B/V}$).

\subsection{Proof of Theorem~\ref{thm:het-upper}}\label{app:het-upper}

The aggregator's averaged quantized logits are
$\bar\ell(x) = \tfrac{1}{K}\sum_i \tilde\ell_i(x)$, where
$\tilde\ell_i = \ell_i + \xi_i$, with $\ell_i$ the empirical logit
produced by node $i$'s local optimizer and $\xi_i$ the dithered
quantization noise (zero-mean and independent across nodes and coordinates,
Assumption~\ref{ass:A5p}). The student $\hat P$ fits the softmax of
$\bar\ell$ within $\varepsilon_{\mathrm{fit}}$ (Assumption~\ref{ass:A6}),
and local optimizers terminate within $\varepsilon_{\mathrm{opt}}$ of
the local MLE (Assumption~\ref{ass:A4}).

The proof follows the same structure as the homogeneous-bandwidth
case. The per-node budgets $B_i$ replace the uniform $B$ only in the
conditional quantization increment. The bounded-logit student transfer,
probe generalization, and local-data calculations then apply exactly as
in Appendix~\ref{ssec:trace-upper}.

\paragraph{Conditional cross-entropy expansion.}
Condition on the local data, probe set, transmitted logits, reported
caps, and scheduled rates; call the resulting pre-dither information
$\mathcal H$. The rates and caps are then fixed, while the payload
dithers remain independent.
Write $\bar\ell=\bar\ell^{\mathrm{emp}}+\bar\xi$, where
$\bar\ell^{\mathrm{emp}}=K^{-1}\sum_i\ell_i$ is the aggregate before
quantization and $\bar\xi=K^{-1}\sum_i\xi_i$ is the averaged dither.
For fixed empirical logits, put
$A(z)=\log\sum_ve^{z_v}$. The difference between the quantized and
unquantized population cross-entropies is
\[
A(\bar\ell^{\mathrm{emp}}+\bar\xi)-A(\bar\ell^{\mathrm{emp}})
-p^{\star\top}\bar\xi.
\]
Assumption~\ref{ass:A5p} gives
$\mathbb E[\bar\xi\mid\mathcal H]=0$, so the conditional
linear term vanishes. Taylor expansion at
$\bar\ell^{\mathrm{emp}}$ therefore gives the inherited unquantized
risk plus
\[
\frac12\operatorname{tr}\!\left[
H_{\mathrm{soft}}(\mathrm{softmax}(\bar\ell^{\mathrm{emp}}))
\operatorname{Cov}(\bar\xi\mid\mathcal H)\right]
+\text{a controlled cubic remainder}.
\]
This is an exact conditional expansion of the risk increment; it
does not assert that the full KL risk decomposes into statistical and
bandwidth components.

\paragraph{Term (I): unquantized aggregate teacher.}
The $G=1$ statistical calculation used in
Appendix~\ref{ssec:trace-upper} bounds the unquantized aggregate
teacher's expected risk by $O(d/(Kn))$, plus the local-optimization
slack, under both $P_X^\star$ and $Q$. The per-node $B_i$ do not enter
this term. Probe generalization and distillation-fit error enter only
when this teacher bound is transferred to the fitted student below.

\paragraph{Term (II): bandwidth.}
Under Assumption~\ref{ass:A5p}'s coordinatewise independence,
$\mathrm{Cov}(\bar\xi\mid\mathcal H)$ is diagonal with $v$-th entry
$(1/K^2)\sum_{i=1}^K \mathrm{Var}(\xi_{i,v}\mid\mathcal H)
=(1/K^2)\sum_i C_{q,i}\,2^{-2B_i/V}$. Hence
$\mathrm{Cov}(\bar\xi\mid\mathcal H) \preceq \sigma^2 I_V$ with
\[
\sigma^2 \;:=\; \frac{1}{K^2}\sum_{i=1}^{K} C_{q,i}\,2^{-2B_i/V}
=\frac1{K^2}\sum_{i=1}^K q_i(B_i).
\]
Applying Lemma~\ref{lem:soft-trace}'s trace bound at the intermediate
distribution $\mathrm{softmax}(\bar\ell^{\mathrm{emp}})$, together
with the standard inequality
$\mathrm{tr}(AB) \le \|B\|_\mathrm{op}\,\mathrm{tr}(A)$ for
$A \succeq 0$, gives
\[
\tfrac{1}{2}\,\mathrm{tr}\!\bigl(
H_\mathrm{soft}(\mathrm{softmax}(\bar\ell^{\mathrm{emp}}))
\cdot \mathrm{Cov}(\bar\xi\mid\mathcal H)\bigr)
\;\le\; \frac{\sigma^2}{2}
\;=\; \frac1{2K^2}\sum_{i=1}^{K} C_{q,i}2^{-2B_i/V}.
\]
Before distillation, since $C_{q,i}\le4L_{\ell,i}^2/3$, this is at most
$\frac{2}{3K^2}\sum_iL_{\ell,i}^2 2^{-2B_i/V}$ conditionally on
$\mathcal H$. The cubic remainder is
$O(\sigma^3(\log V)^{3/2})$ and is dominated under
\textbf{(SE$_{c_0}$)} by the same argument as in
Appendix~\ref{ssec:trace-upper}.

Apply \eqref{eq:student-transfer}, the probe oracle inequality,
and Lemma~\ref{lem:kl-transfer}'s reverse-KL comparison. Exactly as in
the homogeneous proof, the aggregate-teacher quadratic term is
multiplied by
$C_{\mathrm{dist}}=C_\triangle(1+\rho C_\leftrightarrow)$; the transfer
of the cubic term is absorbed into $C_R$, and the fit term becomes
$c_{\mathrm{fit}}\rho\varepsilon_{\mathrm{fit}}$. Finally, averaging
the conditional bandwidth and remainder bounds over
$(\mathcal D,X_{\mathrm{probe}})$ gives the two expectations displayed
in the theorem. Combining these bounds with Term~(I) proves
Theorem~\ref{thm:het-upper}. \qed

\paragraph{Remark (exact channel versus allocation surrogate).}
Conditional on the realized caps and scheduled rates, the aggregate
teacher's exact-variance quadratic upper-bound term is
\[
\frac1{6K^2}\sum_iL_{\ell,i}^2(2^{B_i/V}-1)^{-2}.
\]
Its envelope
\[
\frac{2}{3K^2}\sum_iL_{\ell,i}^2 2^{-2B_i/V}
\]
differs by at most the
uniform finite-codebook factor $4$ for $B_i\ge V$. The fitted-student
bound multiplies both expressions by the same common
$C_{\mathrm{dist}}$, so the optimizer is unchanged. The envelope is
precisely the objective of Theorem~\ref{thm:alloc} with
$w_i=L_{\ell,i}^2$, up to an irrelevant common multiplier. Thus the
closed-form allocation optimizes a rigorous constant-factor upper
envelope, not the exact finite-codebook expression.

\subsection{Proof of Theorem~\ref{thm:alloc}}\label{app:alloc}

The objective $F$ is convex in each $B_i$: it is the sum of decaying
exponentials in a positive linear function of $B_i$, weighted by $w_i > 0$.
The box-constrained simplex slice in Theorem~\ref{thm:alloc} is convex.
First consider an interior optimum and form the Lagrangian
\[
\mathcal{L}(B, \nu)
\;=\;
\frac{1}{K^2}\sum_{i=1}^{K} w_i 2^{-2B_i/V}
\;+\;
\nu\!\left(\sum_{i=1}^{K} B_i - B_{\mathrm{tot}}\right).
\]
Stationarity, $\partial \mathcal{L}/\partial B_i = 0$, gives
\[
-\frac{2 \ln 2}{V K^2}\, w_i \cdot 2^{-2B_i/V}
\;+\; \nu \;=\; 0
\quad\Longleftrightarrow\quad
2^{-2B_i/V} \;=\; \frac{\nu V K^2}{2 w_i \ln 2}.
\]
Taking $\log_2$,
\[
B_i \;=\; -\frac{V}{2}\log_2\!\left(\frac{\nu V K^2}{2 w_i \ln 2}\right)
\;=\; \frac{V}{2}\log_2(w_i) \;+\; \mathrm{const}(\nu).
\]
Imposing $\sum_i B_i = B_{\mathrm{tot}}$ pins down the constant:
\[
B_i^\star \;=\; \frac{B_{\mathrm{tot}}}{K}
\;+\; \frac{V}{2}\log_2(w_i)
\;-\; \frac{V}{2K}\sum_{j=1}^{K} \log_2(w_j)
\;=\; \frac{B_{\mathrm{tot}}}{K} \;+\; \frac{V}{2}\log_2\!\left(\frac{w_i}{w_{\mathrm{gm}}}\right),
\]
where $w_{\mathrm{gm}} = \exp\!\bigl(\tfrac{1}{K}\sum_j \log w_j\bigr) = (\prod_j w_j)^{1/K}$
is the geometric mean. When all $w_i$ are equal, $\log_2(w_i/w_{\mathrm{gm}}) = 0$
and the optimal allocation is uniform.

The KKT conditions with $B_i\in[B_{\min},B_{\max}]$ clip this interior
solution at either boundary:
$B_i^\star=\min\{B_{\max},\max\{B_{\min},(V/2)\log_2(w_i/\nu)\}\}$.
The left side of the budget equation is continuous and non-increasing in
$\nu$ and spans $[KB_{\min},KB_{\max}]$, so a level satisfying the
budget exists. Strict convexity makes the allocation unique; only the
auxiliary level $\nu$ can be nonunique when every coordinate is fixed
at a boundary. Fixing saturated coordinates and
re-solving on the remaining set is exactly the active-set construction
stated in the theorem. \qed

For the integer-rate version, write $b_i=B_i/V$. Increasing $b_i$ by
one changes node $i$'s objective contribution by
$-(3/4)K^{-2}w_i4^{-b_i}$. These marginal improvements decrease
geometrically along each node's sequence. Every feasible integer
allocation selects a prefix from each sequence, and repeatedly taking
the largest currently available improvement selects the largest
feasible collection of prefixes. The usual exchange argument therefore
proves the marginal-gain rule stated after Theorem~\ref{thm:alloc} is an
exact discrete minimizer (with caps enforced by ending the corresponding
sequence).

For the exact finite-codebook objective, put
$f(x):=(2^x-1)^{-2}$. Direct differentiation gives
\[
f''(x)=2(\ln2)^2\frac{2^x(2^{x+1}+1)}{(2^x-1)^4}>0,
\qquad x>0.
\]
Thus $f$ is decreasing and strictly convex, so the positive forward
differences $f(b)-f(b+1)$ decrease with the integer rate $b$. The exact
objective's node-$i$ marginal gains are
$(6K^2)^{-1}w_i\{f(b_i)-f(b_i+1)\}$. They again form one decreasing
sequence per node, and the same prefix/exchange argument proves that
successively taking the largest available exact gain is a global integer
minimizer of $G_w$, including when a finite upper cap ends a sequence.

\begin{algorithm}[!htbp]
\caption{Metadata-aware FPLD bandwidth allocation.}
\label{alg:adaptive-fpld}
\small
\begin{algorithmic}[1]
\Require $K$ nodes; $m$ probe contexts per round; payload budget
$B_{\mathrm{tot}}\ge KV$; rounds $T$; public scale grid
$\mathcal L_s$; optional cap $B_{\max}\ge B_{\mathrm{tot}}/K$.
\Statex \textbf{Codec registration (once per epoch).}
\For{each node $i=1,\dots,K$ \textbf{in parallel}}
  \State choose $L_{\ell,i}^{\mathrm{req}}$ bounding every logit in
  the epoch
  \State set $r_i\gets\min\{r:\lambda_r\ge
  L_{\ell,i}^{\mathrm{req}}\}$ and transmit the $s$-bit index
\EndFor
\State aggregator sets $L_{\ell,i}\gets\lambda_{r_i}$ and
$w_i\gets L_{\ell,i}^2$ for every $i$
\State compute the active-set minimizer of $F_w$ from
Theorem~\ref{thm:alloc}
\If{integer per-coordinate rates are required}
  \State compute the exact discrete allocation by successive largest
  marginal decreases in $F_w$, or in $G_w$ when optimizing the exact
  finite-codebook objective
\EndIf
\State communicate the scheduled payload budgets $B_i$ to the nodes
\Statex \textbf{Payload rounds.}
\For{$t=1,\dots,T$}
  \For{each node $i$ \textbf{in parallel}}
    \State quantize and transmit the probe logits using cap
    $L_{\ell,i}$ and payload budget $B_i$
  \EndFor
  \State aggregator decodes with the same $L_{\ell,i}$ and public
  dithers, then runs the FPLD aggregation rule
\EndFor
\State if a reported cap no longer bounds the epoch logits, start a
new codec epoch or use the dynamic protocol of Section~\ref{ssec:streaming}
\end{algorithmic}
\end{algorithm}

\subsection{Grid resolution of the self-describing codec}\label{app:gridres}

Suppose that $M_s\ge2$ and every required cap lies in
$[\lambda_1,\lambda_{M_s}]$. Define
\[
h_s:=\max_{1\le r<M_s}\log_2(\lambda_{r+1}/\lambda_r).
\]
Then upward rounding gives
$L_{\ell,i}\le2^{h_s}L_{\ell,i}^{\mathrm{req}}$. Consequently each
exact dither variance in Assumption~\ref{ass:A5p} is at most
$2^{2h_s}$ times the variance of the otherwise identical codec using
the unrounded required cap.

\subsection{Exact integer allocation}\label{app:integer}

Theorem~\ref{thm:alloc} is a continuous relaxation. For the integer
per-coordinate rates of Assumption~\ref{ass:A5p}, suppose
$B_{\mathrm{tot}}/V$ and the finite endpoints divided by $V$ are
integers. The exact discrete optimum is obtained by starting at the
lower bounds and repeatedly assigning the next $V$-bit block to the
node with the largest marginal decrease in $F$: node $i$'s successive
decreases are proportional to $w_i4^{-b_i}$ and form a decreasing
sequence, so selecting the largest available decrease is optimal.
The exact finite-codebook objective $G_w$ in
Corollary~\ref{cor:adaptive} is separable as well. Since
$x\mapsto(2^x-1)^{-2}$ is decreasing and strictly convex on $x>0$, its
successive one-bit decreases also form decreasing sequences; the same
largest-marginal-decrease rule, using the finite-codebook decreases, is
therefore its exact integer optimizer. This is the ``finite-exact''
policy used in Section~\ref{sec:exp}.
As a simpler constant-factor guarantee, rounding each continuous
$B_i^\star/V$ down respects the budget constraint and increases $F$ by
less than a factor $4$; any leftover blocks can then be assigned by the
same marginal rule.

\subsection{Proof of Corollary~\ref{cor:adaptive}}\label{app:adaptive}

The result is a deterministic comparison for objectives that are
positive and linear in the node weights.

\begin{lemma}[Positive-weight transfer]\label{lem:transfer}
Let $h_i:\mathcal B\to[0,\infty)$ and
$H_w(B):=\sum_iw_ih_i(B)$ on a common nonempty feasible set
$\mathcal B$. Suppose $w,\hat w\in(0,\infty)^K$ and constants
$0<a\le b<\infty$ satisfy
$aw_i\le\hat w_i\le bw_i$ for every $i$. If
$B^\star\in\arg\min_{B\in\mathcal B}H_w(B)$ and
$\hat B^\star\in\arg\min_{B\in\mathcal B}H_{\hat w}(B)$, then
\[
H_w(\hat B^\star)\le\frac ba H_w(B^\star).
\]
No convexity, interiority, or uniqueness assumption is needed.
\end{lemma}

\begin{proof}
For every $B\in\mathcal B$, coordinatewise comparison and
$h_i(B)\ge0$ give
$aH_w(B)\le H_{\hat w}(B)\le bH_w(B)$. Optimality yields
\[
aH_w(\hat B^\star)
\le H_{\hat w}(\hat B^\star)
\le H_{\hat w}(B^\star)
\le bH_w(B^\star),
\]
and division by $a$ proves the claim.
\end{proof}

Apply Lemma~\ref{lem:transfer} to
$h_i(B)=K^{-2}2^{-2B_i/V}$. The logarithmic condition in
Corollary~\ref{cor:adaptive} gives $a=2^{-\varepsilon}$ and
$b=2^\varepsilon$, hence $b/a=2^{2\varepsilon}$. Its exact-metadata
case has $a=b=1$. The relative-error condition gives
$a=1-\eta$ and $b=1+\eta$. Finally, for the exact finite-codebook
objective use
$h_i(B)=(6K^2)^{-1}(2^{B_i/V}-1)^{-2}$; the same lemma applies
without change.

For the scale-grid comparison, if the required cap equals
$\lambda_1$ then it is unchanged. Otherwise minimality of $r_i$ gives
$\lambda_{r_i-1}<L_{\ell,i}^{\mathrm{req}}\le\lambda_{r_i}$, while
the definition of $h_s$ gives
$\lambda_{r_i}\le2^{h_s}\lambda_{r_i-1}$. Thus
$L_{\ell,i}<2^{h_s}L_{\ell,i}^{\mathrm{req}}$. At a fixed payload
rate, Assumption~\ref{ass:A5p}'s exact dither variance is proportional
to the squared cap, proving the stated factor $2^{2h_s}$. This
completes the proof of Corollary~\ref{cor:adaptive} and its accompanying
grid-resolution claim. \qed

\section{Scalar-Class Lower Bounds, the Multi-Round Baseline, and Exact Experiments}\label{app:rebuttal}

This section proves the scalar-class lower bounds of
Proposition~\ref{prop:rebuttal-lb} and the multi-round baseline rate. It also
gives full setup, additional results, and reproducibility details for the
exact finite-codebook campaign of Sections~\ref{sec:exp}
and~\ref{sec:reallm}.

\subsection{A lower bound for bounded-resolution unbiased scalar quantize-and-average
protocols}\label{app:rebuttal-lb}

Theorem~\ref{thm:matching-lb} lower-bounds the bandwidth contribution for the
\emph{dithered uniform} scalar-quantized FPLD construction. We show that the same
$K^{-1}2^{-2B/V}$ rate is a \emph{mean-squared} lower bound for the much larger
class that uses \emph{any} unbiased scalar quantizer (and, under a mild
Hessian-weighted spread condition, a \emph{KL} lower bound for that whole
class), settling whether a smarter (scalar) encoder could improve the rate.

\paragraph{Protocol class.} Let $\mathcal{P}_B^{\mathrm{scalar}}$ denote the
single-round FPLD-class protocols in which each node $i$ encodes each
coordinate $\ell_{i,v}$ ($v=1,\dots,V$) into $B/V$ bits through an
arbitrary, possibly coordinate-dependent, randomized \emph{unbiased} scalar
quantizer. The encoder may use private randomness and public randomness $\Omega_i$
shared with the aggregator (e.g.,\ a vector of independent dithers), all
independent of the context and input logits;
conditionally on $\Omega_i$, coordinate $v$'s decoder output is a function of its
$B/V$-bit message alone and hence lies in a set
$\mathcal C_{i,v}(\Omega_i)$ of at most $2^{B/V}$ reconstruction points that
does not depend on the current context or logit value.
Unbiasedness is marginal over all randomness,
$\mathbb{E}[\tilde\ell_{i,v}\mid \ell_{i,v}]=\ell_{i,v}$; resolution is bounded
($|\tilde\ell_{i,v}-\ell_{i,v}|\le C_\Delta L_\ell 2^{-B/V}$, i.e.,\ no cell wider than
the uniform step up to a constant);
reconstruction errors are conditionally independent across coordinates and
nodes given the input logits; and the
aggregator emits
the softmax of the $K$-fold average $\bar\ell=\frac1K\sum_i\tilde\ell_i$. This
class contains Assumption~\ref{ass:A5}'s dithered uniform quantizer as one instance, and
also includes non-uniform, companding, and learned fixed-rate scalar quantizers
whenever they satisfy these same unbiasedness, bounded-resolution, and
conditional-independence requirements.
All risks in this subsection use Definition~\ref{def:bandwidth}'s
population-logit isolation, so $\ell_{i,v}=\ell_v^\star(X)$.

\paragraph{Finite-resolution logit spread.}
Any scalar rate--distortion lower bound needs the source to remain
spread at the resolution being studied. Put $b:=B/V\in\mathbb N$,
$M:=2^b$, and
$d_{\mathcal C}(z):=\min_{c\in\mathcal C}(z-c)^2$. The condition used
in Proposition~\ref{prop:rebuttal-lb}(i) is
\begin{equation}\label{eq:QS-logit}
\textbf{(QS-$\ell$)}\qquad
\inf_{\mathcal C:\,|\mathcal C|\le M}
\mathbb E_Xd_{\mathcal C}(\ell_v^\star(X))
\ge c_{\mathrm{sq}}M^{-2}
\quad\text{for every }v.
\end{equation}
Unlike a density hypothesis, \eqref{eq:QS-logit} is meaningful for the
finite context space $\mathcal X$ used in this paper. For a finitely
supported source it is necessarily a finite-resolution condition: it
cannot persist once $M$ reaches the support size. Rate statements for
the scalar class therefore refer to ranges of $B/V$ on which
$c_{\mathrm{sq}}$ is bounded away from zero.

\begin{proof}[Proof of Proposition~\ref{prop:rebuttal-lb}(i)]
Fix a node $i$ and coordinate $v$, write $x := \ell^\star_v(X)$, and
condition on the public randomness $\Omega_i$: the reconstruction then lies in
the fixed set $\mathcal C_{i,v}(\Omega_i)$ of at most $M=2^{B/V}$ points, so
$\mathbb{E}[(\tilde\ell_{i,v}-x)^2\mid X, \Omega_i]
\ge d_{\mathcal C_{i,v}(\Omega_i)}(x)$. Marginal unbiasedness gives
$\mathrm{Var}(e_{i,v}\mid X)=\mathbb{E}[(\tilde\ell_{i,v}-x)^2\mid X]
\ge\mathbb{E}_{\Omega_i}[d_{\mathcal C_{i,v}(\Omega_i)}(x)]$. The uniform infimum in
\eqref{eq:QS-logit} survives the average over $\Omega_i$, so Fubini gives
\[
\mathbb{E}_X[\mathrm{Var}(e_{i,v}\mid X)]
\;\ge\;\mathbb{E}_{\Omega_i}\,\mathbb{E}_X[
d_{\mathcal C_{i,v}(\Omega_i)}(\ell^\star_v(X))]
\;\ge\;c_{\mathrm{sq}}M^{-2}
=c_{\mathrm{sq}}2^{-2B/V}.
\]
Independence across nodes gives
$\mathrm{Var}(\bar e_v\mid X)=K^{-2}\sum_i\mathrm{Var}(e_{i,v}\mid X)$; taking
$\mathbb{E}_X$ yields $\mathbb{E}_X[\mathrm{Var}(\bar e_v\mid X)]\ge c_{\mathrm{sq}}K^{-1}2^{-2B/V}$.
(The dithered uniform quantizer of Theorem~\ref{thm:matching-lb} has
$X$-independent variance $\ge\tfrac{L_\ell^2}{3}2^{-2B/V}$ and also
satisfies this floor whenever \textbf{(QS-$\ell$)} holds.)
\end{proof}

\paragraph{A continuous-source sufficient condition.}
If, in a continuous-context extension, the law of
$\ell_v^\star(X)$ has density at least $p_{\min}$ on an interval of
length $a$, then \eqref{eq:QS-logit} holds with
$c_{\mathrm{sq}}=p_{\min}a^3/32$. Indeed, the points of any
$M$-point set lie within distance $r=a/(4M)$ of a subset having
Lebesgue measure at most $2Mr=a/2$. On the remaining half-interval
$d_{\mathcal C}\ge r^2$, giving
$p_{\min}(a/2)r^2=p_{\min}a^3/(32M^2)$.

\paragraph{From mean-squared error to KL.} The bandwidth KL of
Definition~\ref{def:bandwidth} is the softmax-Hessian-weighted version
of this error,
$\mathbb{E}[\mathrm{KL}]_{\mathrm{bandwidth}}
=\tfrac12\,\mathbb{E}_X\sum_v H_{\mathrm{soft},vv}(p^\star)\,
\mathrm{Var}(\bar e_v\mid X)+R_3$ (with $\mathrm{Cov}(\bar e\mid X)$ diagonal by
cross-coordinate independence). For the subtractively dithered construction of
Theorem~\ref{thm:matching-lb} the node-level per-coordinate variance is
\emph{independent of $X$} and equals exactly
$\Delta_B^2/12 \ge \tfrac{L_\ell^2}{3}2^{-2B/V}$, so
$\mathrm{Var}(\bar e_v\mid X)\ge\tfrac{L_\ell^2}{3K}2^{-2B/V}$ after averaging,
and the variance is constant in $X$. Plugging this actual averaged variance
into the trace, with
$\mathbb{E}_X[1-\|p^\star\|_2^2]\ge c_p$ from \textbf{(ND)}, the trace identity
$\mathrm{tr}(H_{\mathrm{soft}})=1-\|p\|_2^2$ (Lemma~\ref{lem:soft-trace}), and
$R_3$ controlled under \textbf{(SE$_{c_p}$)}, recovers
Theorem~\ref{thm:matching-lb}.
For a \emph{general} unbiased scalar quantizer the variance depends on $X$, so the
trace no longer factorizes; we close the gap with the $H_{\mathrm{soft}}$-weighted
analogue of \textbf{(QS-$\ell$)}. For each coordinate, let $\mu_v$ be
the corresponding $H_{\mathrm{soft}}$-weighted logit measure under
$X \sim P^\star_X$,
\[
\mu_v(A)=\mathbb{E}_X\Bigl[H_{\mathrm{soft},vv}(p^\star(X))\,
\mathbf{1}\{\ell^\star_v(X)\in A\}\Bigr],\qquad
\sum_v\mu_v(\mathbb{R})=
\mathbb{E}_X\bigl[1-\|p^\star\|_2^2\bigr]\ge c_p,
\]
so their total mass is bounded below by the non-degeneracy constant of
\textbf{(ND)}. We assume this weighted mass remains spread when each
coordinate may choose its own $M$-point reconstruction set:
\begin{equation}\label{eq:QS-H}
\textbf{(QS-H)}\qquad
\inf_{\substack{\mathcal C_1,\ldots,\mathcal C_V:\\
                  |\mathcal C_v|\le M}}
\sum_{v=1}^V\int d_{\mathcal C_v}(z)\,d\mu_v(z)
\ge c_HM^{-2}.
\end{equation}
This is the $H_{\mathrm{soft}}$-weighted version of
\textbf{(QS-$\ell$)}. For finite $\mathcal X$ it is again a
finite-resolution condition, not an asymptotic density claim.

With $\{\mu_v\}$ and \textbf{(QS-H)} in hand, part~(ii) follows
from this quantization-spread condition plus an exact Hessian-comparison
identity.

\begin{proof}[Proof of Proposition~\ref{prop:rebuttal-lb}(ii)]
For node $i$, part~(i)'s argument gives
$\mathrm{Var}(e_{i,v}\mid X)\ge d_{i,v}(\ell^\star_v(X))$ with
$d_{i,v}(x) := \mathbb{E}_{\Omega_i}\bigl[d_{\mathcal C_{i,v}(\Omega_i)}(x)\bigr]$, so
\[
\mathrm{Var}(\bar e_v\mid X)
= K^{-2}\sum_i\mathrm{Var}(e_{i,v}\mid X)
\ge K^{-2}\sum_i d_{i,v}(\ell^\star_v(X)).
\]
For each realized $\Omega_i$, \textbf{(QS-H)} applies to the collection
$\{\mathcal C_{i,v}(\Omega_i)\}_{v=1}^V$. Averaging over $\Omega_i$ therefore gives
\[
\sum_v\int d_{i,v}(z)\,d\mu_v(z)
=\mathbb E_{\Omega_i}\sum_v\int
d_{\mathcal C_{i,v}(\Omega_i)}(z)\,d\mu_v(z)
\ge c_HM^{-2}=c_H2^{-2B/V}.
\]
Summing over
nodes, $\mathbb{E}_X\sum_v H_{\mathrm{soft},vv}\mathrm{Var}(\bar e_v\mid X)\ge
c_H\,K^{-1}2^{-2B/V}$.

For a general non-symmetric quantizer the odd terms of a cumulant
expansion do not cancel, so we avoid the expansion altogether and use
an exact integral representation, valid for every realization of
$\bar e$. With $A(\ell) := \log\sum_v e^{\ell_v}$,
$p := \mathrm{softmax}(\ell^\star(X))$, and
$q_s := \mathrm{softmax}(\ell^\star(X) + s\bar e)$, Taylor's theorem
with integral remainder gives
\[
\mathrm{KL}\bigl(p \,\|\, \mathrm{softmax}(\ell^\star(X) + \bar e)\bigr)
= A(\ell^\star + \bar e) - A(\ell^\star) - p^\top\bar e
= \int_0^1 (1-s)\, \bar e^\top H_{\mathrm{soft}}(q_s)\,\bar e\; ds.
\]
The class's bounded resolution gives
$\|\bar e\|_\infty \le C_\Delta L_\ell 2^{-B/V} =: h$
deterministically, whence $q_{s,v}/p_v \in [e^{-2h}, e^{2h}]$ for all
$s, v$ (the numerator and the partition function each move by a
factor $e^{\pm h}$); by the pairwise form
$z^\top H_{\mathrm{soft}}(q)\,z = \tfrac12\sum_{v,w} q_v q_w (z_v - z_w)^2$,
\[
\bar e^\top H_{\mathrm{soft}}(q_s)\,\bar e
\;\ge\; e^{-4h}\, \bar e^\top H_{\mathrm{soft}}(p)\,\bar e,
\qquad\text{so}\qquad
\mathrm{KL} \;\ge\; \tfrac12\, e^{-4h}\, \bar e^\top H_{\mathrm{soft}}(p)\,\bar e.
\]
Taking $\mathbb{E}[\,\cdot\mid X]$ (cross-coordinate independence and
mean-zero errors kill the off-diagonal terms of $H_{\mathrm{soft}}$)
and then $\mathbb{E}_X$,
\[
\mathbb{E}[\mathrm{KL}]_{\mathrm{bandwidth}}
\ge \tfrac12 e^{-4h}\,
\mathbb{E}_X\!\sum_v H_{\mathrm{soft},vv}\mathrm{Var}(\bar e_v \mid X)
\ge \tfrac12 e^{-4h} c_H\, K^{-1} 2^{-2B/V}
\ge \tfrac14 c_H\, K^{-1} 2^{-2B/V}
\]
under the deterministic resolution condition
\begin{equation}\label{eq:SE2}
\textbf{(RES)}\qquad
C_\Delta\, L_\ell\, 2^{-B/V} \;\le\; \tfrac{\ln 2}{4},
\end{equation}
which makes $e^{-4h} \ge \tfrac12$. Unlike \textbf{(SE$_{c_p}$)},
\textbf{(RES)} is $K$-free, and no control of the cumulant series
is needed: the Hessian-ratio bound is uniform in $p$, including
near-singular directions, because the same $H_{\mathrm{soft}}$
appears on both sides.
\end{proof}

\paragraph{A continuous-source sufficient condition for
\textbf{(QS-H)}.}
Suppose that, for coordinates in a nonempty set $\mathcal I$, there is
an interval $J_v$ of length $a_v$ on which
$\mu_v(A)\ge\beta_v|A|$ for every measurable $A\subseteq J_v$, with
$\beta_v>0$. Applying the covering argument used after
\eqref{eq:QS-logit} separately to every coordinate gives
$c_H=\sum_{v\in\mathcal I}\beta_v a_v^3/32$. This is a convenient sufficient condition for a
continuous-context extension; the proposition itself uses
\eqref{eq:QS-H}, which also covers finite context distributions over
the stated range of resolutions.

\paragraph{Remark (matching upper bound).} The bandwidth-term \emph{upper} bound
of Theorem~\ref{thm:fpld-upper} likewise holds for any quantizer in the
bounded-resolution unbiased scalar class $\mathcal P_B^{\mathrm{scalar}}$
whose per-coordinate error variance is at most $C_q'\,2^{-2B/V}$, with
direct-teacher coefficient $C_q'/2$. The derivation of
Appendix~\ref{ssec:trace-upper} uses this variance bound
($\mathrm{Cov}(\bar\xi)\preceq\sigma^2 I_V$,
$\sigma^2=C_q'2^{-2B/V}/K$), the trace bound
$\mathrm{tr}(H_{\mathrm{soft}})\le1$, and the class's deterministic
bounded-resolution condition to control, rather than cancel, the cubic
remainder; it does not use the specific dithered-uniform form. Together with
Proposition~\ref{prop:rebuttal-lb} this pins the
bandwidth rate at $\Theta(K^{-1}2^{-2B/V})$ for the entire bounded-resolution
unbiased scalar
class, in mean-squared error under \textbf{(QS-$\ell$)}, and in KL under
\textbf{(QS-H)}. Over rate ranges where the spread constants
are uniform, the quantizer affects the constants --- $c_{\mathrm{sq}}$
in the MSE lower bound, $c_H/4$ in the KL lower bound, and $C_q'/2$
in the upper bound --- but not the rate.

\paragraph{Remark (vector quantizers; remaining scope).} The preceding matching
lower- and upper-bound arguments use cross-coordinate independence to keep
$\mathrm{Cov}(\bar e)$ diagonal, which is why their common rate statement is
for scalar quantizers. Section~\ref{sec:uncond}
treats arbitrary single-round $B$-bit encoders and decoders --- including vector,
biased, and non-averaging constructions --- and proves that their exponent is
governed by the effective dimension $V-1$ and the total budget $KB$. What remains
open is a matching separation for the full \emph{symmetric} vector-encoder class
of Proposition~\ref{prop:symmetry}, as well as unrestricted interactive
multi-round protocols (Section~\ref{sec:discuss}).

\subsection{Vanilla multi-round FPLD improves only polynomially}\label{app:rebuttal-vanilla}

Theorem~\ref{thm:lb-mr} shows that $T$-round \emph{residual} refinement reaches
$K^{-1}(2^{B/V}-1)^{-2T}$ up to constants ($2^{-2TB/V}$ up to the
finite-codebook factor). For comparison, we record the rate of the natural
\emph{vanilla} multi-round baseline, in which each round re-quantizes the full
logit and the aggregator averages across rounds.

\begin{lemma}[Cross-round averaging]\label{lem:rebuttal-vanilla}
Consider $T$-round FPLD in which every node, in each round, transmits a dithered
$B/V$-bit quantization of the full logit (Assumption~\ref{ass:A5}, clip $L_\ell$),
with dither independent across nodes and across rounds, and the aggregator forms
the average $\hat\ell=(KT)^{-1}\sum_{t=1}^{T}\sum_{i=1}^{K}\tilde\ell_i^{(t)}$ of
all $KT$ quantized logits. Under \textbf{(ND)} and \textbf{(SE$_{c_p}$)},
\[
\mathbb{E}\bigl[\mathrm{KL}(P^\star\,\|\,\hat P)\bigr]_{\mathrm{bandwidth}}
\;=\;\Theta\!\bigl((KT)^{-1}\,2^{-2B/V}\bigr),
\]
a polynomial ($1/T$) improvement over the single-round rate. Moreover, the
average is the best linear unbiased combination of the $KT$ quantized logits.
Thus, within that aggregation class, vanilla full-logit re-quantization improves
the bandwidth term only polynomially.
At the same total budget $TB$, this is exponentially worse than the
$(2^{B/V}-1)^{-2T}$ of residual refinement (Theorem~\ref{thm:lb-mr},
$B \ge 2V$).
\end{lemma}

\begin{proof}
In the bandwidth isolation of Definition~\ref{def:bandwidth},
$\ell_i^{(t)}=\ell^\star$, so
$\tilde\ell_i^{(t)}=\ell^\star+\xi_i^{(t)}$, where the round-$t$
quantization error $\xi_i^{(t)}$ is mean-zero given $X$ and has
per-coordinate variance $C_q2^{-2B/V}$
(Assumption~\ref{ass:A5}), independent across $(i,t)$ and across coordinates. The averaged
error $\bar e=(KT)^{-1}\sum_{t,i}\xi_i^{(t)}$ thus has diagonal covariance with
each entry
\[
\mathrm{Var}(\bar e_v\mid X)
=\frac{1}{(KT)^2}\sum_{t=1}^{T}\sum_{i=1}^{K}\mathrm{Var}(\xi_{i,v}^{(t)}\mid X)
=\frac{C_q}{KT}\,2^{-2B/V}.
\]
The softmax-KL cumulant expansions in
Appendices~\ref{ssec:trace-upper} and~\ref{ssec:matching-lb-proof} give
\[
\mathbb{E}[\mathrm{KL}]_{\mathrm{bandwidth}}
=\tfrac12\,\mathbb{E}_X\,\mathrm{tr}\!\bigl(H_{\mathrm{soft}}(p^\star)\,
\mathrm{Cov}(\bar e\mid X)\bigr)+R_3
=\tfrac12\,\frac{C_q\,2^{-2B/V}}{KT}\,\mathbb{E}_X\bigl[1-\|p^\star\|_2^2\bigr]+R_3,
\]
using $\mathrm{tr}(H_{\mathrm{soft}})=1-\|p\|_2^2$ (Lemma~\ref{lem:soft-trace}).
Since $c_p\le\mathbb{E}_X[1-\|p^\star\|_2^2]\le1$ by \textbf{(ND)} and $R_3$ is
dominated under \textbf{(SE$_{c_p}$)}, both bounds are $\Theta((KT)^{-1}2^{-2B/V})$.
Finally, among unbiased linear combinations
$\hat\ell=\sum_{t,i}a_{i,t}\tilde\ell_i^{(t)}$ with $\sum_{t,i}a_{i,t}=1$,
Cauchy--Schwarz gives $\sum_{t,i}a_{i,t}^2\ge1/(KT)$, with equality at the uniform
average; hence averaging minimizes the error variance.
\end{proof}

\noindent The $1/T$ rate of Lemma~\ref{lem:rebuttal-vanilla} is the one observed
empirically for the cross-round-averaging baseline in the multi-round simulation.

\subsection{Additional exact finite-codebook checks}
\label{app:synthetic-extra}

These three studies accompany the single-round rate check retained in
Section~\ref{sec:exp}, and use the same exact $2^b$-point codecs.

\paragraph{Heterogeneous allocation.} For $K=4$ and caps
$L=(1,1,4,4)$, Figure~\ref{fig:fig2} evaluates the integer policies on
the same exact channel. We display total rates divisible by $K$, so the
uniform policy assigns the same integer rate to every node. At
$B_{\mathrm{tot}}/V=8$, finite-exact allocation uses rates $(1,1,3,3)$
and reduces KL from $0.03920$ under uniform $(2,2,2,2)$ to $0.02726$, a
$30.4\%$ reduction. At $B_{\mathrm{tot}}/V=20$, the reduction is
$52.9\%$. The finite-exact and integer-surrogate policies coincide at
these displayed budgets. The same pattern holds when the four nodes have
distinct local population logits; full values are in
Appendix~\ref{app:rebuttal-reallm}.

\begin{figure}[!tbp]
\centering
\includegraphics[width=0.95\linewidth]{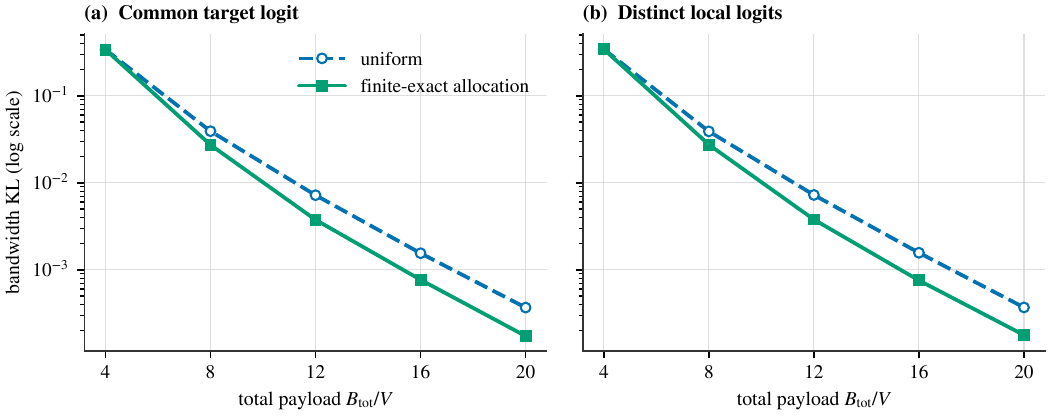}
\caption{At equal payload, finite-exact allocation matches uniform at the
one-bit floor and lowers observed KL once rates can differ, for both common
and distinct local logits. Exact finite-codebook comparisons use $K=4$,
caps $(1,1,4,4)$, $V=256$, and $100$ codec seeds. Panel (a) uses a common
target logit and panel (b) distinct local population logits. Displayed
budgets are divisible by $K$; finite-exact and integer-surrogate
allocations coincide.}
\label{fig:fig2}
\end{figure}

\paragraph{Vocabulary, robustness, and quantizer diagnostics.}
Holding $K=4$, $b=4$, the target-generation rule, and the codec cap fixed,
a controlled $2{,}374$-fold sweep from $V=64$ to $151{,}936$ keeps empirical KL
between $0.9917$ and $1.0018$ times its quadratic prediction
(Figure~\ref{fig:vocabulary}). This is the controlled dimension check;
the real-model comparison below is not, because architecture and logit law
also change. Exhaustive sign perturbations of the four allocation weights
at six error radii satisfy both versions of
Corollary~\ref{cor:adaptive}'s bound: the largest ratio across both objectives
is $2.125$ for the surrogate at $\varepsilon=1$, below
$2^{2\varepsilon}=4$. Finally, exact
finite stochastic rounding and exact finite dithered quantization both
show the predicted exponent. At $b=8$ their KL values are
$4.97\times10^{-6}$ and $2.53\times10^{-6}$, respectively; this compares
two codecs, not all members of Proposition~\ref{prop:rebuttal-lb}'s class.

\begin{figure}[!tbp]
\centering
\includegraphics[width=0.95\linewidth]{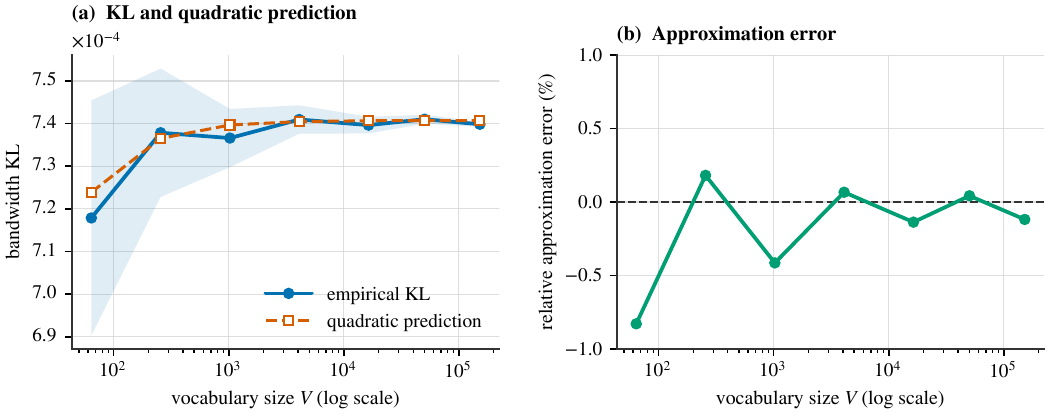}
\caption{Across a $2{,}374$-fold vocabulary sweep, mean empirical KL stays
within $0.83\%$ of the trace-sharpened quadratic prediction. The exact-channel
study fixes $K=4$, $b=4$, the target-generation rule, and the codec cap;
the realized softmax trace varies with $V$ and is included in the prediction.
Panel (a) compares empirical KL with that prediction, and panel (b) reports
$100(\mathrm{KL}_{\mathrm{emp}}/\mathrm{KL}_{\mathrm{quad}}-1)$. Bands are
mean $\mathbin{\pm}1.96$ Monte Carlo standard errors over $100$ codec seeds.}
\label{fig:vocabulary}
\end{figure}

\paragraph{Multiple rounds.} Figure~\ref{fig:multiround-exact} uses two
bits per coordinate per round. The KL under residual refinement falls from
$1.85\times10^{-2}$ at $T=1$ to $3.10\times10^{-7}$ at $T=6$, whereas
retransmitting the full logit and keeping the latest result is flat and
cross-round averaging improves only at its $1/T$ baseline rate.

\begin{figure}[!tbp]
\centering
\includegraphics[width=0.76\linewidth]{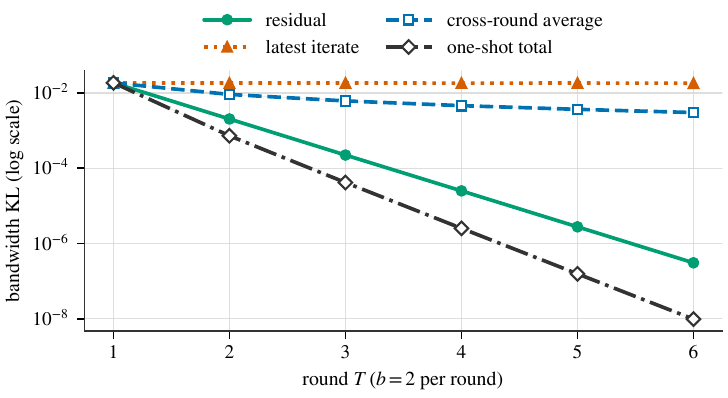}
\caption{Residual refinement lowers KL by nearly five orders of magnitude
over six rounds, while keeping only the latest full-logit transmission is
flat and cross-round averaging follows its $1/T$ baseline. The exact-channel
comparison uses $K=4$, $V=256$, $b=2$ bits per coordinate per round, and
$100$ codec seeds. ``One-shot total'' is one round at the matched total
payload $Tb$; it is a reference, not a sequential protocol.}
\label{fig:multiround-exact}
\end{figure}

\subsection{Rate-$35$ interior check of the exponential surrogate}
\label{app:rate35}

At $B_{\mathrm{tot}}/V=35$ all continuous allocations in
Table~\ref{tab:reallm-main} are interior, so Theorem~\ref{thm:alloc}'s
exponential surrogate predicts the relative improvement
$1-\operatorname{GM}(w)/\operatorname{AM}(w)$. This gives $12.19\%$ on
WikiText-2 and $14.36\%$ on LAMBADA, while the observed rate-$35$ KL
reductions are $12.24\%$ $[11.66\%,12.83\%]$ and $14.14\%$
$[13.54\%,14.73\%]$. This is a descriptive high-rate check of the
quadratic surrogate, not an identity for empirical KL\@.

\subsection{Selected exact-codebook and real-logit experiments}
\label{app:rebuttal-reallm}

\paragraph{Sources and codecs.} The practical GPT allocation results in
Table~\ref{tab:reallm-main} are frozen under
\texttt{runs/practical\_gpt\_6657a2299108}; the two production replays and
their audit all record the execution commit
\texttt{6657a2299108238e51ff386272ac73e79257f2ed}. They use the nominal
subtractive-dither lattice $\Delta=2L/2^b$ and continuous per-node rates,
as stated in the main text. The exact-codebook results come from JSONs
with \texttt{smoke=false} under the campaign directory
\texttt{runs/8ecd15a66e4c}; every selected job records the execution commit
\texttt{9488fbd6c8f29a806671e740a59ddb0f12dc8e52}. The exact scalar codec
has $2^b$ reconstruction points and step $2L/(2^b-1)$. All selected
integer policies obey the same rate floor and exact total payload.

Test text from WikiText-2 and LAMBADA is converted into consecutive 128-token
windows after nonempty rows are joined in source order. The GPT allocation
study uses $2{,}000$ windows. Homogeneous Pythia and Falcon3 rate and
multi-round checks use $1{,}000$ windows. The WikiText-2 dataset revision is
\texttt{\seqsplit{b08601e04326c79dfdd32d625aee71d232d685c3}} and the
LAMBADA revision is
\texttt{\seqsplit{900124bf3b8235c6daf21033af9948b3f07346c4}}. The latter
is a text-stream probe rather than the standard passage-final LAMBADA
task. The practical confirmation uses $50$ paired codec seeds and $5{,}000$
crossed-bootstrap draws. The crossed bootstrap treats the observed contexts
as an empirical context distribution and independently resamples the paired
codec-seed axis and the context axis. Its intervals therefore describe
two-way empirical variability conditional on the locked checkpoints and
source streams; they are not population-coverage statements for model
architectures or corpora. The WikiText-2 perplexity contrast at rate $14$
was selected from the earlier three-seed pilot and designated as the primary
endpoint before the confirmation run; the seed count, bootstrap, and analysis
were then locked before execution. All other practical contrasts are
descriptive secondary estimates with unadjusted intervals. Exact real
policies use $20$ codec seeds and $2{,}000$ crossed-bootstrap draws under the
same conditional interpretation.

\begin{table}[ht]
\centering\footnotesize
\setlength{\tabcolsep}{3.5pt}
\renewcommand{\arraystretch}{1.06}
\caption{Only vocabulary-compatible GPT checkpoints are averaged in the
reported federation; Pythia and Falcon3 are used solely for homogeneous
rate and multi-round replications. The table lists the locked checkpoints;
full revisions, tokenizer hashes, and probe hashes are in the lock JSONs.
Exact parameter counts for the two largest models are $2{,}775{,}208{,}960$
(Pythia-2.8B) and $3{,}227{,}655{,}168$ (Falcon3-3B).}
\label{tab:checkpoint-locks}
\begin{tabular}{llrp{0.22\linewidth}}
\toprule
role & checkpoint ID & parameters & revision prefix\\
\midrule
GPT federation & \texttt{distilgpt2} & 82M & \texttt{2290a62682d0}\\
GPT federation & \texttt{gpt2} & 124M & \texttt{607a30d783df}\\
GPT federation & \texttt{gpt2-medium} & 355M & \texttt{6dcaa7a952f7}\\
GPT federation & \texttt{gpt2-large} & 774M & \texttt{32b71b12589c}\\
GPT federation & \texttt{gpt2-xl} & 1.5B & \texttt{15ea56dee5df}\\
GPT federation & \texttt{EleutherAI/gpt-neo-125M} & 125M & \texttt{21def0189f57}\\
GPT federation & \texttt{EleutherAI/gpt-neo-1.3B} & 1.3B & \texttt{dbe59a7f4a88}\\
rate/refinement & \texttt{EleutherAI/pythia-2.8b} & 2.8B & \texttt{2a259cdd96a4}\\
rate/refinement & \texttt{tiiuae/Falcon3-3B-Base} & 3.2B & \texttt{092c29e3114f}\\
\bottomrule
\end{tabular}
\end{table}

\paragraph{Practical GPT allocation.} The seven required WikiText-2 caps are
\[
(41.65,53.27,76.66,48.19,56.32,68.46,33.62),
\]
giving weight spread $5.20$. For the LAMBADA stream they are
\[
(30.20,56.45,78.63,54.74,59.74,55.62,37.18),
\]
with spread $6.78$.
The two policies use identical standardized dither arrays for every
corpus, rate, seed, node, context, and coordinate. This pairing preserves
each policy's marginal distribution. Budget residuals are at most
$7.11\times10^{-15}$, so differences are floating-point summation noise.

\begin{table}[ht]
\centering\small
\setlength{\tabcolsep}{5pt}
\renewcommand{\arraystretch}{1.08}
\caption{Water-filling produces positive paired KL and perplexity reductions
at every displayed rate on both streams. The practical continuous-rate
comparison uses the seven-model GPT federation. The brackets report paired
crossed-bootstrap $95\%$ percentile intervals over codec seeds and observed
contexts, conditional on the locked checkpoints and source streams:
WikiText-2 perplexity at rate $14$ is the prespecified primary endpoint, and
all other intervals are descriptive and unadjusted.}
\label{tab:reallm-practical}
\begin{tabular}{llrr}
\toprule
probe & $B_{\mathrm{tot}}/V$ & KL reduction (\%) & perplexity reduction (\%)\\
\midrule
WikiText-2 & 14 & 10.35 [9.90,10.80] & 27.78 [25.50,29.96]\\
& 21 & 11.74 [11.13,12.34] & 8.94 [7.27,10.56]\\
& 28 & 11.93 [11.32,12.54] & 2.67 [1.79,3.55]\\
& 35 & 12.24 [11.66,12.83] & 0.65 [0.18,1.10]\\
\addlinespace
LAMBADA stream & 14 & 12.26 [11.82,12.70] & 31.05 [28.82,33.21]\\
& 21 & 13.72 [13.12,14.31] & 10.30 [8.69,11.99]\\
& 28 & 14.30 [13.67,14.90] & 2.76 [1.89,3.66]\\
& 35 & 14.14 [13.54,14.73] & 0.97 [0.47,1.49]\\
\bottomrule
\end{tabular}
\end{table}

At rate $14$, every one of the $50$ seed-level means favors
water-filling on both corpora for both KL and perplexity. Because the
probe windows are consecutive in source order, we also ran a post hoc
dependence sensitivity analysis. It resamples seeds and overlapping
circular blocks of ordered contexts while preserving policy pairing.
Table~\ref{tab:reallm-block} shows that blocks as long as $256$ contexts
leave the primary conclusion unchanged. This sensitivity analysis is not
part of the prespecified confirmation analysis.

\begin{table}[ht]
\centering\small
\renewcommand{\arraystretch}{1.08}
\caption{The rate-$14$ perplexity reductions remain positive when ordered
contexts are resampled in circular blocks as long as $256$. Entries are
paired $95\%$ intervals from the post hoc sensitivity analysis
($5{,}000$ draws), which was not part of the prespecified confirmation.}
\label{tab:reallm-block}
\begin{tabular}{lrrrr}
\toprule
probe & block 4 & block 16 & block 64 & block 256\\
\midrule
WikiText-2 & [25.54,29.96] & [25.49,29.95] & [25.46,29.95] & [25.27,30.37]\\
LAMBADA stream & [28.80,33.16] & [28.84,33.21] & [28.83,33.19] & [28.89,33.11]\\
\bottomrule
\end{tabular}
\end{table}

At rate $35$, water-filling perplexity is $20.7785$ against a
WikiText-2 teacher value of $19.8742$, a relative excess of $4.55\%$
$[4.23\%,4.88\%]$. The LAMBADA values are $29.0524$ and $27.7807$, an excess
of $4.58\%$ $[4.25\%,4.92\%]$. These are descriptive crossed-bootstrap
intervals. The corresponding excesses under uniform allocation are
$5.23\%$ and $5.60\%$.

All continuous water-filling rates are positive, so the interior
exponential-surrogate reduction is exactly
$1-\operatorname{GM}(w)/\operatorname{AM}(w)$. It equals $12.19\%$ on
WikiText-2 and $14.36\%$ on LAMBADA\@. Both surrogate predictions lie
within the paired intervals around the empirical rate-$35$ KL reductions
in Table~\ref{tab:reallm-practical}. This agreement is evidence that the
high-rate empirical KL is in the quadratic regime; the formula is exact
for the surrogate, not for empirical KL\@.

\paragraph{Strict finite-codebook allocation.} The exact GPT replay
rounds required caps upward on the public scale grid and uses integer
per-coordinate rates. \rev{The rounding changes the WikiText-2 cap range from
$33.62$--$76.66$ to $34.90$--$83.00$.} Table~\ref{tab:reallm-alloc} reports the rows at
which the finite-exact rate vector differs from uniform. The rate vector
follows the GPT model order in Table~\ref{tab:checkpoint-locks}.

\begin{table}[ht]
\centering\footnotesize
\setlength{\tabcolsep}{5pt}
\renewcommand{\arraystretch}{1.12}
\caption{Exact finite-codebook allocation reduces KL in every displayed
case where the optimized integer allocation differs from uniform. The
comparisons use the seven-model GPT federation; brackets are descriptive
paired $95\%$ crossed-bootstrap percentile intervals for percent KL
reduction from uniform, conditional on the locked checkpoints and source
streams.}
\label{tab:reallm-alloc}
\begin{tabular}{@{}llcrrr@{}}
\toprule
probe & $B_{\mathrm{tot}}/V$ & \shortstack{finite-exact\\rates} &
\shortstack{uniform\\KL} & \shortstack{exact\\KL} &
\shortstack{reduction (\%)\\$[95\%\ \mathrm{interval}]$}\\
\midrule
WikiText-2 & 28 & $(4,4,5,4,4,4,3)$ & 0.25516 & 0.24571 &
\shortstack[r]{3.70\\$[2.96,4.41]$}\\
WikiText-2 & 35 & $(5,5,6,5,5,5,4)$ & 0.06036 & 0.05714 &
\shortstack[r]{5.33\\$[4.52,6.09]$}\\
\addlinespace
LAMBADA stream & 21 & $(2,3,4,3,3,3,3)$ & 1.17325 & 1.13776 &
\shortstack[r]{3.02\\$[2.46,3.59]$}\\
LAMBADA stream & 28 & $(3,4,5,4,4,4,4)$ & 0.26429 & 0.24707 &
\shortstack[r]{6.51\\$[5.80,7.24]$}\\
\bottomrule
\end{tabular}
\end{table}

The static WikiText-2 scale header uses four bits at each of seven nodes
and is shared over $2{,}000$ probes, for $0.014$ header bits per probe
context in total, separate from payload. This accounting conditions on
one-time registration of the calibrated grid endpoints.

\paragraph{Rate and empirical curvature.} \rev{To test the communication laws
beyond GPT-2 without mixing vocabularies in a federation, GPT-2,
Pythia-2.8B, and Falcon3-3B Base are each replicated across nodes with
independent dithers, every replay using the exact finite codebook.
Table~\ref{tab:reallm-rate}}
reports the empirical softmax trace, the range over
$K\in\{1,2,4,8\}$ of slopes fitted to $\log_2(\mathrm{KL})$ at $b\ge3$,
and the empirical-to-quadratic ratio at $K=8,b=8$. \rev{Across the
six model--probe pairs (WikiText-2 and the LAMBADA stream), fixed-probe
mean softmax traces range from $0.624$ to $0.792$; these values are
descriptive evidence for non-degeneracy.}
Pythia retains its
native output dimension $50{,}304$ and Falcon3 uses $131{,}072$
coordinates. The quantities are fixed-probe diagnostics rather than
population confidence bounds. \rev{Within these diagnostics, the finite-rate exponent and the
sequential-refinement gain of Table~\ref{tab:real-multiround} persist for
models with up to $3.2$ billion parameters.}

\paragraph{Multi-round refinement.} Table~\ref{tab:real-multiround}
gives exact-codebook residual-refinement trajectories at two
bits per coordinate per round. Each model is replicated across four nodes
with independent dithers.

\begin{table}[ht]
\centering\small
\renewcommand{\arraystretch}{1.06}
\caption{Exact residual refinement reduces real-logit KL by factors of
approximately $6{,}000$--$44{,}000$ across GPT-2, Pythia-2.8B, and
Falcon3-3B. Each model is replicated over $K=4$ nodes and uses $b=2$ bits
per coordinate per round with independent dithers; the final round $T$ is
listed by model.}
\label{tab:real-multiround}
\begin{tabular}{lrrr}
\toprule
model & final $T$ & KL at $T=1$ & KL at final $T$\\
\midrule
GPT-2 & 6 & 7.616 & $1.72\times10^{-4}$\\
Pythia-2.8B & 5 & 2.222 & $3.66\times10^{-4}$\\
Falcon3-3B & 5 & 2.449 & $4.08\times10^{-4}$\\
\bottomrule
\end{tabular}
\end{table}

\begin{table}[ht]
\revfloat\centering\small
\renewcommand{\arraystretch}{1.06}
\caption{Exact finite-codebook rate fits remain near the theoretical slope
$-2$, and empirical KL is $0.999$--$1.006$ times the trace-sharpened
quadratic prediction at $(K,b)=(8,8)$. Each row replicates one model across
nodes with independent dithers; no cross-family logits are averaged. Slopes
are fitted over $b\ge3$ for $K\in\{1,2,4,8\}$, and the final column is the
empirical-to-predicted ratio.}
\label{tab:reallm-rate}
\begin{tabular}{llrrr}
\toprule
model & probe & empirical trace & slope range & ratio at $(K,b)=(8,8)$\\
\midrule
GPT-2 & WikiText-2 & 0.7742 & $[-2.064,-1.931]$ & 0.9999\\
GPT-2 & LAMBADA stream & 0.7920 & $[-2.038,-1.960]$ & 1.0060\\
Pythia-2.8B & WikiText-2 & 0.6350 & $[-2.068,-2.021]$ & 1.0006\\
Pythia-2.8B & LAMBADA stream & 0.7318 & $[-2.042,-2.022]$ & 1.0006\\
Falcon3-3B & WikiText-2 & 0.6236 & $[-2.068,-2.018]$ & 0.9995\\
Falcon3-3B & LAMBADA stream & 0.7339 & $[-2.044,-2.019]$ & 0.9990\\
\bottomrule
\end{tabular}
\end{table}

\paragraph{Additional exact synthetic checks.} For the common-target
allocation at $B_{\mathrm{tot}}/V=8$, finite-exact KL is $0.02726$
versus $0.03920$ for uniform. At rate $20$, the corresponding values are
$1.7316\times10^{-4}$ and $3.6788\times10^{-4}$. With distinct local
population logits, they are $0.02744$ and $0.03899$ at rate $8$, and
$1.7773\times10^{-4}$ and $3.6934\times10^{-4}$ at rate $20$. For the
exhaustive weight-robustness diagnostic, the largest ratio across the finite
and surrogate objectives is $2.125$ for the surrogate at log-weight error
$\varepsilon=1$, below $2^{2\varepsilon}=4$. At $b=8$, the KL values for the
exact dithered codec
and the finite stochastic scalar codec are $2.532\times10^{-6}$ and
$4.966\times10^{-6}$.

The smooth metadata path of Section~\ref{sec:exp} runs the scheduler of
Theorem~\ref{thm:streaming} for $T=200$ rounds at $K=4$, $V=256$, and
$B_{\mathrm{tot}}/V=12$, with each node's required cap rounded upward on the
public grid. Along it, the largest
one-step log-weight drift is $0.25$ and the largest lag/current exact
objective ratio is $1.070$, below $2^{2(0.25)}=1.414$; the reported scale
indices therefore both decode the finite codebooks and schedule the next
round within the theorem's deterministic bound. The cap grid has
a log-base-two step size of $1/8$ and is fixed before codec seeds and policies
are evaluated.

\paragraph{Computational provenance.} Models were evaluated in FP16 and
centered logit caches were stored in float32. GPU, environment, model,
tokenizer, dataset, and probe records are stored with every selected job.
All downstream codecs use the frozen caches. Extraction used Quadro RTX
6000 and Tesla V100 devices, as recorded with each job. Within
each displayed comparison, all policies replay the same hashed cache; we
make no cross-hardware comparison. The environment records
Python 3.11.9, NumPy 1.26.4, SciPy 1.13.1, PyTorch 2.6.0,
Transformers 4.51.3, Datasets 3.5.0, and Matplotlib 3.9.4. Selected raw
seed-by-context arrays, result JSONs, locks, probe copies, and SHA-256
checksums are included in the reproducibility package; the logit caches
are regenerable from the locked inputs.

\section{Proofs for Section~\ref{sec:uncond} (Unconditional Rates)}
\label{app:uncond}

\subsection{Grouped FPLD}\label{app:groupedstmt}

The global-index construction is fragile to node dropout and to
finite-sample disagreement across nodes. Coordinate splitting is a
robust finite-$n$ alternative, but each coordinate is then estimated
using only its assigned nodes, so full splitting loses the $K$-fold
statistical averaging ($d/n$ in place of $d/(Kn)$). A grouped
coordinate-splitting protocol interpolates between these endpoints.

\begin{theorem}[Grouped FPLD]\label{thm:grouped}
Fix $G \in \{1, \dots, K\}$ with $G \mid K$ and $G \mid V$.
Assume $GB/V\in\mathbb N$ and $GB/V\ge1$.
Partition nodes into $G$ groups of size $K/G$ and coordinates into
$G$ blocks of size $V/G$; each node in group $g$ transmits only
block $g$, scalar-quantized at $GB/V$ bits per coordinate
(Assumption~\ref{ass:A5}), and the aggregator averages within each
group and concatenates across groups. Set
\[
q_G:=\frac{L_\ell^2}{3(2^{GB/V}-1)^2},
\qquad \sigma_G^2:=\frac GKq_G.
\]
Under Assumptions~\ref{ass:A1}--\ref{ass:A7}, suppose
\[
\sigma_G(\log V)^{3/2}\le c_0,
\qquad
\frac{dK}{Gn}\le c_0',
\]
where $c_0,c_0'>0$ are sufficiently small constants depending only
on the regularity envelopes in Assumptions~\ref{ass:A1}
and~\ref{ass:A7}; $c_0$ is chosen so the quantization remainder below is
no larger than the displayed quadratic term. Then there are constants
$c_1,c_2,c_{\mathrm{opt}},c_{\mathrm{fit}},C_R>0$ and a dimension-free
$C_{\mathrm{dist}}\ge1$ such that
\[
\begin{aligned}
&\mathbb E_{\mathcal D,X_{\mathrm{probe}}}
\mathbb E_X\mathbb E_U\,\mathrm{KL}(P^\star \| \hat P)\\
&\quad\le
c_1 (1 + \bar\kappa)\frac{dG}{Kn}
+c_2\rho\sqrt{\frac{V\log V}m}
+\frac{C_{\mathrm{dist}}}{2}\sigma_G^2\\
&\qquad
+C_R\sigma_G^3(\log V)^{3/2}
+c_{\mathrm{opt}}\varepsilon_{\mathrm{opt}}
+c_{\mathrm{fit}}\rho\varepsilon_{\mathrm{fit}},
\end{aligned}
\]
where $\bar\kappa \ge 0$ is the between-context score ratio
\[
\bar\kappa:=\lambda_{\max}\!\left(
I(\theta^\star)^{-1/2}
\mathbb E_{X\sim P_X^\star}[\mu_X\mu_X^\top]
I(\theta^\star)^{-1/2}\right),
\qquad
\mu_x:=\sum_vp^\star(v\mid x)g_v(x),
\]
and $g_v(x):=\nabla_\theta\ell_{\theta^\star}(x)_v$, with
$(1+\bar\kappa)$ replaced by $1$ at $G=1$.
Thus $G=1$ recovers FPLD, while $G=K$ recovers full coordinate
splitting.
\end{theorem}

Appendix~\ref{app:kappa} collects the properties of $\bar\kappa$ and
the $G=1$ base case, where the statistical quadratic form is exactly
the conditional Fisher information and the bandwidth terms coincide
with Theorem~\ref{thm:fpld-upper}.

\paragraph{Finite-sample regime.} The condition binds hardest at
$G=1$, where it is the one-shot-averaging requirement $n\gtrsim Kd$
already used in Theorem~\ref{thm:fpld-upper}. It makes the local-MLE
bias and logit Taylor remainders in Assumption~\ref{ass:A7} no larger
than the leading $dG/(Kn)$ term. No separate
$G\log V/(Kn)$ condition is needed: Appendix~\ref{app:grouped}
controls the statistical nonlinearity by a global Hessian comparison
on the bounded canonical-logit set. The first condition in
Theorem~\ref{thm:grouped} controls only the quantization remainder.

\begin{corollary}[Optimal group size for the leading envelope]\label{cor:groupopt}
Under Theorem~\ref{thm:grouped}'s quantization small-error condition,
its full bandwidth contribution is bounded by
$C_{\mathrm{bw}}(G/K)e^{-\beta G}$ with
$C_{\mathrm{bw}}:=4C_{\mathrm{dist}}L_\ell^2/3$ and
$\beta:=2B\ln2/V$. Therefore consider the conservative leading
envelope
\[
\varphi(G):=\frac1K\left\{
\frac{c_1(1+\bar\kappa)d}{n}G
+C_{\mathrm{bw}}G e^{-\beta G}\right\},
\qquad 1\le G\le K.
\]
Put $s:=e c_1(1+\bar\kappa)d/(C_{\mathrm{bw}}n)$. If $s\ge1/e$,
$\varphi$ is nondecreasing and its minimizer is $G=1$. If $s<1/e$,
let $y_+>1$ be the larger solution of $y e^{-y}=s$ and set
$G_+:=(1+y_+)/\beta$. The global minimizer over $[1,K]$ is the
lowest-value member of $\{1,K\}\cup(\{G_+\}\cap[1,K])$.
\end{corollary}

\begin{remark}
The boundary comparison is necessary because $\varphi$ need not be
convex on the whole interval. When $s\le e^{-2}$,
$\ln(1/s)\le y_+\le2\ln(1/s)$, so an interior minimizer has order
$G_+\asymp(V/B)\log(1/s)$. For the protocol, evaluate $\varphi$ over
the divisors $G\mid K$, $G\mid V$ that also satisfy
Theorem~\ref{thm:grouped}'s regime conditions. The corollary optimizes
a valid finite-codebook envelope, not the $G$-dependent exact
distortion $q_G$. Grouping is a protocol knob orthogonal to the
per-node budgets $B_i$ of Theorem~\ref{thm:alloc}; composing grouped
aggregation with within-group water-filling is left to future work.
\end{remark}

\subsection{The between-context factor $\bar\kappa$}\label{app:kappa}

The constant $\bar\kappa$ of Theorem~\ref{thm:grouped} is finite under
Assumption~\ref{ass:A1}; when considering sequences of growing models,
its scaling is part of the regularity constants. It vanishes when the
conditional mean logit-gradients do. At $G=1$, replace
$(1+\bar\kappa)$ by $1$: the statistical quadratic form is exactly
the conditional Fisher
$I_x:=\sum_vp^\star(v\mid x)(g_v(x)-\mu_x)(g_v(x)-\mu_x)^\top$,
so no between-context factor is incurred; Theorem~\ref{thm:fpld-upper}
supplies the sharper base-case constant, and the bandwidth terms
coincide exactly. With $C_{\mathrm{dist}}$ removed, the quantization
term displayed in Theorem~\ref{thm:grouped} is the aggregate teacher's
trace-sharpened quadratic upper bound based on the exact grouped dither
variance; the dimension-free factor transfers it to the distilled
student, and the first small-error condition controls the explicit
cubic remainder.

\subsection{The finite-sample categorical counterpart}
\label{app:catcounterpart}

Let $\mathfrak M_n^{\mathrm{cat}}(K,B,V)$ denote the finite-sample
counterpart of the categorical subproblem used in
Theorem~\ref{thm:uncond}: the parameter is
$\ell^\star\in\mathcal L_{c_p}$ at one fixed context and each node
observes $n$ i.i.d.\ draws from $P_{\ell^\star}$. A node that knows
$\ell^\star$ exactly can simulate those draws with private randomness
and feed them to any finite-$n$ encoder. Hence
\[
\mathfrak M(K,B,V)\le\mathfrak M_n^{\mathrm{cat}}(K,B,V)
\]
for every $n$, and Theorem~\ref{thm:uncond}'s communication lower bound
holds verbatim for this finite-sample categorical problem. It also
embeds into a conditional model with a point-mass context marginal
($\rho=1$) whenever that model class contains the categorical logit
submodel $\mathcal L_{c_p}$. Assumption~\ref{ass:A1} alone does not
assert this additional richness. If a model class contains this
subproblem and also satisfies the usual local asymptotic minimax
conditions giving the $B=\infty$ lower bound $d/(Kn)$, its finite-$n$
minimax risk is at least a constant times the maximum of that
statistical bound and Theorem~\ref{thm:uncond}'s display.

What does \emph{not} transfer directly is the achievability. The
global-index protocol requires all nodes to compute the same codeword
index: if independently estimated logits produce different indices,
concatenating their substrings need not decode to a codeword near any
estimate. Coordinate splitting degrades more gracefully at finite $n$,
but forfeits cross-node averaging within each assigned coordinate
block; Theorem~\ref{thm:grouped} provides an upper bound for that
robust protocol throughout its joint $(n,B)$ region. A matching
\emph{joint} lower bound --- one that could detect
noise--communication interaction of the kind strong-data-processing
arguments reveal in small-$B$
regimes~\citep{zhangduchiwainwright2013} --- is open.

\subsection{Preliminaries}\label{app:prelim}

Throughout, logits are centered ($\mathbf 1^\top \ell = 0$), and for
$P = (p_v), Q = (q_v) \in \Delta(\mathcal V)$ the squared Hellinger
distance is
$h^2(P,Q) = \tfrac12 \sum_v (\sqrt{p_v} - \sqrt{q_v})^2
= 1 - \sum_v \sqrt{p_v q_v}$. Note $h$ \emph{is} a metric (Euclidean
distance between the vectors $(\sqrt{p_v})_v$ up to $1/\sqrt2$),
whereas KL is not; the testing argument below uses $h$ for
exactly this reason.

\begin{lemma}[KL dominates squared Hellinger]\label{lem:hellkl}
$\mathrm{KL}(P \| Q) \ge 2 h^2(P, Q)$ for all $P, Q$.
\end{lemma}

\begin{proof}
$\mathrm{KL}(P\|Q) = 2\sum_v p_v(-\log\sqrt{q_v/p_v})
\ge 2\sum_v p_v(1 - \sqrt{q_v/p_v}) = 2h^2(P,Q)$, by
$-\log x \ge 1 - x$.
\end{proof}

\begin{lemma}[Two-sided quadratic bounds on a box]\label{lem:quadbox}
Fix $a > 0$; let $\ell, \ell'$ be centered with
$\|\ell\|_\infty, \|\ell'\|_\infty \le a$, set
$P = \mathrm{softmax}(\ell)$, $Q = \mathrm{softmax}(\ell')$,
$\delta = \ell' - \ell$, and $q_a := e^{-2a}/V$. Then
\begin{equation}
\tfrac{q_a}{2}\|\delta\|_2^2 \le \mathrm{KL}(P\|Q)
\le \tfrac12 \|\delta\|_2^2,
\qquad
h^2(P,Q) \ge \tfrac{q_a^2}{8}\|\delta\|_2^2.
\label{eq:appquad}
\end{equation}
The upper KL bound holds with no box constraint.
\end{lemma}

\begin{proof}
For any $\xi$ with $\|\xi\|_\infty \le a$ and
$q = \mathrm{softmax}(\xi)$, we have $q_v \ge e^{-a}/(Ve^{a}) = q_a$. With
$A(\xi) = \log\sum_v e^{\xi_v}$, $\nabla^2 A(\xi) = H_{\mathrm{soft}}(q) =
\mathrm{diag}(q) - qq^\top$ and, for any $u$,
$u^\top H_{\mathrm{soft}}(q) u = \mathrm{Var}_{v\sim q}(u_v) = \min_c \sum_v
q_v(u_v - c)^2$. Hence, for \emph{centered} $u$,
$u^\top H_{\mathrm{soft}}(q)u \ge q_a \min_c \sum_v (u_v - c)^2 = q_a\|u\|_2^2$
(the minimum over $c$ is $\|u\|_2^2 - V\bar u^2 = \|u\|_2^2$), and
$u^\top H_{\mathrm{soft}}(q) u \le \sum_v q_v u_v^2 \le \|u\|_2^2$ (this side needs
no box). Taylor's theorem with integral remainder for the Bregman divergence
of $A$ gives
$\mathrm{KL}(P\|Q) = \int_0^1 (1-s)\,\delta^\top
\nabla^2 A(\ell + s\delta)\,\delta\, ds$; every segment point stays
in the box (convexity), and $\int_0^1 (1-s)\,ds = \tfrac12$ yields the
two-sided KL bound. For the Hellinger bound, coordinatewise
$(\sqrt{p_v}-\sqrt{q_v})^2 = (p_v-q_v)^2/(\sqrt{p_v}+\sqrt{q_v})^2$
and $(\sqrt{p_v}+\sqrt{q_v})^2 \le 2(p_v+q_v) \le 4$, so
$h^2 \ge \tfrac18\|P - Q\|_2^2$; and
$Q - P = \bar H \delta$ with
$\bar H := \int_0^1 H_{\mathrm{soft}}(\mathrm{softmax}(\ell + s\delta))\,ds$, a
symmetric PSD matrix annihilating $\mathbf 1$ with all eigenvalues
$\ge q_a$ on $\{\mathbf 1\}^\perp$, whence
$\|Q - P\|_2 \ge q_a \|\delta\|_2$ for centered $\delta$.
\end{proof}

\begin{lemma}[Packing and covering a Euclidean ball]\label{lem:packcover}
For $d \ge 1$, $r > 0$: (i) for $0 < \rho \le r$ there are
$M \ge (r/\rho)^d$ points in $B_d(0,r)$ with pairwise distances
$\ge \rho$; (ii) for $0 < u \le r$ there is a $u$-cover of
$B_d(0,r)$ of at most $(1 + 2r/u)^d \le (3r/u)^d$ points.
\end{lemma}

\begin{proof}
Take a maximal $\rho$-separated (respectively, $u$-separated) subset. For (i),
maximality makes the $\rho$-balls around the points cover
$B_d(0,r)$, so $M\,\mathrm{vol}\,B(0,\rho) \ge \mathrm{vol}\,B(0,r)$.
For (ii), the set is a $u$-cover by maximality, and the disjoint
$u/2$-balls around its points fit inside $B(0, r + u/2)$, giving
$N \le ((r + u/2)/(u/2))^d$.
\end{proof}

\begin{lemma}[Fano for communication-constrained randomized
protocols]\label{lem:fanocomm}
Let $J$ be uniform on $\{1,\dots,M\}$, let $U$ (public) and $W$
(decoder-private) be independent of $J$ and of each other, let the
transcript $Y$ take at most $2^N$ values conditionally on $(U,W)$,
and let $\hat J = g(Y, U, W)$. Then
$\Pr[\hat J \ne J] \ge 1 - (N + 1)/\log_2 M$.
\end{lemma}

\begin{proof}
Fano's inequality applied to the chain $J \to (Y,U,W) \to \hat J$ gives
$\Pr[\hat J \ne J] \ge 1 - (I(J; Y,U,W) + \log 2)/\log M$, and
$I(J; Y,U,W) = I(J; U,W) + I(J; Y \mid U,W)
\le 0 + H(Y \mid U,W) \le N \log 2$.
\end{proof}

\subsection{Proof of Theorem~\ref{thm:uncond}}\label{app:uncond-proof}

If $L_\ell=0$, the stated lower bound is zero and is immediate. Assume
$L_\ell>0$ below.
Set $L := L_\ell$, $a := L/2$, $q_a := e^{-L}/V$, $r := L/2$,
$\eta := 1/(KB+2)$, and
$\rho_{\mathrm{pack}} := r\,2^{-(KB+3)/(V-1)} \le r$. Let
$\varepsilon := \sup_{\ell^\star\in\mathcal L_{c_p}} \mathbb{E}[\mathrm{KL}(P^\star \|
\hat P)]$ for an arbitrary protocol in $\mathcal P_B$.

\emph{Step 1 (packing).} Identify the centered hyperplane with
$\mathbb{R}^{V-1}$ isometrically; Lemma~\ref{lem:packcover}(i) gives
centered logits $\ell_1, \dots, \ell_M$ with $\|\ell_j\|_2 \le r$,
pairwise $\|\ell_j - \ell_k\|_2 \ge \rho_{\mathrm{pack}}$, and
$\log_2 M \ge (V-1)\log_2(r/\rho_{\mathrm{pack}}) = KB + 3$. Each has
$\|\ell_j\|_\infty \le r = a$, so Lemma~\ref{lem:quadbox} applies;
and $\max_v p_{j,v} \le e^{a}/(V e^{-a}) = e^{L}/V$, so
$1 - \|P_j\|_2^2 \ge 1 - e^L/V \ge c_p$: the packing lies in the
non-degenerate class.

\emph{Step 2 (Hellinger separation).} By \eqref{eq:appquad},
$h(P_j, P_k) \ge (q_a/(2\sqrt2))\rho_{\mathrm{pack}} =: 2\tau$
for $j \ne k$, i.e.,\
$\tau^2 = q_a^2\rho_{\mathrm{pack}}^2/32$.

\emph{Step 3 (testing at level $\eta$).} Suppose, for contradiction, that
$\varepsilon < 2\eta\tau^2$. Draw $J$ uniformly from $\{1,\dots,M\}$, give all nodes
$\ell_J$, run the protocol, and decode
$\hat J \in \arg\min_k h(P_k, \hat P)$ (ties by smallest index).
Conditionally on $J = j$, Markov gives
$\Pr[\mathrm{KL}(P_j\|\hat P) > \varepsilon/\eta] \le \eta$; on the
complement, Lemma~\ref{lem:hellkl} gives
$h(P_j, \hat P) \le \sqrt{\mathrm{KL}/2} \le
\sqrt{\varepsilon/(2\eta)} < \tau$, and for $k \ne j$ the triangle
inequality (here the metric property of $h$ is essential) gives
$h(P_k, \hat P) \ge 2\tau - h(P_j,\hat P) > \tau > h(P_j, \hat P)$,
so $\hat J = j$. Hence $P_e := \Pr[\hat J \ne J] \le \eta$.
($\hat P$ is a function of the transcript, $U$, and decoder-private
randomness only; node-private randomness enters through the
transcript.)

\emph{Step 4 (Fano at level $\eta$).} Lemma~\ref{lem:fanocomm} with
$N = KB$ and $\log_2 M \ge KB+3$ gives
$P_e \ge 1 - (KB+1)/(KB+3) = 2/(KB+3)$, and
$2/(KB+3) > 1/(KB+2) = \eta$ if and only if
$2(KB+2) > KB+3$, which is equivalent to $KB+1 > 0$:
contradiction.

\emph{Step 5 (arithmetic).} Hence
\[
\varepsilon \;\ge\; 2\eta\tau^2
\;=\; \frac{2}{KB+2}\cdot \frac{q_a^2 r^2}{32}\, 2^{-2(KB+3)/(V-1)}
\;=\; \frac{L^2 e^{-2L}}{64 V^2 (KB+2)}\, 2^{-2(KB+3)/(V-1)},
\]
using $q_a^2 r^2 = (e^{-2L}/V^2)(L^2/4)$. \hfill$\square$

\subsection{Proof of Proposition~\ref{prop:vqsplit} and
Corollary~\ref{cor:twosided}}\label{app:vqsplit}

If $L_\ell=0$, the singleton codebook reconstructs the only possible
logit exactly, so assume $L_\ell>0$.
Put $d:=V-1$, $r:=\sqrt V L_\ell$, and $N:=KB$, and identify the
centered hyperplane isometrically with $\mathbb R^d$. We first treat
$N\ge d\log_2 3$. Set $u:=3r2^{-N/d}\le r$. By
Lemma~\ref{lem:packcover}(ii), the radius-$r$ Euclidean ball has a
fixed $u$-cover $\mathcal C$ with
\[
|\mathcal C|\le(3r/u)^d=2^N.
\]
Give its elements deterministic labels in $\{0,1\}^N$ (padding if
necessary). Every node knows the same $\ell^\star$, the same
codebook, and the same deterministic nearest-neighbor rule (ties are
broken by the smallest label), so all nodes compute the same label.
Node $i$ sends bits $(i-1)B+1,\ldots,iB$ of that label. The decoder
concatenates the $K$ messages, recovers the selected
$\hat\ell\in\mathcal C$, and outputs
$\hat P=\mathrm{softmax}(\hat\ell)$. Then
$\|\hat\ell-\ell^\star\|_2\le u$, and the box-free upper inequality
in Lemma~\ref{lem:quadbox} gives
\[
\mathrm{KL}(P^\star\|\hat P)
\le\frac12u^2
=\frac92 r^2 2^{-2N/d}
=\frac92 V L_\ell^2 2^{-2KB/(V-1)}.
\]

If $N<d\log_2 3$, use the singleton codebook $\mathcal C=\{0\}$
and send an arbitrary fixed transcript. Since
$2^{-2N/d}>1/9$ and $\|\ell^\star\|_2\le r$, the same display holds:
\[
\mathrm{KL}(P^\star\|\mathrm{softmax}(0))
\le\tfrac12\|\ell^\star\|_2^2
\le\tfrac12r^2
<\tfrac92r^2 2^{-2N/d}.
\]
This proves Proposition~\ref{prop:vqsplit} for every $K,B\in\mathbb N$.

For comparison, the earlier coordinate-splitting construction remains
a useful finite-sample protocol because different node estimates
populate disjoint coordinate blocks rather than incompatible pieces
of one label. When $1\le K\le V-1$, partitioning an orthonormal basis
of the centered hyperplane into blocks and covering each block
separately gives the valid but weaker deterministic bound
$\tfrac92 KVL_\ell^2 2^{-2KB/(V-1+K)}$ by the same covering and
quadratic-KL argument. This comparator motivates the grouped protocol
of Theorem~\ref{thm:grouped}; it is not needed for the unconditional
upper bound.

Corollary~\ref{cor:twosided} combines Theorem~\ref{thm:uncond} with
the global-index bound above; the deterministic guarantee also bounds
the expected risk. \hfill$\square$

\subsection{Proof of Proposition~\ref{prop:symmetry}}

Proposition~\ref{prop:rebuttal-lb} gives the scalar-class lower bound
$c_H(4K)^{-1}2^{-2b}$. The subtractively dithered quantizer belongs to
that class. Its direct aggregate-teacher quadratic term is
$C_q(2K)^{-1}2^{-2b}$, and under \textbf{(SE$_{c_0}$)} the cubic remainder is
no larger than that term; hence
$R_{\mathrm{sc}}^\star(P^\star)\le C_qK^{-1}2^{-2b}$.

Proposition~\ref{prop:vqsplit}'s codebook is universal over
$\mathcal L$. Run the same deterministic encoding and decoding rule at
each context $x$ with input $\ell^\star(x)$. Its pointwise KL bound
therefore survives averaging over $P_X^\star$, giving risk at most
$\frac92VL_\ell^2 2^{-2KB/(V-1)}$ for the same predictor. The ratio of
the scalar lower bound to this unrestricted upper bound is
\[
\frac{c_H}{18VL_\ell^2K}
2^{\,2KB/(V-1)-2b},
\]
which diverges exponentially in $K$ at fixed $b$ and $V$. Finally,
Corollary~\ref{cor:twosided} supplies the separate minimax statement for
the noiseless one-context problem under its hypotheses.
\hfill$\square$

\subsection{Proof of Theorem~\ref{thm:grouped}}\label{app:grouped}

Let $N_1,\ldots,N_G$ partition the nodes into sets of size $K/G$ and
let $C_1,\ldots,C_G$ partition the coordinates into sets of size
$V/G$. For $v\in C_g$, define the unquantized grouped aggregate
\[
a_v(x):=\frac GK\sum_{i\in N_g}
\ell_{\hat\theta_i^{\mathrm{MLE}}}(x)_v,
\qquad
\xi_v(x):=\frac GK\sum_{i\in N_g}\xi_{i,v}(x),
\]
where $\xi_{i,v}$ is the subtractive-dither quantization error. The transmitted
aggregate is $a+\xi$. We first analyze exact local MLEs; the
optimization slack in Assumption~\ref{ass:A4} is added at the end.

\emph{Quantization conditional on the local estimates.}
Assumption~\ref{ass:A5} at rate $GB/V$ gives
$\mathbb E[\xi\mid a]=0$ and
$\mathrm{Cov}(\xi\mid a)=\sigma_G^2I_V$, with
$\sigma_G^2=(G/K)q_G$: each coordinate averages $K/G$ independent
errors of variance $q_G$. Expand the cross-entropy
$A(a+\xi)-p^{\star\top}(a+\xi)$ around $a$, where
$A(z)=\log\sum_ve^{z_v}$. The conditional linear term vanishes, and
Lemma~\ref{lem:soft-trace} gives
\[
\frac12\mathbb E[\xi^\top H_{\mathrm{soft}}
(\mathrm{softmax}(a))\xi\mid a]
=\frac12\sigma_G^2\operatorname{tr}
H_{\mathrm{soft}}(\mathrm{softmax}(a))
\le\frac12\sigma_G^2.
\]
The same bounded-dither cumulant argument as in
Appendix~\ref{ssec:trace-upper}, now at variance $\sigma_G^2$, bounds
the conditional remainder by
$C_0\sigma_G^3(\log V)^{3/2}$ for a universal $C_0$. Thus the aggregate
teacher has a quadratic quantization increment of at most $\sigma_G^2/2$;
the first small-error condition makes its cubic remainder no larger
than a fixed multiple of this term.

\emph{A global quadratic comparison for the statistical part.}
Both $\ell^\star$ and $a$ lie in $[-L_\ell,L_\ell]^V$. Along their
line segment, if $p_s:=\mathrm{softmax}((1-s)\ell^\star+sa)$, then
$p_s(v)/p^\star(v)\le e^{4L_\ell}$. Hence, for every vector $h$,
\[
h^\top H_{\mathrm{soft}}(p_s)h
=\operatorname{Var}_{p_s}(h_v)
\le e^{4L_\ell}\operatorname{Var}_{p^\star}(h_v)
=e^{4L_\ell}h^\top H_xh,
\qquad H_x:=H_{\mathrm{soft}}(p^\star(\cdot\mid x)).
\]
Indeed, choose the centering constant $\mathbb E_{p^\star}h_v$ in the
variance on the left, and then use the likelihood-ratio bound.
The integral Taylor formula therefore yields the deterministic bound
\begin{equation}\label{eq:group-global}
\mathrm{KL}(p^\star\|\mathrm{softmax}(a))
\le\frac{e^{4L_\ell}}2
(a-\ell^\star)^\top H_x(a-\ell^\star).
\end{equation}
This global comparison requires no sub-Gaussian coordinate maximum.

Write $u_i=\hat\theta_i^{\mathrm{MLE}}-\theta^\star$,
$z_i=u_i-\mathbb Eu_i$, and
$g_v(x)=\nabla_\theta\ell_{\theta^\star}(x)_v$. Taylor's theorem and
Assumption~\ref{ass:A1} give
\[
\ell_{\hat\theta_i^{\mathrm{MLE}}}(x)_v
=\ell^\star(x)_v+g_v(x)^\top u_i+r_{i,v}(x),
\qquad |r_{i,v}(x)|\le\frac{M_\ell}{2}\|u_i\|_2^2,
\]
for a finite $M_\ell$. For $v\in C_g$, the centered linear part is
$g_v^\top\bar z_g$, where
$\bar z_g=(G/K)\sum_{i\in N_g}z_i$. Define
\[
A_g(x):=\sum_{v,v'\in C_g}(H_x)_{vv'}g_v(x)g_{v'}(x)^\top.
\]
The groups are independent and centered, and Assumption~\ref{ass:A7}
gives
\begin{equation}\label{eq:group-stat}
\mathbb E[(\delta^{\mathrm s})^\top H_x\delta^{\mathrm s}]
=\sum_g\mathbb E[\bar z_g^\top A_g(x)\bar z_g]
\le\frac{Gc_\Sigma}{Kn}\sum_g
\operatorname{tr}(A_g(x)I(\theta^\star)^{-1}).
\end{equation}
Because each principal block of $H_x$ is bounded by its corresponding
block of $\operatorname{diag}(p^\star)$,
$A_g(x)\preceq\sum_{v\in C_g}p^\star_vg_vg_v^\top$. Put
$\mu_x:=\sum_vp^\star_vg_v$. Throughout this calculation,
$\mathbb E_x[f(x)]$ abbreviates
$\mathbb E_{X\sim P_X^\star}[f(X)]$. Define
\[
I_x:=\sum_vp^\star_v(g_v-\mu_x)(g_v-\mu_x)^\top,
\qquad I(\theta^\star)=\mathbb E_xI_x.
\]
Then the score decomposition gives
\begin{equation}\label{eq:group-kappa}
\mathbb E_x\sum_g\operatorname{tr}(A_gI^{-1})
\le\operatorname{tr}\!\left[I^{-1}
\{I+\mathbb E_x(\mu_x\mu_x^\top)\}\right]
\le d(1+\bar\kappa).
\end{equation}
Equivalently,
$\bar\kappa=\sup_{h\ne0}\mathbb E_x(h^\top\mu_x)^2/
\mathbb E_x\operatorname{Var}_{p^\star}(h^\top g_Y)$, so it is the
maximal between-context-to-within-context score ratio. At $G=1$,
$A_1(x)=I_x$ exactly, and the right side of
\eqref{eq:group-kappa} is $d$ rather than $d(1+\bar\kappa)$.

It remains to justify replacing the exact logit difference in
\eqref{eq:group-global} by its centered linear part. The common MLE
bias contributes at most $C_b/n^2$ to the expected $H_x$-quadratic
form because $\|\mathbb Eu_i\|\le c_b/n$ and the logit gradients are
uniformly bounded. For the Taylor remainder, put
$R_g=(M_\ell/2)(G/K)\sum_{i\in N_g}\|u_i\|^2$. Then
$|\bar r_{g,v}|\le R_g$, and Jensen's inequality together with
Assumption~\ref{ass:A7} gives
$\mathbb ER_g^2\le(M_\ell^2c_4/4)(d/n)^2$. Since
$h^\top H_xh\le\sum_vp^\star_vh_v^2$,
\[
\mathbb E[\bar r^\top H_x\bar r]
\le\sum_g p^\star(C_g)\mathbb ER_g^2
\le C_r(d/n)^2.
\]
All cross terms are absorbed by the corresponding squares using the
PSD Cauchy--Schwarz inequality. Finally,
\[
\frac{(d/n)^2}{dG/(Kn)}=\frac{dK}{Gn},
\qquad
\frac{n^{-2}}{dG/(Kn)}=\frac{K}{dGn}
\le\frac1{d^2}\frac{dK}{Gn},
\]
so Theorem~\ref{thm:grouped}'s finite-sample condition absorbs both
remainders into $c_1(1+\bar\kappa)dG/(Kn)$. Combining
\eqref{eq:group-global}--\eqref{eq:group-kappa} with the conditional
quantization bound proves the aggregate-teacher risk under
$P_X^\star$. The same bounded-derivative calculation under $Q$ is
$O(dG/(Kn))$; its constant is absorbed into $c_1$. Finally, apply
equation~\eqref{eq:student-transfer}, the probe oracle inequality, and
Lemma~\ref{lem:kl-transfer}'s reverse-KL comparison. This multiplies
the quadratic term by $C_{\mathrm{dist}}$, absorbs the corresponding
cubic factor into $C_R$, and contributes
$c_2\rho\sqrt{V\log V/m}+c_{\mathrm{fit}}\rho
\varepsilon_{\mathrm{fit}}$. The optimization slack is handled as in
Theorem~\ref{thm:fpld-upper}, completing the proof.
\hfill$\square$

\subsection{Proof of Corollary~\ref{cor:groupopt}}\label{app:groupopt}

The quantization identity gives
$q_G\le(4L_\ell^2/3)e^{-\beta G}$, and the small-error condition
allows the cubic term to be bounded by the quadratic term. Thus the
full bandwidth contribution is at most
$C_{\mathrm{bw}}(G/K)e^{-\beta G}$ with
$C_{\mathrm{bw}}=4C_{\mathrm{dist}}L_\ell^2/3$, as stated. Apart from the common
factor $1/K$, minimize
$\varphi(G)=\alpha G+\gamma G e^{-\beta G}$ on $[1,K]$, where
$\alpha=c_1(1+\bar\kappa)d/n$ and
$\gamma=C_{\mathrm{bw}}$. Then
$\varphi'(G) = \alpha + \gamma e^{-\beta G}(1 - \beta G)$;
substituting $y := \beta G - 1$, stationarity reads
$y e^{-y} = \alpha e/\gamma =: s
=e c_1(1+\bar\kappa)d/(C_{\mathrm{bw}}n)$. Since
$y\mapsto ye^{-y}$ peaks at $1/e$ (at $y=1$), if $s\ge1/e$ then
$\varphi'(G)\ge0$ throughout and $G=1$ is optimal. If $s<1/e$, the
larger root $y_+>1$ gives the only local minimum
$G_+=(1+y_+)/\beta$; the smaller root is a local maximum. Hence a
global minimizer on the compact interval must be the lowest-value
member of $\{1,K\}\cup(\{G_+\}\cap[1,K])$. Elementary bounds for
$s\le e^{-2}$ are as follows:
at $y = \ln(1/s)$, $ye^{-y} = s\ln(1/s) \ge s$, so
$y_+ \ge \ln(1/s)$ (as $ye^{-y}$ decreases past $1$); at
$y = 2\ln(1/s)$, $ye^{-y} = 2s^2\ln(1/s) \le s$ if and only if
$2s\ln(1/s) \le 1$, true for $s \le e^{-2}$ since
$x \mapsto 2x\ln(1/x)$ increases on $(0, e^{-1})$ with value
$4e^{-2} < 1$ at $x = e^{-2}$; hence $y_+ \le 2\ln(1/s)$.
Restricting the final comparison to valid divisors and to the
theorem's regime gives the protocol statement.
\hfill$\square$

\section{Proofs for Section~\ref{ssec:streaming} (Dynamic
Allocation)}\label{app:streaming}

\subsection{Setup for the one-round-lag protocol}\label{app:lagsetup}

There are two timing options when codec caps change across rounds. A short
control phase can carry the current indices before scheduling; the
aggregator then knows the current weights and computes the exact current
envelope oracle. In a lower-latency single-message round, the payload
allocation must be committed before the current header arrives, so round
$t$ is scheduled from the indices decoded in round $t-1$.

At round $t$, node $i$ rounds its current required cap upward on the
same public grid as Assumption~\ref{ass:A5p}, uses the resulting shared
cap $L_{\ell,i}(t)$ for its codec, and places its $s$-bit index
before the payload. Define the observed allocation weight
\[
w_i(t):=L_{\ell,i}(t)^2>0.
\]
No common clip rule is imposed: nodes may use different procedures to
choose their required caps, because the objective depends on the
reported codec cap itself. Let $\mathcal B$ be either the nonempty
continuous set
\[
\{B\in[V,B_{\max}]^K:\textstyle\sum_iB_i=B_{\mathrm{tot}}\}
\]
or, for the finite-codebook protocol, its nonempty integer-rate subset
with $B_i/V\in\mathbb N$. The per-round allocation objective and oracle are
\[
F_t(B_1, \dots, B_K)
= \frac{1}{K^2}\sum_{i=1}^K w_i(t)\, 2^{-2B_i/V},
\qquad
B^\star(t)\in\arg\min_{B\in\mathcal B}F_t(B),
\qquad
F_t^\star = F_t\bigl(B^\star(t)\bigr),
\]
with the continuous optimizer given by the active-set water-filling of
Theorem~\ref{thm:alloc} and the integer optimizer given by its marginal-gain
rule. No interiority assumption is imposed. The metadata accounting is
exactly that of Assumption~\ref{ass:A5p}: a refreshed header adds
$s/m$ bits per context per node.

\subsection{The one-round-lag protocol in full}\label{app:lagprotocol}

Choose any feasible initialization $B(1)$, for example uniform. In
round $t$, node $i$ first places $r_i(t)$ in its message header and
then encodes its $m$ payload vectors using the cap
$L_{\ell,i}(t)=\lambda_{r_i(t)}$ and the already scheduled budget
$B_i(t)$. The aggregator reads the header, decodes the payload with
the same cap, and stores $w_i(t)=L_{\ell,i}(t)^2$. For $t\ge2$, the
single-message scheduler uses the most recently decoded weights:
\[
B(t)\in\arg\min_{B\in\mathcal B}F_{\hat w(t-1)}(B).
\]
Here $\hat w(t-1)$ permits deterministic numerical approximation in
the scheduler; the exact grid-index implementation has
$\hat w(t-1)=w(t-1)$. The selected $B(t)$ is sent on the uncharged
downlink before the round-$t$ payload. If an application allows a
header-only control phase before that scheduling message, replace
$\hat w(t-1)$ by $w(t)$ and obtain the current envelope oracle exactly.

\subsection{Proof of Theorem~\ref{thm:streaming}}

The triangle inequality in log scale gives, for every node,
\[
\left|\log_2\frac{\hat w_i(t-1)}{w_i(t)}\right|
\le
\left|\log_2\frac{\hat w_i(t-1)}{w_i(t-1)}\right|
+\left|\log_2\frac{w_i(t-1)}{w_i(t)}\right|
\le\varepsilon_{\mathrm w}+\delta_t=:r_t.
\]
Hence, for every feasible allocation $B$,
\[
2^{-r_t}F_t(B)\le F_{\hat w(t-1)}(B)
\le2^{r_t}F_t(B).
\]
Lemma~\ref{lem:transfer}, with $a=2^{-r_t}$ and $b=2^{r_t}$,
gives
$F_t(B(t))\le2^{2r_t}F_t^\star$. Exact grid-index metadata has
$\varepsilon_{\mathrm w}=0$. The Taylor expansion
$2^{2\delta_t}=\exp(2\ln2\,\delta_t)
=1+2\ln2\,\delta_t+O(\delta_t^2)$ gives the displayed small-drift
form. The proof used only a common feasible set and nonnegative terms,
so it covers active-set and integer allocations. Replacing
$2^{-2B_i/V}$ by $(2^{B_i/V}-1)^{-2}/6$ applies the same argument to
$G_{w(t)}$. If current metadata is available before scheduling, use
$\hat w=w(t)$, so $r_t=0$. \hfill$\square$

\subsection{Proof of Proposition~\ref{prop:staticfail}}\label{app:staticproof}

For any allocation weights $u \in (0,\infty)^K$, interior
water-filling gives $B_i(u) = B_{\mathrm{tot}}/K +
\tfrac V2 \log_2(u_i/u_{\mathrm{gm}})$ with
$u_{\mathrm{gm}} = (\prod_j u_j)^{1/K}$, so
$2^{-2B_i(u)/V} = 2^{-2B_{\mathrm{tot}}/(KV)}\,
(u_i/u_{\mathrm{gm}})^{-1}$.
Hence, with current weights $w = w(t)$ and their means
$w_{\mathrm{am}},w_{\mathrm{gm}}$ as in
Proposition~\ref{prop:staticfail}:
\[
F_t\bigl(B(u)\bigr) = \frac{1}{K^2}\, 2^{-\frac{2B_{\mathrm{tot}}}{KV}}
\sum_i w_i\, \frac{u_{\mathrm{gm}}}{u_i},
\qquad
F_t^\star = F_t\bigl(B(w)\bigr)
= \frac1K\, 2^{-\frac{2B_{\mathrm{tot}}}{KV}}\, w_{\mathrm{gm}}.
\]
With $u = w(t_0)$ and $\Delta_i = \log_2 w_i(t) - \log_2 u_i$:
$w_i/u_i = 2^{\Delta_i}$ and
$u_{\mathrm{gm}}/w_{\mathrm{gm}} = 2^{-\bar\Delta}$, so
$F_t(\bar B)/F_t^\star = \tfrac1K \sum_i 2^{\Delta_i - \bar\Delta}
\ge 1$ by Jensen (the exponents average to zero), with equality if and
only if all $\Delta_i$ are equal. This fixes the sign convention: a
stale allocation assigns too few bits to nodes whose weight has grown.
For the uniform benchmark, $F_t(\mathrm{unif}) = \tfrac{1}{K^2}\,
2^{-2B_{\mathrm{tot}}/(KV)}\sum_i w_i$, so
$F_t(\mathrm{unif})/F_t^\star = w_{\mathrm{am}}(t)/w_{\mathrm{gm}}(t)$. In a two-group
example with $K$ even, half the $\Delta_i = +\sigma/2$, half
$-\sigma/2$, and equal current weights, the static ratio is
$\tfrac12(2^{\sigma/2} + 2^{-\sigma/2}) = \cosh(\sigma\ln2/2) > 1$,
the uniform ratio is $1$, and the gap grows exponentially in
$\sigma$. Comparing the two exact ratios proves the if-and-only-if
criterion in the proposition. Two three-node examples make the final
claim explicit; in each case the total budget can be chosen large enough
and the caps wide enough that both allocations are interior. Write
$a_i=\log_2 u_i$ and $b_i=\log_2 w_i$. First, take
$a=(-3,-2,-1)$ and $b=(-3,0,-1)$. The ranking changes, while
$\Delta=(0,2,0)$ and
\[
 \frac13\sum_i2^{\Delta_i-\bar\Delta}
 =\frac{2\,2^{-2/3}+2^{4/3}}{3}\approx1.2599
 <\frac{13}{24}2^{4/3}\approx1.3649
 =\frac{w_{\mathrm{am}}}{w_{\mathrm{gm}}};
\]
thus a crossover is not sufficient. Conversely, take
$a=(-2,-1,2)$ and $b=(-2,-1,0)$. The ranking is unchanged, but
$\Delta=(0,0,-2)$ and
\[
 \frac13\sum_i2^{\Delta_i-\bar\Delta}
 =\frac{2\,2^{2/3}+2^{-4/3}}{3}\approx1.1906
 >\frac76\approx1.1667
 =\frac{w_{\mathrm{am}}}{w_{\mathrm{gm}}};
\]
thus a crossover is not necessary. Hence crossover alone is neither
necessary nor sufficient, even when $K=3$. \hfill$\square$

\subsection{Exact metadata-streaming implementation}\label{app:streamexp}

The frozen streaming experiment directly implements the codec timing in
Section~\ref{ssec:streaming}. At each round it computes the required cap,
rounds it upward on the public grid, and treats the squared reported cap
as the scheduling weight. The current policy minimizes the exact finite
objective with the current indices; the lag policy uses the preceding
indices; and the frozen policy retains the first-round indices. All
integer rates are at least one and sum to the same payload budget. There
is no EWMA, moment estimator, warm-up, nominal lattice, or zero-rate node
in this experiment.

For every round the result JSON records the current and lagged rate
vectors, the finite objective relative to the current oracle, the
realized one-step log-weight drift, and
Theorem~\ref{thm:streaming}'s pathwise bound. It separately records the
required caps, reported indices, $s$ bits per node per codec epoch, $Ks$
total header bits, contexts per epoch, and both per-node and total
amortized header costs. These fields provide the provenance for the
smooth-drift diagnostic reported in Section~\ref{sec:exp}.

\begin{revblock}
\section{Student Training, Payloads below One Bit, and Heterogeneous Sites}
\label{app:practical}

This appendix treats three practical questions that the theory leaves
implicit: how much of an aggregate-teacher gain survives the training of
a student (Appendix~\ref{app:practical-student}), how the allocation rule
extends below one bit per coordinate when $V$ is large
(Appendix~\ref{app:practical-subbit}), and what changes when sites are
not identically distributed (Appendix~\ref{app:practical-noniid});
Appendix~\ref{app:practical-lit} adds further statistical connections. The
two numerical studies use exact finite $2^b$-point dithered codebooks as in
Assumption~\ref{ass:A5p}; their code and frozen outputs are in the
reproducibility package (\texttt{scripts/run\_student\_transfer.py} and
\texttt{scripts/run\_subbit\_allocation.py}).

\subsection{From aggregate-teacher gains to student gains}
\label{app:practical-student}

\paragraph{What the bound says.} In Theorems~\ref{thm:fpld-upper}
and~\ref{thm:het-upper} the payload allocation enters only the
bandwidth term; the statistical term $d/(Kn)$, the probe term
$\rho\sqrt{V\log V/m}$, and the slacks $\varepsilon_{\mathrm{opt}}$ and
$\varepsilon_{\mathrm{fit}}$ do not depend on the $B_i$. A relative
reduction $r$ of the bandwidth term therefore reduces the student's risk
bound by $r$ times the share of that bound contributed by the bandwidth
term. The factor $C_{\mathrm{dist}}\ge1$ is an upper-bound constant, so
the bound does not say how large that share is for a trained student. We
measured it.

\paragraph{Design.} The exact-channel federation of
Appendix~\ref{app:synthetic-extra} is used: $K=4$ nodes with caps
$(1,1,4,4)$, $V=256$, and, at each probe context, a common teacher logit
bounded by $0.95$ that every node holds exactly, so the local
statistical term is absent and attention is on the quantization noise.
Contexts carry Gaussian features $\phi(x)\in\mathbb R^{p}$, and the
teacher's canonical logit is linear in $\phi(x)$. The aggregator averages
the $K$ dithered logits on $m$ probe contexts and fits a softmax-linear
student $\mathrm{softmax}(W^\top\phi(x)+c)$, with $(p+1)V$ parameters,
by minimizing the empirical probe cross-entropy to the averaged teacher;
risk is $\mathrm{KL}(P^\star\|\hat P)$ averaged over $4{,}000$ held-out
contexts. Four regimes separate the non-bandwidth terms: an exact fit
($p=16$, L-BFGS to convergence); an early-stopped fit ($p=16$, five
gradient steps from zero), which creates distillation-fit slack
$\varepsilon_{\mathrm{fit}}>0$; a misspecified student ($p=16$, with
teacher logits also depending on $16$ features the student does not
see), which adds an approximation error that Theorem~\ref{thm:fpld-upper}
excludes by assuming $P^\star\in\mathcal F_\Theta$; and a high-capacity
student ($p=128$). The probe size is $m\in\{256,1{,}024,4{,}096\}$ and
the payload is $B_{\mathrm{tot}}/V\in\{8,16\}$; uniform rates are
$(2,2,2,2)$ and $(4,4,4,4)$, finite-exact rates $(1,1,3,3)$ and
$(3,3,5,5)$. Both policies use the same dither stream ($20$ seeds), and
every configuration also fits the student to the unquantized teacher,
which gives the student's own error floor.

\paragraph{Results.} Table~\ref{tab:student-transfer} reports, for each
configuration, the teacher-level KL reduction on the probes, the
student-level KL reduction, the share of the uniform-allocation student
KL that is attributable to quantization (student KL minus the
unquantized floor, divided by student KL), and the reduction of that
quantization-attributable excess. Three findings hold whenever the
student is fitted to convergence.

\begin{enumerate}
\item \emph{The relative gain on the quantization error carries over in
full when the fit converges.} With an exact fit, a misspecified student,
or a high-capacity student, water-filling removes the same fraction of
the student's quantization-attributable KL as of the teacher's probe KL:
$29.4$--$30.2\%$ against $29.6$--$29.8\%$ at $B_{\mathrm{tot}}/V=8$, and
$50.4$--$51.6\%$ against $50.8$--$51.0\%$ at $16$, at every probe size.
\item \emph{The student inherits only a fraction of the teacher-level
quantization KL, set by its size relative to the probe set.} The ratio
of the student's quantization-attributable KL to the teacher's probe KL
is $0.075$, $0.018$, and $0.0044$ for $p=16$ at $m=256$, $1{,}024$, and
$4{,}096$, and $1.06$, $0.15$, and $0.034$ for $p=128$. These values
are within $6\%$ of the least-squares prediction $1/m+p/(m-p-2)$
derived below.
\item \emph{The overall gain is the teacher-level gain multiplied by the
quantization share.} When the student's own error is zero (exact fit,
high capacity), the student-level reduction equals the teacher-level
reduction. A misspecified student has an approximation floor of
$2.4$--$2.6\times10^{-3}$. Its quantization share at
$B_{\mathrm{tot}}/V=8$ falls from $0.53$ to $0.07$ as $m$ grows, and its
overall reduction falls from $15.9\%$ to $2.0\%$ with it.
\end{enumerate}
Optimization slack behaves differently. With five gradient steps, the
fit leaves an error floor of $0.7$--$1.3\times10^{-3}$, and its noise ratio
barely depends on $m$ ($0.044$, $0.032$, and $0.030$ at
$B_{\mathrm{tot}}/V=8$, against $0.075$, $0.018$, and $0.0044$ for the
exact fit). It also
removes only $11.7$--$15.1\%$ of the quantization-attributable excess,
about half the teacher-level $29.6\%$, so the overall reduction is
$8.0$--$9.7\%$. A plausible mechanism is the following. Away from the
optimum the KL has a nonzero gradient. Quantization also biases the
softmax target slightly, since $\mathbb E\,\mathrm{softmax}(\ell+e)\ne
\mathrm{softmax}(\ell)$. Together these give a cross term that is not
proportional to the allocation objective. At $B_{\mathrm{tot}}/V=16$ the
early-stopped excess is below $10\%$ of the student KL and its reduction is
not estimable from $20$ seeds.

\begin{table}[!htbp]
\revfloat\centering\footnotesize
\setlength{\tabcolsep}{4pt}
\renewcommand{\arraystretch}{1.08}
\caption{Water-filling removes the same fraction of a student's
quantization-attributable KL as of the aggregate teacher's KL when the
student is fitted to convergence, and the overall student gain is that
fraction times the quantization share. Exact finite-codebook channel,
$K=4$, caps $(1,1,4,4)$, $V=256$, softmax-linear students with $p+1$
parameters per coordinate ($p=16$ unless stated), and $20$ paired dither
seeds. ``Teacher'': percent reduction of probe KL from uniform to
finite-exact allocation; ``student'': percent reduction of held-out
student KL; ``excess'': percent reduction of student KL minus the
unquantized-teacher floor; ``share'': fraction of the uniform-allocation
student KL above that floor. Dashes mark excess reductions that are not estimable
because the excess is below $10\%$ of the student KL and its $95\%$
interval exceeds $\pm35$ points.}
\label{tab:student-transfer}
\begin{tabular}{@{}lrrrrrrrrr@{}}
\toprule
& & \multicolumn{4}{c}{$B_{\mathrm{tot}}/V=8$} &
\multicolumn{4}{c}{$B_{\mathrm{tot}}/V=16$}\\
\cmidrule(lr){3-6}\cmidrule(l){7-10}
student & $m$ & teacher & student & excess & share & teacher & student & excess & share\\
\midrule
exact fit & $256$ & 29.8 & 29.5 & 29.5 & 1.00 & 50.8 & 50.4 & 50.4 & 1.00\\
 & $1{,}024$ & 29.6 & 29.4 & 29.4 & 1.00 & 50.9 & 50.7 & 50.7 & 1.00\\
 & $4{,}096$ & 29.6 & 30.1 & 30.1 & 1.00 & 51.0 & 51.3 & 51.3 & 1.00\\
\addlinespace
early-stopped fit & $256$ & 29.6 & 8.7 & 14.3 & 0.57 & 51.0 & 5.7 & --- & 0.10\\
 & $1{,}024$ & 29.6 & 8.0 & 11.7 & 0.58 & 51.0 & 0.4 & --- & 0.08\\
 & $4{,}096$ & 29.7 & 9.7 & 15.1 & 0.62 & 51.0 & 1.5 & --- & 0.07\\
\addlinespace
misspecified & $256$ & 29.7 & 15.9 & 30.1 & 0.53 & 51.0 & 2.2 & 50.6 & 0.04\\
 & $1{,}024$ & 29.7 & 6.4 & 29.5 & 0.22 & 51.0 & 0.6 & 51.0 & 0.01\\
 & $4{,}096$ & 29.6 & 2.0 & 29.9 & 0.07 & 51.0 & 0.1 & 51.6 & 0.00\\
\addlinespace
high capacity ($p=128$) & $256$ & 29.7 & 30.2 & 30.2 & 1.00 & 51.0 & 51.2 & 51.2 & 1.00\\
 & $1{,}024$ & 29.8 & 30.1 & 30.1 & 1.00 & 51.0 & 50.9 & 50.9 & 1.00\\
 & $4{,}096$ & 29.6 & 29.5 & 29.5 & 1.00 & 51.0 & 51.1 & 51.1 & 1.00\\
\bottomrule
\end{tabular}
\end{table}

\paragraph{Why the share scales as it does.} To first order the student
is a least-squares fit of the averaged logits on $p+1$ regressors, and
the dither errors are mean zero, independent of the logits, and
independent across nodes and coordinates. For Gaussian features the
out-of-sample variance of such a fit is approximately $1/m+p/(m-p-2)$
times the noise variance, and the KL weights both levels by the same softmax Hessian, so
the student inherits approximately that fraction of the teacher-level
quantization KL. The weighting is identical across allocation policies,
which is why the \emph{relative} reduction of the excess equals the
teacher-level reduction. This is the classical errors-in-the-response
situation \citep{carroll2006}: unbiased, independent noise in the
target inflates variance but introduces no bias in a linear-in-logit
fit, and the fit averages it over the $m$ probes.

\paragraph{Practical reading.} The teacher-level gains of
Section~\ref{sec:reallm} transfer to a student in proportion to the
share of the student's risk that quantization causes. That share is
large when budgets are tight or when the student is large relative to
the probe set, which is the regime of fine-tuning a language model on a
few thousand probe contexts, and small when the student's own fit,
optimization, or approximation error dominates. A practitioner can
estimate the share directly by also distilling once from the
unquantized local logits on a small pilot probe set, which is how the
floor is obtained here.

\subsection{Payloads below one bit per coordinate}
\label{app:practical-subbit}

\paragraph{The cost of one bit per coordinate.} The finite codebook of
Assumption~\ref{ass:A5} has $b\ge1$, so a node spends at least $V$ bits
per probe context; for $V=131{,}072$ that is $16$\,kB per context. Two
departures from the paper's channel reduce this: allocating bits across
coordinates, which includes top-$k$ sparsification as a special case,
and entropy coding.

\paragraph{Allocating across coordinates.} For independent mean-zero
logit errors $e_v$ with variances $s_v^2$, the second-order expansion
used in Appendix~\ref{ssec:soft-trace} gives
\[
\mathbb E\,\mathrm{KL}\bigl(\mathrm{softmax}(\ell)\,\|\,
\mathrm{softmax}(\ell+e)\bigr)
=\tfrac12\textstyle\sum_v p_v(1-p_v)\,s_v^2+\text{higher order},
\]
because only the diagonal $p_v(1-p_v)$ of the softmax Hessian
$\operatorname{diag}(p)-pp^\top$ survives the expectation. Each
coordinate is therefore a separate ``node'' with weight
$w_v=p_v(1-p_v)L^2$, with $L$ the residual cap, and
Theorem~\ref{thm:alloc}, applied to the exponential envelope with
each coordinate in the role of a node of dimension one, gives
depths $b_v=\max\{0,\tfrac12\log_2(w_v/\nu)\}$, rounded to integers by
the greedy rule of Appendix~\ref{app:integer}. With the lower bound
$b_v\ge0$ active, every coordinate whose weight falls below the level
$\nu$ receives no bits. This is top-$k$ sparsification, with $k$ set by
the water level and graded depths among the retained coordinates rather
than a fixed $k$ and a common depth. Node and coordinate allocation
combine into a single problem with weights $w_{iv}=L_i^2p_v(1-p_v)$, with
$L_i$ the residual cap of node $i$.

Three implementation points matter. First, the weights must be public
so that the decoder knows which coordinates were sent and at what depth.
From the second round on, the previously broadcast student supplies them: downlink is uncharged,
both sides hold its probe probabilities, and a node transmits the
dithered \emph{residual} of its logit relative to the student's logit,
so an unsent coordinate reverts to the student's value rather than to
zero. No index overhead is incurred, and because the objective is linear
in the weights, the argument of Corollary~\ref{cor:adaptive} bounds the
cost of using the student's probabilities in place of the target's by
$2^{2\varepsilon}$, with $\varepsilon$ the largest absolute $\log_2$-ratio of the two
weight vectors. Second, an unsent coordinate keeps its whole residual at
every node, so its error is a bias that does not average down by $1/K$.
The allocation must price this: for a residual bounded by $L$ we use the
per-coordinate costs $c(0)=L^2/3$, the second moment of a residual spread
uniformly over $[-L,L]$, and $c(b)=L^2/\{3K(2^b-1)^2\}$ for $b\ge1$, the
exact dither variance after averaging $K$ nodes. These costs are
decreasing and convex in $b$ for $K\ge2$, so the greedy rule is exact
for the resulting objective $\sum_vp_v(1-p_v)c(b_v)$. Third, an unbiased alternative to
deterministic top-$k$ selection keeps coordinate $v$ with probability
$\pi_v$ and rescales by $1/\pi_v$; it preserves the $1/K$ averaging of
the theory but inflates the variance by $(1-\pi_v)/\pi_v$, so it pays
off only when $K$ is large.

\paragraph{Entropy coding.} Entropy coding of the dithered cell indices,
conditional on the public dither, is lossless. It changes no
reconstruction, so every distortion statement in the paper is unchanged;
it lowers only the expected number of bits, and the entropy-coded
dithered quantizer is within $0.754$ bit per coordinate of the
squared-error rate--distortion function for every source \citep{zamir1992}. The gain
is largest when a node's logits concentrate well inside its cap, which
is typical for the low-probability tail of a language model. The price
is that the budget becomes an expected rather than a hard per-context
constraint; over the $m$ probe contexts of a round the total length
concentrates, so a hard per-round budget can be met with a small
margin. The allocation rule keeps its form at high resolution, where
the distortion of an entropy-coded uniform quantizer is
$\propto2^{2h_v}2^{-2R_v}$ with $h_v$ the differential entropy of the
coordinate \citep{gershogray1992}, so the weights become entropy powers
in place of squared caps.

\paragraph{A numerical check at $V=131{,}072$.} The study uses $V=131{,}072$, the Falcon3 output
dimension, and $K=4$ nodes. It draws $30$ contexts, each with a public
student logit of Zipf form ($\ell_{\mathrm{ref},v}=-1.7\log r_v$ for a
random ranking $r$) and a target that differs from it by a residual
$\delta_v\sim N(0,0.4^2)$ clipped to $[-1,1]$, so $L=1$. The target
softmax has mean trace $0.73$ and mean entropy $2.37$ nats, close to
the real models' fixed-probe traces $0.62$--$0.79$
(Table~\ref{tab:reallm-rate}). Every policy is evaluated with $10$
dither seeds per context. Table~\ref{tab:subbit} compares the dense
one-bit codebook, top-$k$ selection by the student's probabilities with a
common depth of two bits, and the coordinate water-filling rule, all at
equal payload per node. The water-filling rule spends $1/128$ bit per
coordinate, $1{,}024$ bits per node per context instead of $131{,}072$,
and reaches KL $1.03\times10^{-3}$, $30$ times lower than the dense
one-bit codebook ($3.09\times10^{-2}$) and $3.8$ times lower than top-$k$
at the same payload. Top-$k$ with a common depth stops improving at
about $3.2\times10^{-3}$, because extra coordinates cannot reduce the
error on the few high-probability coordinates that dominate the KL.
Graded depths can, and this is where the water-filling depths matter. The
deepest coordinate receives $7$ to $12$ bits (the cap in the code is
$12$). These are oracle-free results in the sense that only public
quantities set the allocation, but they are synthetic. The residual
model is an assumption, and on real logits the gain will depend on how
far the broadcast student is from the sites.

\begin{table}[!htbp]
\revfloat\centering\footnotesize
\setlength{\tabcolsep}{4pt}
\renewcommand{\arraystretch}{1.08}
\caption{Coordinate water-filling reaches lower KL at $1/128$ bit per
coordinate than the dense one-bit codebook at one bit. Synthetic
residual coding at $V=131{,}072$, $K=4$, residual cap $L=1$, public
weights $p_v(1-p_v)$ from the broadcast student, $30$ contexts by $10$
dither seeds; entries are mean KL with $95\%$ intervals across contexts.
Without any payload, KL is $5.04\times10^{-2}$
$[4.08,5.99]\times10^{-2}$.}
\label{tab:subbit}
\begin{tabular}{@{}rrlll@{}}
\toprule
bits/coordinate & bits/node & dense one-bit & top-$k$, two bits & water-filling\\
\midrule
$1/128$ & $1{,}024$ & --- & $3.88\ [3.63,4.13]\times10^{-3}$ & $1.03\ [0.97,1.10]\times10^{-3}$\\
$1/32$ & $4{,}096$ & --- & $3.65\ [3.42,3.88]\times10^{-3}$ & $3.85\ [3.62,4.09]\times10^{-4}$\\
$1/8$ & $16{,}384$ & --- & $3.22\ [3.05,3.39]\times10^{-3}$ & $1.36\ [1.28,1.45]\times10^{-4}$\\
$1/2$ & $65{,}536$ & --- & $3.21\ [2.86,3.57]\times10^{-3}$ & $4.20\ [3.93,4.46]\times10^{-5}$\\
$1$ & $131{,}072$ & $3.09\ [2.88,3.31]\times10^{-2}$ & $3.42\ [3.18,3.66]\times10^{-3}$ & $1.98\ [1.86,2.11]\times10^{-5}$\\
\bottomrule
\end{tabular}
\end{table}

\subsection{Sites that are not identically distributed}
\label{app:practical-noniid}

\paragraph{What does not use identical distributions.} The quantization
analysis is conditional on the logits that the nodes transmit. The
quantization increment of the aggregate teacher (the bandwidth terms of
Theorems~\ref{thm:fpld-upper} and~\ref{thm:het-upper} before the
transfer to the student), the allocation rule of
Theorem~\ref{thm:alloc}, the robustness bound of
Corollary~\ref{cor:adaptive}, and the one-round-lag guarantee of
Theorem~\ref{thm:streaming} use only that node $i$'s logits are bounded
by its reported cap and that the dither is independent of them. They
therefore hold verbatim when the nodes' models, and the data that
produced them, differ. Two checks in the paper already have this
structure: the synthetic allocation study with distinct local
population logits (Figure~\ref{fig:fig2}(b); finite-exact allocation
lowers KL from $0.03899$ to $0.02744$ at $B_{\mathrm{tot}}/V=8$, a
$29.6\%$ reduction against $30.4\%$ with a common target), and the
seven-model federation of Section~\ref{sec:reallm}, whose members differ
in architecture (GPT-2 versus GPT-Neo), size ($82$ million to
$1.5$ billion parameters), and training corpus.

\paragraph{What does change.} Identically distributed sites
(Assumption~\ref{ass:A3}) are used for two things: the statistical term
and the identification of the target. With site distributions
$P_i^\star$, the unquantized aggregate teacher
$\mathrm{softmax}(K^{-1}\sum_i\ell_i)$ estimates the normalized
geometric mean of the $P_i^\star$, the logarithmic opinion
pool~\citep{genest1986}, rather than a single $P^\star$. Its risk against
a deployment target therefore contains a heterogeneity bias that no
bandwidth reallocation removes, and the pooled $d/(Kn)$ rate holds only
for the component of the parameter that the sites share, as in
heterogeneity-aware distributed and integrative
estimation~\citep{duan2022,cai2022integrative}. Shifts in the context
distribution enter through the coverage ratio $\rho$, which must be
computed against the deployment marginal and is controlled by choosing
the public probe law $Q$ to cover every site's context distribution, or
by importance weighting of the probe contexts~\citep{shimodaira2000};
\citet{kpotufe2021} describe when such reweighting can and cannot help.

\paragraph{Practice under heterogeneity.} Three adjustments keep the
paper's prescriptions usable. (i)~\emph{Weighted pooling.} If the
aggregate uses site weights $a_i$ (for example proportional to sample
sizes, or estimated on a validation probe), its dither variance is
$\sum_ia_i^2q_i(B_i)$, so Theorem~\ref{thm:alloc} applies with
$w_i\propto a_i^2L_{\ell,i}^2$: sites that count more in the pool
receive more bits, and nothing else in the allocation changes.
(ii)~\emph{Measuring heterogeneity from the messages.} Because the dither
variance is known exactly, the aggregator can estimate the between-site
spread of logits on the probe set without bias: since
$\mathbb E\|\tilde\ell_i-\tilde\ell_j\|_2^2=\|\ell_i-\ell_j\|_2^2
+V\{q_i(B_i)+q_j(B_j)\}$, the statistic
$\|\tilde\ell_i-\tilde\ell_j\|_2^2-V\{q_i(B_i)+q_j(B_j)\}$ is unbiased, a
Sheppard-type correction that indicates
when sites disagree enough to make the pool a poor target.
(iii)~\emph{Target-specific estimation.} When one site, or one
under-represented population, is the target, transfer-learning and
federated-transfer estimators that borrow from the other sites only to
the extent that they agree~\citep{cai2021,li2022,tian2023,li2023}, or
robust multi-task and maximin aggregation~\citep{meinshausen2015,duan2023},
replace the equal-weight pool. Combining these estimators with the
quantized channel analyzed here, so that bits are allocated toward the
sites that are most informative for the target, is a natural extension
that the present theory does not cover.

\subsection{Further connections to the statistical literature}
\label{app:practical-lit}

The averaging step of FPLD is a one-shot divide-and-conquer
estimator~\citep{chen2014,battey2018,jordan2019,volgushev2019,fan2019,dobriban2021,fan2023};
\citet{huo2019} review aggregated inference, and \citet{duan2022}
and \citet{cai2022integrative} treat heterogeneity across sites. The
minimax program under communication budgets that Sections~\ref{sec:lb}
and~\ref{sec:uncond} follow has statistical counterparts for estimation,
adaptation, and testing~\citep{szabo2020,cai2022nonparametric,cai2022adaptive,szabo2023},
and the privacy-constrained analogue is studied
by~\citet{duchi2018}; the lower-bound tools are those
of~\citet{yu1997} and~\citet{tsybakov2009}. Averaging logits is the
logarithmic opinion pool~\citep{genest1986}, as opposed to the linear
pool~\citep{stone1961,ranjan2010}, and weighted pools relate to
frequentist model averaging~\citep{hjort2003}; the KL risk is the
expected excess logarithmic score~\citep{gneiting2007}. The quantizer is a
coarsening mechanism~\citep{heitjan1989,heitjan1991}, and subtractive
dither is what makes the Sheppard-type variance
correction~\citep{dempster1983} exact. Top-$k$ sparsification is a
hard-thresholding rule~\citep{donoho1994}, here applied to a vector of
logit residuals with a budget rather than a noise-level threshold.
\end{revblock}

\end{document}